\documentclass[11pt]{article}
\usepackage[a4paper,top=3cm,bottom=2cm,left=2cm,right=2cm,marginparwidth=1.75cm]{geometry}
\usepackage{fullpage}
\usepackage{graphics} 
\usepackage{epsfig}
\usepackage{amsmath} 
\usepackage{amssymb} 
\usepackage{amsthm}
\usepackage{tabularx,booktabs}
\usepackage{mathrsfs}
\usepackage{selectp}
\usepackage{Shorthands}
\usepackage[normalem]{ulem}
\usepackage{authblk}
\usepackage{subcaption}
\usepackage{natbib}
\usepackage{Shorthands}
\usepackage{algorithm}
\usepackage{algpseudocode}
\usepackage{wrapfig}
\usepackage{scalerel,stackengine}
\stackMath
\newcommand\reallywidehat[1]{%
\savestack{\tmpbox}{\stretchto{%
  \scaleto{%
    \scalerel*[\widthof{\ensuremath{#1}}]{\kern-.6pt\bigwedge\kern-.6pt}%
    {\rule[-\textheight/2]{1ex}{\textheight}}
  }{\textheight}%
}{0.5ex}}%
\stackon[1pt]{#1}{\tmpbox}%
}

\usepackage{xr}
\usepackage{xcolor}

\newboolean{showcomments}
\setboolean{showcomments}{true}
\ifthenelse{\boolean{showcomments}}
{ \newcommand{\mynote}[3]{
		\fbox{\bfseries\sffamily\scriptsize#1}
		{\small$\blacktriangleright$\textsf{\emph{\color{#3}{#2}}}$\blacktriangleleft$}}
	\newcommand{\zzz}[1]{{\setlength{\fboxsep}{2pt}\fcolorbox{black}{yellow}{\textsf{\emph{#1}}}}\xspace}}
{ \newcommand{\mynote}[3]{}
	\newcommand{\zzz}[1]{}}

\usepackage{float}
\usepackage{booktabs}
\usepackage[table]{xcolor}
\definecolor{oursrow}{RGB}{238,246,252}
\usepackage[toc,page,header]{appendix}
\usepackage{minitoc}

\usepackage{hyperref}
\hypersetup{       
    unicode=false,          
    pdftoolbar=true,       
    pdfmenubar=true,      
    pdffitwindow=false,     
    pdfstartview={FitH},   
    pdfnewwindow=true,      
    colorlinks=true,      
    linkcolor=blue,          
    citecolor=blue,       
    filecolor=magenta,      
    urlcolor=blue,           
    breaklinks=true
}

\usepackage[toc,page,header]{appendix}
\usepackage{minitoc}

\title{Byzantine-Robust Federated Representation Learning}
\date{}

\author[1]{Leonardo F. Toso\footnote{Correspondence to: \texttt{leonardo.toso@columbia.edu}.}}
\author[1]{James Anderson}
\author[2]{Rafael Pinot}
\author[3]{Nirupam Gupta}
\affil[1]{Columbia University, USA}
\affil[2]{Sorbonne Université and Université Paris Cité, CNRS, LPSM, France}
\affil[3]{University of Copenhagen, Denmark}

\begin{document}

\doparttoc 
\faketableofcontents

\maketitle

\begin{abstract}

We study federated learning (FL) with adversarial clients, where the goal is to minimize the average loss of the honest (non-adversarial) clients without knowing their identity. Under heterogeneity, a single shared model parameter is statistically inappropriate: it cannot capture the distinct data-generating processes across clients, incurring an irreducible model-heterogeneity bias and severely limiting robustness to adversarial clients (a.k.a. Byzantine-robustness). We address this problem through representation learning, where each client learns a personalized linear head, while collaboratively estimating a shared nonlinear representation through Byzantine-robust aggregation. We demonstrate that the heterogeneity among honest representation gradients is controlled by the representation error and statistical errors that decay either with the number of data samples per client ($\tau$) or the number of iterations ($T$). In particular, our non-asymptotic parameter recovery error bound reveals three terms: (i) an initialization-dependent error that goes away with $T$, (ii) finite-sample noise terms that decreases with $\tau$ and the number of honest clients, and (iii) a stochastic gradient variance term that also reduces with $T$. Importantly, with no irreducible model-heterogeneity bias in our bounds. We extend the regression analysis to multiclass classification, and empirically validate it on CIFAR-10, FEMNIST, and School Exam Score datasets.
\end{abstract}

\allowdisplaybreaks

\section{Introduction}

Learning from decentralized data requires accounting for heterogeneous clients while defending against adversarial participants (also referred to as {\em Byzantine} clients~\cite{lamport2019byzantine, guerraoui2024byzantine}). This is crucial for federated learning (FL), where clients keep their data locally and communicate only model updates to a central server \citep{mcmahan2023communicationefficient,konecny2016federated}. Typically, the client models are non-identical, and thus a single common model parameter is statistically inappropriate. In addition, adversarial clients make the problem even harder by sharing arbitrary updates to corrupt the underlying learning process. Byzantine-robust FL addresses the latter through {\em robust aggregators} \citep{blanchard2017machine,yin2018byzantine,mhamdi2018hidden,small1990survey,guerraoui2024robust,toso2026gradient}. 

In particular, the class of $(f,\kappa)$-robust aggregators \citep{guerraoui2024robust}, where $f$ is the number of adversarial clients and $\kappa$\footnote{$\kappa$ typically scales with $f/n$, with $n$ being the total number of participating clients \citep{allouah2023fixing}.} the robustness coefficient, controls its deviation from the honest average through the heterogeneity of the honest clients' updates. Consequently, under a single common model parameter, the robust aggregation error is nonzero even under an infinite number of local data samples \citep{ghosh2019robust,karimireddy2020byzantine,allouah2023fixing,mishchenko2023partially,toso2026gradient}. Indeed, this is the fundamental limitation of Byzantine-robust FL under heterogeneity: the server cannot tell the difference between honest heterogeneity and adversarial behavior when the honest clients are heterogeneous \citep{karimireddy2020byzantine}.  Hence, the fundamental question: \emph{Should we instead separate the common and client-specific components of each local model, collaborate robustly to learn what is shared and personalize what is specific?}

We address this question by leveraging a common nonlinear representation and client-specific linear heads. More precisely, clients can have different models while relying on a common representation of their covariates (features). Learning common feature representations have been proven to improve statistical efficiency in multitask and personalized federated learning~\citep{tripuraneni2020theory,du2020few,collins2021exploiting,yang2023fedrep,mishchenko2023partially,zhang2024sample,zhang2024guarantees}. Common representations have also proved useful for learning across similar dynamical systems
\citep{zhang2022multitask,zhang2024sample,lee2024regret,fallah2025adversarial}. The intuition is that, although the representation gradients depend on the client-specific heads, under a realizability assumption, each composition of the shared representation and a personalized head can fit the underlying client data distribution. Therefore, after fitting the local heads, the remaining representation-gradient heterogeneity is governed only by representation-recovery and statistical errors, rather than by a persistent model-parameter heterogeneity bias arising from fitting a single common model across clients. 

Motivated by this, we consider alternating optimization to (i) fit each personalized head with clients' specific data samples and to (ii) robustly aggregate the representation updates of all participating clients (honest and adversarial). We then bound the honest representation gradient heterogeneity and use it  to establish the convergence and parameter recovery guarantees. In the end, our bounds comprise an initialization-dependent term that decreases with number of training rounds $T$, finite-sample noise term that scales inversely with the number of data samples per client $\tau$, and stochastic gradient term that scales inversely with $\sqrt{T}$. We first provide the results for nonlinear regression and then instantiate it to multiclass classification with squared loss. \\

\noindent \textbf{Informal main result.} Let $\mathcal H$ denote the set of honest clients and let $T$ be the number of communication rounds. Suppose that the number of samples per client satisfies $\tau\geq\tau_{\mathsf{burn\text{-}in}}>0$,with $\tau_{\mathsf{burn\text{-}in}}$ specified in Section~\ref{sec:assumptions} (equation \ref{eq:main-sample-burnin}), and let the robust aggregation coefficient $\kappa = \mathcal{O}(f/n)$ satisfy $\kappa\leq \underline{\kappa}$ with $\underline{\kappa}$ defined in Section~\ref{sec:assumptions} (\eqref{eq:main-kappa-condition}). Then, up to logarithmic factors, with high probability, our main result (Theorem~\ref{nonlinear_convergence_mb}) proves\vspace{-0.1cm}
\begin{equation*}
\begin{aligned}
\frac{1}{T}\sum_{t=0}^{T-1}\left\|\mathbb E\nabla_{\mathbf{B}}L_{\mathcal H}
\left(\{\mathbf{h}_i^{(t+1)}\}_{i\in\mathcal H},\mathbf{B}^{(t)}\right)\right\|^2
&\hspace{-0.15cm}\lesssim \hspace{-0.3cm}\underbrace{\frac{\Delta_0}{\sqrt T}}_{\text{initialization}}
\hspace{-0.35cm}+\hspace{-0.05cm}\underbrace{ \kappa\left(\frac{\mathsf{d}_{\mathsf{rep}}}{\tau|\mathcal H|}\hspace{-0.05cm}+\hspace{-0.05cm}\frac{\mathsf{d}_{\mathsf{head}}}{\tau}\right)\sigma^2}_{\text{finite samples}}\hspace{-0.05cm}+\hspace{-0.05cm}\underbrace{\left(\kappa\hspace{-0.05cm}+\hspace{-0.05cm}\kappa^2\hspace{-0.05cm}+\hspace{-0.05cm}\frac{1\hspace{-0.05cm}+\hspace{-0.05cm}\kappa}{|\mathcal H|}\right)\frac{\sigma_{\mathsf{g}}^2}{\sqrt T}}_{\text{stochastic gradients}},
\end{aligned}\\[-0.1cm]
\end{equation*}
where $L_{\mathcal H}(\cdot)$ is the empirical honest averaged loss and $\sigma^2$, $\sigma_{\mathsf{g}}^2$ are the label noise and stochastic gradient variances, respectively. In addition, $\Delta_0$ denotes the initialization-dependent error and $\mathsf{d}_{\mathsf{rep}}$ and $\mathsf{d}_{\mathsf{head}}$ are the underlying representation ($\mathbf{B}^{(t)}$) and client-specific head ($\mathbf{h}_i^{(t+1)}, \forall i \in \mathcal{H}$) dimensions, respectively. We note that the first and last terms decay with the number of rounds, while the finite-sample terms decreases with the local sample size. Importantly, there is no persistent model-heterogeneity term, demonstrating the benefit of collaboratively learning what is shared and personalizing what is specific under adversarial attacks. This rigorously complements the benefit of partial personalization pointed out in \citep{mishchenko2023partially}\footnote{We position our paper within the literature throughout and include a related work section in Appendix \ref{app:related_work}.} by carefully providing a high-probability, finite-sample guarantee for a common nonlinear representation under arbitrary adversarial updates.\\

\noindent \textbf{Contributions.} Our main contributions are summarized as follows:\\

\noindent$\bullet$ We propose adversarially robust nonlinear representation learning for heterogeneous FL. Each honest client learns a personalized head, while collaborating robustly without knowing honest ones' identity to learn a common nonlinear feature representation.\\

\noindent$\bullet$ We prove that the representation-gradient heterogeneity is controlled only by the representation-recovery error, finite-sample estimation error, and the stochastic-gradient variance (Lemma~\ref{theorem:gradient_heterogeneity_nonlinear_mb}). Therefore, our guarantee is not affected by any model and data heterogeneity bias that hurts traditional Byzantine-robust FL \citep{karimireddy2020byzantine,allouah2023fixing}.\\

\noindent$\bullet$ We establish a high-probability ergodic convergence guarantee for the common representation under adversarial attacks
(Theorem~\ref{nonlinear_convergence_mb}). We then derive the non-asymptotic guarantees for recovering the underlying representation parameter (Theorem~\ref{thm:nonlinear_param_recovery}) and  (Corollary~\ref{nonlinear_parameter_recovery_mb}). Our bounds characterize the dependence on the number of local samples per client, the number of honest clients, the stochastic-gradient noise, and the robustness coefficient of the aggregation rule.\\

\noindent$\bullet$ We go beyond regression and extend our parameter recovery  guarantee to multiclass classification. We demonstrate that under a class-probability margin condition, the classification error inherits the nonlinear prediction-error rate from our regression analysis (Corollary~\ref{cor:multiclass-classification}).\\

\noindent \textbf{Notation.} For a positive integer $n$, we write $[n]:=\{1,\ldots,n\}$. The Euclidean norm for vectors and spectral norm for matrices are denoted by $\|\cdot\|$, while $\|\cdot\|_F$ denotes the Frobenius norm. For a matrix $A$, $\lambda_{\min}(A)$ and $\lambda_{\max}(A)$ denote its smallest and largest eigenvalues, and $a\vee b:=\max\{a,b\}$. We write $\mathbb E_t[\cdot]:=\mathbb E[\cdot\mid\mathcal F_t]$ for conditional expectation with respect to the filtration $\mathcal F_t$. We also write $h\lesssim g $ when $h\leq Cg$ for a constant $C>0$. We write $a\asymp b$ if there exist constants $C_1,C_2>0$ such that $ C_1 b\leq a\leq C_2 b.$ For a real-valued random variable $X$, its sub-Gaussian norm \cite[Definition 2.6.4]{vershynin2018high} is defined as follows:
\begin{align*}
\|X\|_{\psi_2} := \inf\left\{t>0 :
\mathbb{E}\left[ \exp\left(\frac{X^2}{t^2}\right) \right] \leq 2\right\}.
\end{align*}
The random variable $X$ is said to be sub-Gaussian if $\|X\|_{\psi_2}<\infty$. For a random vector $Z\in\mathbb{R}^d$ and a random matrix $M\in\mathbb{R}^{p\times q}$, we define
\begin{align*}
\|Z\|_{\psi_2}
&:=\sup_{u\in\mathbb{S}^{d-1}}\|\langle u,Z\rangle\|_{\psi_2} \text{ and }
\|M\|_{\psi_2}
:=\sup_{U\in\mathbb{S}_F^{p\times q}}\|\langle U,M\rangle_F\|_{\psi_2}.
\end{align*}
$Z$ and $M$ are sub-Gaussian when its corresponding $\psi_2$-norm is finite.

\section{Problem Formulation}\label{sec:problem_formulation}

We first describe the adversarial FL setup, including the $(f,\kappa)$-robust aggregation definition, and then introduce our nonlinear representation learning problem with client-specific heads and a shared representation. Our adversarial setup follows the standard Byzantine-robust FL framework considered in \citep{karimireddy2020byzantine,farhadkhani2022byzantine,guerraoui2024robust}.

\subsection{Adversarially Robust FL}\label{sec:robust_fl_setup}

We consider a federated system with $n$ clients and a central server. An unknown subset $\mathcal {B}\subset[n]$ of cardinality $|\mathcal {B}|=f<n/2$ is adversarial and may share arbitrary updates with the server with the intention of corrupting the underlying learning process \citep{guerraoui2024robust}. The other clients form the honest set $\mathcal H=[n]\setminus\mathcal {B}$, with $|\mathcal H|=n-f$. \\

\noindent \textbf{Data.} At every training round $t$, each honest client $i\in\mathcal H$ draws two independent batches of $\tau$ i.i.d. samples from a client-specific distribution that we denote by $\mathsf D_i$ on $\mathbb R^{d}\times\mathbb R^{q}$. The first batch is used to estimate the personalized head and the second to evaluate the representation loss and gradient (this is also referred to as ``debiasing'' in \citep{zhang2024sample}). Both batches of data are independent of the history before the current round and independent across honest clients $i \in \mathcal{H}$. To avoid cumbersome notation, we suppress the round and batch superscripts whenever the role of the samples is clear. 

The server aggregates client updates using a robust aggregation rule. Given updates $u_1,\ldots,u_n\in\mathbb R^p$, define the honest average $\bar u_{\mathcal H}:=\frac{1}{|\mathcal{H}|}\sum_{i\in\mathcal H}u_i$. We next define the class of $(f,\kappa)$-robust aggregators \citep{allouah2023fixing,guerraoui2024robust}. 

\begin{definition}[$(f,\kappa)$-robust aggregator]\label{def:robust-agg}
An aggregation rule $\mathsf F:(\mathbb R^p)^n\to\mathbb R^p$ is $(f,\kappa)$-robust if, for every collection of inputs and every honest set $\mathcal H$ with $|\mathcal H|=n-f$,
\begin{align*}
\left\|\mathsf F(u_1,\ldots,u_n)-\bar u_{\mathcal H}\right\|^2
\leq \frac{\kappa}{|\mathcal H|}\sum_{i\in\mathcal H}\|u_i-\bar u_{\mathcal H}\|^2.
\end{align*}\vspace{-0.5cm}
\end{definition}
We also emphasize that the $(f,\kappa)$-robustness condition is satisfied by many aggregation rules, including Krum \citep{blanchard2017machine}, geometric median \citep{small1990survey,acharya2022robust}, coordinate-wise median and trimmed mean \citep{yin2018byzantine}, and minimum-diameter averaging \citep{el2021collaborative}. We note that $\kappa$ typically scales as $f/n$, since standard robust aggregation rules can be combined with the nearest-neighbor mixing (NNM) preprocessing step of \cite{allouah2023fixing} to achieve this information-theoretically optimal scaling.
Therefore, the effect of adversarial clients is determined by $\kappa$ and the heterogeneity of the honest updates. Our analysis then controls this heterogeneity using a shared representation across clients as discussed below.

\subsection{Nonlinear Representation Learning and Personalization}\label{sec:nonlinear_rep_learning_mb}

We consider that the observations of every honest client $i \in \mathcal{H}$ are generated according to
\begin{align*}
Y_{i,k}=\mathbf{h}_i^\star\phi_{\mathbf{B}^\star}(X_{i,k})+V_{i,k}, \text{ for all }k=1,\ldots,\tau,
\end{align*}
where $\mathbf{B}^\star\in\mathbb R^p$ parameterizes a shared nonlinear representation $\phi_{\mathbf{B}^\star}:\mathbb R^{d}\to\mathbb R^r$, $\mathbf{h}_i^\star\in\mathbb R^{q\times r}$ is a client-specific linear head, and $V_{i,k}\in\mathbb R^{q}$ is label noise. The client-specific heads capture heterogeneity, while $\mathbf{B}^\star$ represents what is shared across clients. We assume that the covariates are uniformly bounded, i.e., $\|X_{i,k}\|\leq R$, for every honest clients $i \in \mathcal{H}$ and sample $k \in [\tau]$.

For any candidate head $\mathbf{h}$ and representation parameter $\mathbf{B}$, we define the client-specific and honest-average empirical losses on the current data batch as follows:\vspace{-0.2cm}
\begin{align*}
L_i(\mathbf{h},\mathbf{B}):=\frac{1}{\tau}\sum_{k=1}^{\tau}\left\|Y_{i,k}-\mathbf{h}\phi_{\mathbf{B}}(X_{i,k})\right\|^2 \text{ and } L_{\mathcal H}(\{\mathbf{h}_i\}_{i\in\mathcal H},\mathbf{B}):=\frac{1}{|\mathcal H|}\sum_{i\in\mathcal H}L_i(\mathbf{h}_i,\mathbf{B}).
\end{align*}\vspace{-0.2cm}

Let $r_i(X_{i,k};\mathbf{h},\mathbf{B}):=\mathbf{h}\phi_{\mathbf{B}}(X_{i,k})-Y_{i,k}$ and $J_{\mathbf{B}}(X):=\nabla_{\mathbf{B}}\phi_{\mathbf{B}}(X)\in\mathbb R^{r\times p}$, and thus
\begin{align*}
\nabla_{\mathbf{B}} L_i(\mathbf{h},\mathbf{B})=\frac{2}{\tau}\sum_{k=1}^{\tau}J_{\mathbf{B}}^\top(X_{i,k})\mathbf{h}^\top r_i(X_{i,k};\mathbf{h},\mathbf{B}).
\end{align*}\vspace{-0.4cm}
 
\noindent \textbf{Goal.} We aim to learn the shared representation $\mathbf{B}$ and the personalized heads $\{\mathbf{h}_i\}_{i \in \mathcal{H}}$ by minimizing $L_{\mathcal H}(\cdot)$ despite not knowing $\mathcal H$ operating under adversarial attacks.

To do so, we leverage alternating minimization \cite{collins2021exploiting, zhang2024sample}. At round $t$, every client estimates its local head on the head-fitting batch and sends a stochastic representation-gradient update evaluated on the independent representation batch. The server then robustly aggregates the representation-gradient momenta. More precisely, we have\vspace{-0.1cm}
\begin{align}
\textbf{Head estimation. } \;\;\ \mathbf{h}_i^{(t+1)}&\in\argmin_{\mathbf{h}} L_i(\mathbf{h},\mathbf{B}^{(t)}),\notag \\
\textbf{Representation update. } \;\;\ \mathbf{B}^{(t+1)}&=\mathbf{B}^{(t)}-\eta\mathsf F\left(m_{\mathbf{B},1}^{(t)},\ldots,m_{\mathbf{B},n}^{(t)}\right),\label{alternating_descent}
\end{align}
where $\{m_{\mathbf{B},i}\}_{i \in [n]}$ denotes the representation-gradient momentum given by
$m_{\mathbf{B},i}^{(t)}:=\beta m_{\mathbf{B},i}^{(t-1)}+(1-\beta)\widehat{\nabla}_{\mathbf{B}}L_i^{(t)}(\mathbf{h}^{(t+1)},\mathbf{B}^{(t)})$, for some momentum coefficient $\beta\in(0,1)$. We also assume that $m_{\mathbf{B},i}^{(0)}=\mathbf{0}$. In addition, here $\widehat{\nabla}_{\mathbf{B}}L_i^{(t)}(\cdot)$ denotes the stochastic representation gradient computed by client $i$ at round $t$, and $\widehat{\nabla}_{\mathbf{B}}L_{\mathcal H}^{(t)}(\cdot):=\frac{1}{|\mathcal{H}|}\sum_{i\in\mathcal H}\widehat{\nabla}_{\mathbf{B}}L_i^{(t)}(\cdot)$ is its average over the honest clients.

The head update adapts the predictor to each client’s local model. We assume that this update is solved exactly, corresponding to a sufficiently large number of local optimization steps. This allows us to isolate the effect of learning a common representation on the convergence of adversarially robust FL. Extending our analysis to account for local head estimation error is left for future work.

On the other hand, the representation update keeps only what is common across clients. We reemphasize here that the independent batches (for local heads and representation) here are an artifact that separates the error incurred when fitting the head from the noise in the representation update. In practice, the same local dataset can be partitioned or resampled \citep{zhang2024sample}.

As discussed above, for an $(f,\kappa)$-robust aggregator, the effect of adversarial clients is controlled by the heterogeneity among the honest updates. As the server aggregates representation-gradient momenta, our analysis proceeds in two steps. We first bound the heterogeneity among the honest representation gradients. We then invoke Lemmas~\ref{lem:agg_error} and~\ref{lem:momentum_deviation_hp}, deferred to the appendix, to control the additional error introduced by stochastic gradients and momentum. Therefore, for later use, we define the representation gradient heterogeneity per iteration and averaged over rounds as follows:\vspace{-0.2cm}
\begin{align*}
\bar G^{(t)}&:=\frac{1}{|\mathcal H|}\sum_{i\in\mathcal H}\left\|\nabla_{\mathbf{B}} L_i(\mathbf{h}_i^{(t+1)},\mathbf{B}^{(t)})-\nabla_{\mathbf{B}} L_{\mathcal H}(\{\mathbf{h}_j^{(t+1)}\}_{j\in\mathcal H},\mathbf{B}^{(t)})\right\|^2 \text{ and } \bar G_T:=\frac{1}{T}\sum_{t=0}^{T-1}\bar G^{(t)}.
\end{align*}

\subsection{Assumptions}\label{sec:assumptions}

Let us here collect the assumptions used throughout the analysis. 

\begin{assumption}[Sub-Gaussian noise]\label{assm:subg}
The label noise $V_{i,k}$ is mean-zero and $\sigma^2$-sub-Gaussian. In particular, for every honest client $i\in\mathcal H$, sample $k\in[\tau]$, unit vector $u\in\mathbb S^{q-1}$, and $\lambda\in\mathbb R$, we have
\begin{align*}
\mathbb E[V_{i,k}\mid X_{i,k}]=0 \text{ and }
\mathbb E\left[\exp(\lambda u^\top V_{i,k})\mid X_{i,k}\right]\leq\exp\left(\frac{\lambda^2\sigma^2}{2}\right).
\end{align*}
\end{assumption}

\begin{definition}\label{def:infty-to-two}
For $A\in\mathbb R^{m\times n}$, define $\|A\|_{\infty\to2}:=\sup_{\|z\|_\infty\leq1}\|Az\|_2$.
\end{definition}

\begin{assumption}\label{assumption:uniform_bound_nonlinear}\label{assm:smooth}\label{asm:G-jacobian-lip}
There exist constants $\bar\phi,\bar J,H>0$ such that $\|\phi_{\mathbf{B}}(X)\|\leq\bar\phi$, $\|J_{\mathbf{B}}(X)\|_{\infty\to2}\leq\bar J$, $\|\mathbf{h}_i^{(t+1)}\|\leq H$, and $\|\mathbf{h}_i^\star\|\leq H$ for all honest clients, covariates, and iterates.
Moreover, the honest-average loss is $L_1$-smooth with respect to $\mathbf{B}$ in a neighborhood of $\mathbf{B}^\star$, and every client loss is $L_2$-smooth with respect to its client-specific head. In addition, the population representation gradient is $L_3$-Lipschitz in the heads, i.e., for any $\mathbf{B}$ and head collections $\{\mathbf{h}_i\}_i,\{\mathbf{h}_i'\}_i$,
\begin{align*}
\left\|\nabla_{\mathbf{B}}\mathbb E_XL_{\mathcal H}(\{\mathbf{h}_i\}_i,\mathbf{B})-\nabla_{\mathbf{B}}\mathbb E_XL_{\mathcal H}(\{\mathbf{h}_i'\}_i,\mathbf{B})\right\|
\leq \frac{L_3}{|\mathcal H|}\sum_{i\in\mathcal H}\|\mathbf{h}_i-\mathbf{h}_i'\|.
\end{align*}
On the other hand, for the Jacobian, there exists $L_4>0$ such that   $\|J_{\mathbf{B}}(X)-J_{\mathbf{B}'}(X)\|_{2\to2}\leq L_4\|\mathbf{B}-\mathbf{B}'\|$ for all $X$ and all $\mathbf{B},\mathbf{B}'$ satisfying $\|\mathbf{B}-\mathbf{B}^\star\|,\|\mathbf{B}'-\mathbf{B}^\star\|\leq\rho_0$, for some $\rho_0 > 0.$
\end{assumption}

For each honest client, we define the empirical and population feature covariances as follows:
\begin{align*}
\widehat\Sigma_{i,\phi}(\mathbf{B})&:=\frac{1}{\tau}\sum_{k=1}^{\tau}\phi_{\mathbf{B}}(X_{i,k})\phi_{\mathbf{B}}(X_{i,k})^\top \text{ and }
\Sigma_{i,\phi}(\mathbf{B}):=\mathbb E_{X\sim\mathsf D_i}[\phi_{\mathbf{B}}(X)\phi_{\mathbf{B}}(X)^\top],
\end{align*}
and let $\mathcal N(\mathbf{B}^\star,\rho_0):=\{\mathbf{B}\in\mathbb R^p:\|\mathbf{B}-\mathbf{B}^\star\|\leq\rho_0\}$ be the ball of radius $\rho_0$ centered at $\mathbf{B}^\star$.

\begin{assumption}\label{asm:G-pl-feature}\label{asm:task diversity}
There exists $\mu_1>0$ such that $\lambda_{\min}(\Sigma_{i,\phi}(\mathbf{B}))\geq\mu_1$ for every honest client $i$ and every $\mathbf{B}\in\mathcal N(\mathbf{B}^\star,\rho_0)$.
In addition, the honest-client population provides positive curvature for identifying the shared representation, i.e., we assume that
\begin{align*}
\mu_3:=\lambda_{\min}\!\left(\frac{1}{|\mathcal H|}\sum_{i\in\mathcal H}\mathbb E_X\!\left[J_{\mathbf{B}^\star}^\top(X)\mathbf{h}_i^{\star\top}\mathbf{h}_i^\star J_{\mathbf{B}^\star}(X)\right]\right)>0.
\end{align*}
\end{assumption}

\begin{assumption}\label{assump:subgaussian_gradient_noise}\label{assump:unbiased}
For $\nu_i^{(t)}:=\widehat{\nabla}_{\mathbf{B}}L_i^{(t)}(\mathbf{h}_i^{(t+1)},\mathbf{B}^{(t)})-\nabla_{\mathbf{B}} L_i(\mathbf{h}_i^{(t+1)},\mathbf{B}^{(t)})$, conditioned on $\mathcal F_t$, the vectors $\{\nu_i^{(t)}\}_{i\in\mathcal H}$ are independent, mean-zero, and satisfy $\|\nu_i^{(t)}\|_{\psi_2\mid\mathcal F_t}\leq\sigma_{\mathsf{g}}$ almost surely.
Equivalently, for every honest client and iteration, we have 
\begin{align*}
\mathbb E[\widehat{\nabla}_{\mathbf{B}}L_i^{(t)}(\mathbf{h}_i^{(t+1)},\mathbf{B}^{(t)})]=\nabla_{\mathbf{B}} L_i(\mathbf{h}_i^{(t+1)},\mathbf{B}^{(t)}).
\end{align*}
\end{assumption}

We emphasize that these assumptions are standard conditions used in non-asymptotic analyses of representation learning and Byzantine-robust stochastic optimization. In particular, bounded features and Jacobians, local smoothness, and nondegenerate feature covariance provide the regularity and local identifiability required in our nonlinear setting. Similar covariance, noise, and regularity conditions appear in analyses of linear and nonlinear representation learning \citep{collins2021exploiting,zhang2024sample,zhang2024guarantees}. Our client-diversity condition in Assumption \ref{asm:G-pl-feature} is the nonlinear analogue of the task-diversity conditions used to identify a shared representation across heterogeneous tasks \citep{tripuraneni2020theory,du2020few,collins2021exploiting}. Finally, unbiased stochastic gradients with bounded conditional variance are standard in Byzantine-robust stochastic optimization \citep{karimireddy2020byzantine,farhadkhani2022byzantine,allouah2023fixing}.

For the reminder of the paper it is taken for granted that Assumptions \ref{assm:subg}--\ref{assump:unbiased} hold.

\section{Theoretical Guarantees}

We are now ready to state our main results. The analysis proceeds in three steps. Lemma \ref{theorem:gradient_heterogeneity_nonlinear_mb} controls the heterogeneity that drives the robust-aggregation error, Theorem \ref{nonlinear_convergence_mb} establishes convergence of the population representation gradient, and Corollary \ref{nonlinear_parameter_recovery_mb} converts this guarantee into the parameter recovery error bound. We fix a small probability of failure $\delta\in(0,1)$ and present our results in a probabilistic manner. We first collect the quantities used throughout.\vspace{-0.1cm}
\begin{align*}
&\Delta_{\mathbf{B},T+1}:=\frac{1}{T}\sum_{t=1}^{T}\|\mathbf{B}^{(t)}-\mathbf{B}^\star\|^2, \quad 
\bar Q_T:=\frac{1}{T}\sum_{t=0}^{T-1}\left\|\mathbb E_X \nabla_{\mathbf{B}} L_{\mathcal H}
\left(\{\mathbf{h}_i^{(t+1)}\}_{i\in\mathcal H},\mathbf{B}^{(t)}\right)\right\|^2,\\
&\bar E_T:=\frac{1}{T|\mathcal H|\tau}\sum_{t=0}^{T-1}\sum_{i\in\mathcal H}\sum_{k=1}^{\tau}
\left\|\mathbf{h}_i^{(t+1)}\phi_{\mathbf{B}^{(t)}}(X_{i,k})-\mathbf{h}_i^\star\phi_{\mathbf{B}^\star}(X_{i,k})\right\|^2.
\end{align*}
Here, $\Delta_{\mathbf{B},T+1}$ and $\bar E_T$ denote the average representation-recovery and prediction-recovery errors, respectively, whereas $\bar Q_T$ denotes the squared norm of the population honest-average representation gradient, averaged over rounds $t\in\{0,1,\ldots,T-1\}$. We also define
\begin{align*}
\hspace{-0.4cm}\textbf{Non-asymptotic learning rates. } \mathfrak s_{\mathcal H,\tau}:=\frac{\mathsf{d}_{\mathsf{rep}}+\log(|\mathcal H|T/\delta)}{\tau|\mathcal H|} \text{ and }
\mathfrak s_\tau:=\frac{\mathsf{d}_{\mathsf{head}}+\log(|\mathcal H|T/\delta)}{\tau}.
\end{align*}
\vspace{-0.4cm}
\begin{align*}
\textbf{Initialization error. } \Delta_0&:=L_{\mathcal H}(\{\mathbf{h}_i^{(1)}\}_i,\mathbf{B}^{(0)})-L_{\mathcal H}(\{\mathbf{h}_i^\star\}_i,\mathbf{B}^\star)
+\frac{1}{L_1}\left\|\nabla_{\mathbf{B}}L_{\mathcal H}(\{\mathbf{h}_i^{(1)}\}_i,\mathbf{B}^{(0)})\right\|^2\\ &+(1+\kappa)\Delta_{\mathbf{B}}^{(0)},
\end{align*}
where $\mathsf{d}_{\mathsf{rep}} = p$ and $\mathsf{d}_{\mathsf{head}} = p \vee rq$. We also define $p_\sigma:=(H\bar\phi+\sigma)^2+\sigma^2$ and $\Delta_{\mathbf{B}}^{(0)}:=\|\mathbf{B}^{(0)}-\mathbf{B}^\star\|^2$. The rate $\mathfrak s_{\mathcal H,\tau}$ will capture the complexity of learned the common representation with finite number of data samples, while $\mathfrak s_\tau$ will capture the local head error when fitting the learned representation. Throughout this section, Assumptions \ref{assm:subg}-\ref{assump:unbiased} hold, the server uses an
$(f,\kappa)$-robust aggregator with $f<n/2$, and the local head update is the empirical risk minimizer described  (i.e., in \eqref{alternating_descent}). 

Our analysis is local, we assume that the initial representation is sufficiently close to $\mathbf{B}^\star$ and that  all iterates remain in $\mathcal N(\mathbf{B}^\star,\rho)$, where
$\rho:=\min\{\rho_0,\mu_3/\bar{p}\}$ and
$\bar{p}:=\frac{3}{2}\bar JH^2L_4+\frac{1}{2}H^2L_4^2\rho_0$, and that
$\mu_3\geq 4L_3H\bar J\bar\phi/\mu_1$. Finally, for a constant $C>0$, define the burn-in sample size
\begin{align}\label{eq:main-sample-burnin}
\tau_{\mathsf{burn\text{-}in}}:=\max\left\{\frac{16C^2\bar\phi^4}{\mu_1^2},\frac{4C\bar\phi^2}{\mu_1}\right\}\log(2r|\mathcal H|T/\delta),
\end{align}
and assume $\tau\geq\tau_{\mathsf{burn\text{-}in}}$. This guarantees that the empirical feature covariances from the $T$ rounds are well-conditioned. Below, we use $\lesssim$ to omit constants depending only on the fixed problem parameters $(\bar J,H,\bar\phi,\mu_1,\mu_3,L_{1:4})$. In addition, for conciseness, we define
\begin{align}\label{eq:main-stepsize-condition}
\bar\eta:=\min\left\{\frac{1}{24L_1},
\frac{\mu_3}{8\bar J^2H^2q\left(2H^2\bar J^2+16H^2\bar J^2\bar\phi^4/\mu_1^2\right)},
\frac{1}{32\bar J^2H^2q\bar\phi^2\mu_3},\frac{1}{2\mu_3q},\frac{8}{\mu_3}\right\},
\end{align}
and for any $\eta\leq\bar\eta$, we set $\beta^2=1-24\eta L_1$ and define $\bar a:=\beta^2(1+\eta L_1)(1+4\eta L_1)$ with $\bar b:=4\eta L_1(1+\eta L_1)\beta^2$. We also note that $1-\bar a\asymp\eta L_1$ and $\bar b/(1-\bar a)\asymp1$. In addition, the robust coefficient $\kappa = \mathcal{O}(f/n)$, with $f/n<1/2$, is assumed to satisfy
\begin{align}\label{eq:main-kappa-condition}
\kappa\leq \underline{\kappa}:= C_\kappa \left[\bar J^4qH^4\left(1+\frac{\bar\phi^4}{\mu_1^2}\right)
\left(\frac{1}{\mu_3}+\frac{1}{\mu_3^2}\right)\right]^{-1}, \text{ for some } C_\kappa > 0.
\end{align}

We note that for any $(f,\kappa)$-robust aggregator when combined with nearest-neighbor mixing (NNM) \citep{allouah2023fixing}, its robustness coefficient satisfies $\kappa=\mathcal{O}(f/n)$, which is information-theoretically optimal \citep{allouah2023fixing}. Therefore, the condition in \eqref{eq:main-kappa-condition} is satisfied when $f/n$ is sufficiently small. We also emphasize that conditions on $\kappa$ are also required in \cite{karimireddy2020byzantine,allouah2023tight}.

\begin{lemma}[Gradient Heterogeneity Bound]\label{theorem:gradient_heterogeneity_nonlinear_mb}
Suppose that $\eta\leq\bar\eta$, $\tau \geq \tau_{\mathsf{burn\text{-}in}}$, and $\kappa \leq \underline{\kappa}$. Then, for every $\delta\in(0,1)$, with probability at least $1-\delta$, it holds that\vspace{-0.05cm}
\begin{align}
\bar G_T \lesssim \frac{\Delta_{\mathbf{B}}^{(0)}}{\eta T} +p_\sigma\mathfrak s_{\mathcal H,\tau}+\sigma^2\mathfrak s_\tau +\eta\left(1+\frac{1}{|\mathcal H|}\right)\sigma_{\mathsf{g}}^2\log(T/\delta)+\bar Q_T.
\end{align}
\vspace{-0.5cm}
\end{lemma}
\noindent \textbf{Discussion:} The bound separates initialization, finite-sample, and stochastic-gradient errors. In particular, the heterogeneity of the honest representation gradients decreases with representation learning through $\Delta_{\mathbf{B}}^{(0)}/(\eta T)$ and $\bar Q_T$. That is, it is not controlled by a model parameter heterogeneity bias as in typical adversarially robust FL \citep{allouah2023fixing}. Thus, after fitting the personalized heads, honest clients agree on the direction used to learn the common representation up to errors that go away with more communication rounds $T$ or data samples $\tau$. The proof is in Appendix \ref{appendix:Nonlinear rep}.

Importantly, our analysis does not require a uniform gradient-dissimilarity condition, such as the $G$- or $(G,B)$-dissimilarity assumptions commonly used in heterogeneous FL \citep{li2020federated,karimireddy2020scaffold,allouah2023tight,gorbunov2023variance}. Instead, building on the analysis of the gradient heterogeneity dynamics in adversarially robust FL from \citep{toso2026gradient}, we directly control the trajectory of $\bar G_T$ in this more intricate setting of nonlinear representation learning.

\begin{theorem}[Convergence Bound]\label{nonlinear_convergence_mb}
Suppose that the conditions of Lemma \ref{theorem:gradient_heterogeneity_nonlinear_mb} hold. Then, for every $\delta\in(0,1)$, with probability at least $1-\delta$, it holds that
\begin{align}
\bar Q_T \lesssim \frac{\Delta_0}{\eta T}
+\kappa p_\sigma \, \mathfrak s_{\mathcal H,
\tau}+\kappa\sigma^2 \, \mathfrak s_\tau
+\eta\left(\kappa+\kappa^2+\frac{1+\kappa}{|\mathcal H|}\right)\sigma_{\mathsf{g}}^2\log(T/\delta).
\end{align}
\vspace{-0.1cm}
In addition, suppose that $\eta=\min\{\bar{\eta},1/\sqrt T\}$, then it holds that \vspace{-0.01cm}
\begin{align}
\bar Q_T \lesssim \frac{\Delta_0}{\sqrt{T}}
+\kappa p_\sigma \, \mathfrak s_{\mathcal H,
\tau}+\kappa\sigma^2 \, \mathfrak s_\tau
+\left(\kappa+\kappa^2+\frac{1+\kappa}{|\mathcal H|}\right)\frac{\sigma_{\mathsf{g}}^2\log(T/\delta)}{\sqrt{T}}.
\end{align}\vspace{-0.3cm}
\end{theorem}
\noindent \textbf{Discussion:} In our ergodic convergence bound, the optimization error and stochastic-gradient term both decrease as $1/\sqrt{T}$. The statistical terms vanish with $\tau$ and are inflated by $\kappa$. Therefore, under \eqref{eq:main-kappa-condition}, adversarial clients affect the rate through aggregation robustness coefficient $\kappa$ that typically scales as $f/n$, but most importantly, it does not introduce an irreducible model parameter heterogeneity bias, therefore demonstrating the benefit of learning what is common and personalizing what is specific. We provide the proof in Appendix \ref{sec:convergence_nonlinear_rep}.\\

\noindent \textbf{Key takeaway:} Partial personalization is known to mitigate heterogeneity in federated optimization \citep{mishchenko2023partially}. We rigorously characterize this benefit in adversarially robust FL with nonlinear representation learning by establishing high-probability, non-asymptotic guarantees under arbitrary adversarial updates. By separating the client-specific heads from the common representation, we demonstrate that all remaining errors decrease with the number of rounds $T$, the number of local samples $\tau$, or the number of honest clients $|\mathcal H|$. We emphasize that to the best of our knowledge this is the first time such analysis is provided for adversarially robust FL. In contrast to single-model Byzantine-robust FL, our guarantees have no irreducible model parameter heterogeneity bias.

\begin{corollary}[Parameter Recovery Error~Bound]\label{nonlinear_parameter_recovery_mb}
Suppose the conditions of Theorem \ref{nonlinear_convergence_mb}. Then, for every
$\delta\in(0,1)$, with probability at least $1-2\delta$, it holds that \vspace{-0.1cm}
\begin{align}\label{eq:parameter_recovery_mb}
\bar E_T&\lesssim \frac{\Delta_0}{\sqrt T}+(1+\kappa+\kappa^2)\left(p_\sigma \, \mathfrak s_{\mathcal H,\tau}+\sigma^2 \, \mathfrak s_\tau\right) +\left(\kappa+\kappa^2+\frac{1+\kappa}{|\mathcal H|}\right)
\frac{\sigma_{\mathsf{g}}^2\log(T/\delta)}{\sqrt T}.
\end{align}\vspace{-0.3cm}
\end{corollary}

\noindent \textbf{Discussion:} With $\eta$ of order $1/\sqrt{T}$, the initialization and stochastic-gradient terms vanish at the standard ergodic rate \citep{allouah2023fixing}. When $\kappa=0$, corresponding to the absence of adversarial clients, the second term in the bound above reduces to $p_\sigma\mathfrak s_{\mathcal H,\tau}+\sigma^2\mathfrak s_\tau.$
This is consistent with the statistical recovery guarantees for nonlinear representation learning in multitask settings established in \citep{zhang2024guarantees}: $\mathfrak s_{\mathcal H,\tau}$ captures the complexity of collaboratively learning the common representation, whereas $\mathfrak s_\tau$ captures the client-specific complexity of fitting the personalized heads on that representation. Our setting is more intricate as it additionally accounts for arbitrary adversarial clients and stochastic updates aggregated. The proof is in Appendix \ref{sec:parameter_recovery_nonlinear}.

\section{Multiclass classification}
\label{section:multiclass}

We now demonstrate that the nonlinear regression guarantee can be transferred to classification. For this, we consider a $q$-class problem. Let $C_{i,k}\in[q]$ denote the categorical label associated with $X_{i,k}$, and let $e_c\in\mathbb R^q$ denote the $c$-th canonical basis vector. We set $Y_{i,k}=e_{C_{i,k}}$ and train the raw class scores using the multiclass squared loss given by\vspace{-0.1cm}
\begin{equation*}
\begin{aligned}
M_i(\mathbf{h},\mathbf{B}):=\frac{1}{\tau}\sum_{k=1}^{\tau}
\left\|e_{C_{i,k}}-\mathbf{h}\phi_{\mathbf{B}}(X_{i,k})\right\|^2.
\end{aligned}\\[-0.1cm]    
\end{equation*}
This square loss with vector-valued class encodings is a standard classification surrogate and has been studied both empirically and theoretically in \cite{hui2021evaluation,hu2022understanding,vigogna2022multiclass}. In particular, we define the conditional class-probability vector
\begin{align*}
\pi_i(X):=\mathbb E[Y_i\mid X]
=\big(\mathbb P(C_i=1\mid X),\ldots,\mathbb P(C_i=q\mid X)\big)^\top,
\end{align*}
and suppose that it is realizable by the shared representation, namely $\pi_i(X)=\mathbf{h}_i^\star\phi_{\mathbf{B}^\star}(X).$
For any score vector $z\in\mathbb R^q$, the conditional squared risk decomposes as follows:
\begin{align*}
\mathbb E[\|Y_i-z\|^2\mid X]
=\mathbb E[\|Y_i-\pi_i(X)\|^2\mid X]+\|z-\pi_i(X)\|^2,
\end{align*}
such that its population minimizer is precisely the class-probability vector. In addition, we note that $V_i:=Y_i-\pi_i(X)$ is conditionally mean-zero and bounded, and thus it satisfies Assumption \ref{assm:subg}. Thus, the multiclass squared-loss problem is a direct instance of our regression model analyzed above. 

For a vector $z\in\mathbb{R}^q$, $[z]_c$ denotes its $c$-th coordinate. To make each classifier single-valued, we use the same deterministic tie-breaking convention in both definitions, i.e., if several classes attain the maximum score, we select the class with the smallest index. The learned and Bayes classifiers are\vspace{-0.1cm}
\begin{align*}
\widehat C_i^{(t+1)}(X)
&:=\min\left(\argmax_{c\in[q]}\left[\mathbf{h}_i^{(t+1)}\phi_{\mathbf{B}^{(t)}}(X)\right]_c\right) \text{ and }C_i^\star(X):=\min\left(\argmax_{c\in[q]}\pi_{i,c}(X)\right),
\end{align*}
where $\pi_{i,c}(X):= \mathbb{P}(C_i = c \mid X).$ 

Let $\bar{\mathcal C}_T
:=\frac{1}{T|\mathcal H|\tau}
\sum_{t=0}^{T-1}\sum_{i\in\mathcal H}\sum_{k=1}^{\tau}
\mathbf 1\left\{\widehat C_i^{(t+1)}(X_{i,k})\neq C_i^\star(X_{i,k})\right\}$ be the error rate between the learned and optimal Bayes classifiers.

\begin{corollary}\label{cor:multiclass-classification} Suppose the conditions of Theorem  \ref{theorem:convergence_nonlinear} hold and suppose that there exists a margin parameter $\gamma>0$ such that, for every honest client $i \in \mathcal{H}$, \vspace{-0.1cm}
\begin{align}\label{eq:classification-margin}
\pi_{i,C_i^\star(X)}(X)-\max_{c\neq C_i^\star(X)}\pi_{i,c}(X)\geq\gamma
\end{align}\vspace{-0.1cm}
almost surely.  Then, with probability at least $1-2\delta$, it holds that
\begin{align}\label{eq:classification-recovery}
\bar{\mathcal C}_T\leq\frac{2 \bar E_T}{\gamma^2} \lesssim \frac{\Delta_0}{\gamma^2\sqrt T}
+\frac{(1+\kappa+\kappa^2)}{\gamma^2}\left(p_\sigma\mathfrak s_{\mathcal H,T}+\sigma^2\mathfrak s_T\right) +\left(\kappa+\kappa^2+\frac{1+\kappa}{|\mathcal H|}\right)
\frac{\sigma_{\mathsf{g}}^2\log(T/\delta)}{\gamma^2 \sqrt T}.
\end{align}\vspace{-0.4cm}
\end{corollary}
\noindent \textbf{Discussion.} The margin condition (i.e., inequality (\ref{eq:classification-margin})) ensures that the predicted class remains unchanged whenever the score-estimation error is sufficiently small relative to the separation between the most likely class and its closest competitor. Therefore, $\widehat C_i^{(t+1)}(X_{i,k})\neq C_i^\star(X_{i,k})$ can occur only if the learned class-probability scores incur a squared prediction error of order at least $\gamma^2$. Therefore, Corollary~\ref{cor:multiclass-classification} converts the parameter recovery error guarantee (\eqref{eq:parameter_recovery_mb})  into a classification guarantee: the excess error rate is bounded by the parameter recovery error, with the same dependence on the number of rounds $T$, local sample size $\tau$, number of honest clients $|\mathcal{H}|$, stochastic-gradient variance $\sigma_{\mathsf{g}}^2$ and label noise variance $\sigma^2$, up to $\gamma^{-2}$. The proof is in Appendix \ref{sec:classification-proof}.

\section{Experiments}

We now validate our theory on three heterogeneous federated datasets\footnote{Code to reproduce our results can be found at \url{https://github.com/LeoToso/Byz-robust-nonlinear-rep}.}: CIFAR-10 image classification \citep{krizhevsky2009learning}, FEMNIST character classification from the LEAF benchmark \citep{caldas2018leaf}, and School Exam Score regression \citep{zhou2011malsar,li2015multi}. CIFAR-10 is partitioned into $100$ clients, each with examples from two classes. FEMNIST keeps its natural writer-based partition. For the School Exam Score dataset each of the $139$ schools defines one regression client. We evaluate NNM-preprocessed Krum and coordinate-wise trimmed mean \citep{allouah2023fixing} under ALIE and Mimic attacks. Full details are provided in Appendix~\ref{sec:numerical-details}.\\

\noindent \textbf{Baseline.} We compare against the standard common-model formulation of Byzantine-robust FL \citep{karimireddy2020byzantine,allouah2023fixing,toso2026gradient}. The baseline robustly aggregates updates of one model shared by every client, whereas our approach aggregates only the common-representation updates and keeps the linear heads personalized.\\

\begin{table}[t]
\centering
\caption{Average local test accuracy (\%) on CIFAR-10 and FEMNIST with $50$ honest and $5$ Byzantine clients per round. Entries report mean $\pm$ standard deviation over three seeds.}
\label{tab:classification-results}
\small
\setlength{\tabcolsep}{3.5pt}
\begin{tabular}{lllcccc}
\toprule
& & & \multicolumn{2}{c}{NNM+Krum} & \multicolumn{2}{c}{NNM+TrMean} \\
\cmidrule(lr){4-5}\cmidrule(lr){6-7}
Dataset & Loss & Method & ALIE & Mimic & ALIE & Mimic \\
\midrule
CIFAR-10 & Cross-entropy & Baseline & $50.46\pm0.57$ & $46.39\pm1.76$ & $50.46\pm0.71$ & $47.17\pm1.58$ \\
\rowcolor{oursrow}
& Cross-entropy & Rep. learning & $\mathbf{88.79\pm0.42}$ & $\mathbf{88.50\pm0.59}$ & $\mathbf{88.76\pm0.44}$ & $\mathbf{88.54\pm0.61}$ \\
\rowcolor{oursrow}
& Multiclass & Rep. learning & $88.33\pm0.53$ & $87.95\pm0.51$ & $88.42\pm0.52$ & $87.89\pm0.60$ \\
\midrule
FEMNIST & Cross-entropy & Baseline & $72.95\pm1.95$ & $70.34\pm1.23$ & $73.09\pm2.03$ & $70.66\pm1.47$ \\
\rowcolor{oursrow}
& Cross-entropy & Rep. learning & $91.89\pm0.31$ & $92.43\pm0.15$ & $92.02\pm0.28$ & $92.51\pm0.07$ \\
\rowcolor{oursrow}
& Multiclass & Rep. learning & $\mathbf{93.56\pm0.24}$ & $\mathbf{93.60\pm0.07}$ & $\mathbf{93.57\pm0.22}$ & $\mathbf{93.67\pm0.13}$ \\
\bottomrule
\end{tabular}
\vspace{-0.3cm}
\end{table}

\noindent \textbf{Classification.} Table~\ref{tab:classification-results} reveals the fundamental limitation of adversarially robust FL when learning a common model parameter under heterogeneous honest clients. With cross-entropy, representation learning improves CIFAR-10 accuracy by $38.3$--$42.1$ percentage points and FEMNIST accuracy by $18.9$--$22.1$ points across all aggregator-attack pairs. The benefit persist for two distinct robust aggregators and attacks. This is precisely the result in Lemma~\ref{theorem:gradient_heterogeneity_nonlinear_mb}: after personalization, the honest-update heterogeneity is tied to representation recovery and statistical errors instead of a persistent model-heterogeneity bias.

The multiclass squared loss also attains $87.89$--$88.42\%$ on CIFAR-10 and $93.56$--$93.67\%$ on FEMNIST. On FEMNIST, it improves over cross-entropy by $1.16$--$1.65$ points. On CIFAR-10, the difference remains below $0.7$ points. These results directly complement Corollary~\ref{cor:multiclass-classification}: the regression guarantees can be translated to classification.

\begin{figure*}[t]
\centering
\includegraphics[width=\textwidth]{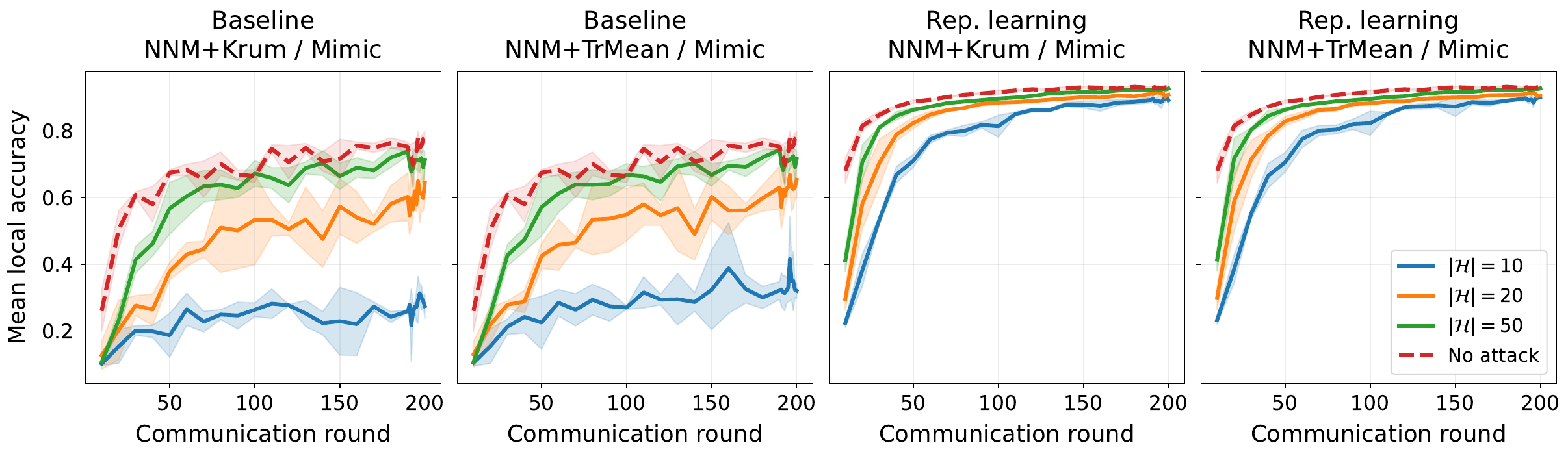}
\vspace{-0.6cm}
\caption{FEMNIST accuracy under the Mimic attack using cross-entropy loss and a fixed number of five Byzantine clients. Each curve reports the mean accuracy and standard deviation across three random seeds as the number of honest clients per round increases.}
\label{fig:femnist_cross_entropy}
\end{figure*}

Figure~\ref{fig:femnist_cross_entropy} shows the effect of increasing the number of honest clients $|\mathcal{H}|$ under the Mimic attack. For both aggregators, representation learning approaches the no-attack curve as $|\mathcal H|$ increases and remains near $90\%$ accuracy even with only ten honest updates. In contrast, the baseline is substantially below the no-attack curve and deteriorates sharply for small $|\mathcal H|$. This supports our results in Theorem~\ref{nonlinear_convergence_mb} and Corollary~\ref{nonlinear_parameter_recovery_mb}, whose client-averaged statistical and stochastic-gradient terms decrease with $|\mathcal H|$, while no irreducible model-heterogeneity term remains.\\

\begin{table}[t]
\centering
\caption{School Exam test MSE using the squared loss with $20$ honest and $5$ Byzantine clients per round. Results report mean $\pm$ standard deviation over three seeds. Lower is better.}
\label{tab:school_h20}
\small
\setlength{\tabcolsep}{4pt}
\begin{tabular}{lcccc}
\toprule
& \multicolumn{2}{c}{NNM+Krum}
& \multicolumn{2}{c}{NNM+TrMean} \\
\cmidrule(lr){2-3}\cmidrule(lr){4-5}
Method & ALIE & Mimic & ALIE & Mimic \\
\midrule
Baseline
& $0.7611\pm0.0452$
& $0.6532\pm0.0201$
& $0.7672\pm0.0465$
& $0.6531\pm0.0226$ \\
\rowcolor{oursrow}
Rep.\ learning
& $\mathbf{0.6324\pm0.0272}$
& $\mathbf{0.6300\pm0.0276}$
& $\mathbf{0.6316\pm0.0275}$
& $\mathbf{0.6301\pm0.0276}$ \\
\bottomrule
\end{tabular}

\vspace{0.12cm}

\begin{tabular}{cc}
\toprule
Pooled mean predictor & Per-school mean predictor \\
\midrule
$0.9913\pm0.0306$ & $0.8826\pm0.0263$ \\
\bottomrule
\end{tabular}
\vspace{-0.4cm}
\end{table}

\noindent \textbf{Regression.} On the School Exam Score dataset, representation learning achieves an MSE of approximately $0.63$ in every adversarial configuration, reducing the common-model baseline error by $3.5\%$--$17.7\%$ and outperforming both constant predictors (i.e., pooled mean and per-school mean). The pooled-mean predictor assigns every example the average training label across all schools, whereas the per-school-mean predictor assigns each example the average training label of its corresponding school. These results support Corollary~\ref{nonlinear_parameter_recovery_mb}, that is, robust aggregation can operate on a genuinely shared representation, while the personalized heads absorb school-specific heterogeneity, improving regression without incurring an irreducible model-parameter heterogeneity bias.

\vspace{-0.2cm}
\section{Conclusion and Future Work} \label{sec:conclusions}
\vspace{-0.2cm}

We established rigorous high-probability, non-asymptotic guarantees for nonlinear representation learning in heterogeneous federated systems with adversarial clients. By separating each client model into a personalized head and a common nonlinear representation, we demonstrated that the heterogeneity among honest representation gradients is governed by representation-recovery, initialization error, finite-sample, and stochastic gradient terms, without an irreducible model-heterogeneity bias. Our bounds explicitly reveal how recovery improves with the number of communication rounds $T$, local samples $\tau$, and honest clients $|\mathcal{H}|$. We further demonstrated that these guarantees extend to multiclass classification and validated our theoretical guarantees on CIFAR-10, FEMNIST, and School Exam Score datasets.

We leave for future work to determine whether the restriction on the robustness coefficient $\kappa$ (i.e., Eq. \eqref{eq:main-kappa-condition}) is fundamental. In particular, \cite{toso2026gradient}, for single-model Byzantine-robust linear and nonlinear regression demonstrated that a restriction on $\kappa$ can instead be replaced by a sufficiently large sample burn-in condition. Establishing whether a similar guarantee also holds for our setting would clarify the statistical limits of personalization under adversarial clients.

\section{Acknowledgments} 
Leonardo F. Toso is funded by the Center for AI and Responsible Financial Innovation (CAIRFI) Fellowship and the Columbia Presidential Fellowship. James Anderson is partially funded by NSF grants EECS 2144634 and CNS 2535097 and the Center of AI Technology (CAIT) in collaboration with Amazon.

\bibliographystyle{alpha}
\bibliography{references}

\newpage
\appendix

\newpage 
\appendix
\addcontentsline{toc}{section}{Appendix}
\part{Appendix}
\parttoc 
\newpage

\section{Appendix Roadmap}\label{app:roadmap}
The appendix is organized as follows. In Appendix \ref{app:related_work}, we position our results relative to FL and adversarially robust representation learning in the related work. In Appendix \ref{app:technical_preliminaries} we collect the concentration inequalities and supporting results used throughout the analysis. These results complement the model and assumptions stated in Sections \ref{sec:problem_formulation}--\ref{sec:assumptions}.

Appendix \ref{app:proof_roadmap} provides a roadmap to the nonlinear representation-learning proof. The detailed heterogeneity and local-curvature analysis is developed in Appendix \ref{appendix:Nonlinear rep}. In  Appendix~\ref{sec:convergence_nonlinear_rep}we then establish the ergodic convergence guarantee, Appendix~\ref{sec:parameter_recovery_nonlinear} converts this result into representation-parameter and prediction-error recovery bounds, and Appendix \ref{sec:classification-proof} proves the multiclass classification extension. We also provide additional details on the experiments implementation in Appendix \ref{sec:numerical-details}.

\section{Related Work}
\label{app:related_work}

We position our results relative to federated optimization and representation learning under adversarial clients.\\

\noindent \textbf{Federated optimization and personalization.} FedAvg and its variants reduce communication by performing multiple local updates between communication rounds~\citep{mcmahan2023communicationefficient,konecny2016federated}. The negative effect of model parameter heterogeneity in FedAvg has also motivated other works such as FedProx \citep{li2020federated} and SCAFFOLD \citep{karimireddy2020scaffold} to alleviate such bottleneck. However, in the setting of adversarially robust FL under heterogeneous honest clients, the model heterogeneity still persists. On the other hand, personalized FL instead allows for the model to vary across clients \citep{fallah2020personalized,dinh2020personalized,collins2022fedavg,mishchenko2023partially,zhang2024sample}. Our work also considers the representation learning setting, but differs by analyzing a nonlinear representation under adversarial clients.\\

\noindent \textbf{Byzantine-robust learning.} Byzantine-robust FL replaces averaging updates with a robust aggregation rule \citep{blanchard2017machine,yin2018byzantine,mhamdi2018hidden,acharya2022robust}. Recent work by \cite{karimireddy2020byzantine,farhadkhani2022byzantine,allouah2023fixing} characterize robust aggregation through a deterministic robustness coefficient $\kappa$ (typically in the order of the fraction of adversarial clients). We also leverage robust aggregators, but make the honest-update heterogeneity endogenous to the representation recovery. That is, the personalization absorbs the heterogeneity among clients, while robust aggregation handles the adversarial updates in the shared representation. In addition, building on the analysis of \citep{toso2026gradient}, we characterize the dynamics of the representation-gradient heterogeneity along the optimization trajectory rather than imposing a uniform gradient-dissimilarity bound.  We refer the reader to \citep{guerraoui2024robust} for a comprehensive study of distributed robust machine learning.\\

\noindent \textbf{Representation learning.} The statistical benefit of learning a common representation across related tasks is well established \citep{collins2021exploiting,zhang2024sample}. Linear models achieve improved sample complexity under task diversity \citep{tripuraneni2020theory,du2020few}, which motivates federated representation learning \citep{collins2021exploiting}. More recent results accommodate nonlinear representations with non-identical and dependent data \citep{zhang2024guarantees}, while common linear representations have also been leveraged in imitation learning and adaptive control \citep{zhang2022multitask,lee2024regret}. Most relevant to our work is \cite{mishchenko2023partially} that establish Byzantine-robust convergence for partially personalized objectives whose honest-client operators share a common representation, but do not provide finite-sample statistical guarantees for it. In contrast, we consider nonlinear representation learning and demonstrate explicitly that the honest-gradient heterogeneity is controlled by representation-recovery and finite-sample errors.

\section{Technical Preliminaries}\label{app:technical_preliminaries}
In this section we summarize the auxiliary results used in the nonlinear representation-learning proofs. The federated model, optimization approach, robust-aggregation condition, and standing assumptions are already stated in the main body of the paper.

\subsection{Matrix Concentration Inequalities}\label{app:matrix_tools}
\begin{lemma}[Weyl's inequality]
\label{lem:weyl}
Let $A,{B} \in \mathbb{R}^{s\times s}$ be symmetric matrices. Then, for all $j=1,\dots,s$,
\begin{align*}
\lambda_j(A+{B}) \in \left[\lambda_j(A) + \lambda_{\min}({B}),\; \lambda_j(A) + \lambda_{\max}({B})\right].
\end{align*}
In particular, for the smallest eigenvalue,
\begin{align*}
\lambda_{\min}(A+{B}) \geq \lambda_{\min}(A) + \lambda_{\min}({B}).
\end{align*}
\end{lemma}

\begin{theorem}[Matrix Bernstein for self-adjoint sums (Theorem 1.4 of {\citep{tropp2012user}})]
\label{thm:matrix_bernstein}
Let $A_1,\dots,A_\tau\in\mathbb{R}^{s\times s}$ be independent, mean-zero, self-adjoint matrices. Assume $\|A_k\|_2 \leq M$ almost surely, and define
$$
v := \left\|\sum_{k=1}^{\tau}\mathbb{E}[A_k^2]\right\|_2 .
$$
Then for all $t\geq 0$, it holds that
\begin{align*}
\mathbb{P} \left(\left\|\sum_{k=1}^{\tau} A_k\right\|_2 \geq t\right)
\leq 2d \exp \left( -\frac{t^2}{2v + \frac{2}{3}Mt}
\right).    
\end{align*}
\end{theorem}

\subsection{Auxiliary Inequalities}
\label{app:aux-ineq}

We list here standard inequalities that will be invoked throughout our analysis.

\begin{lemma}[Triangle inequality and Jensen]
\label{lem:tri-jensen}
For any random vector $Z$,
$\|\mathbb{E}[Z]\|\le \mathbb{E}\|Z\|$ (Jensen).
For any vectors $a,b$, $\|a-b\|\leq \|a\|+\|b\|$ (triangle inequality).
\end{lemma}

\begin{lemma}[Cauchy--Schwarz]
\label{lem:cs}
For any nonnegative random variable $U$, $\mathbb{E}[U]\leq \sqrt{\mathbb{E}[U^2]}$.
\end{lemma}

\begin{lemma}[Young's inequality]
\label{lem:young}
For any $a,b \in \mathbb{R}$ and any $\gamma>0$,
\begin{align*}
2ab \leq \gamma a^2 + \frac{1}{\gamma} b^2.
\end{align*}
Equivalently, for any $u,v \in \mathbb{R}$, we have 
\begin{align*}
(u+v)^2
\leq
(1+\gamma)u^2 + \left(1+\frac{1}{\gamma}\right)v^2.
\end{align*}
\end{lemma}

\begin{lemma}[Averaging inequality]
\label{lem:avg-ineq}
For any matrices $A_1,\dots,A_\tau$ of the same dimensions,
\begin{align*}
\left\|\frac{1}{\tau}\sum_{k=1}^{\tau}A_k\right\|^2 \leq \frac{1}{\tau}\sum_{k=1}^{\tau}\|A_k\|^2.    
\end{align*}
\end{lemma}

\begin{proof}
The proof of this lemma follows from Jensen's inequality.
\end{proof}

\begin{lemma}
\label{lem:variance_decomposition}
Let $a_1, \dots, a_n \in \mathbb{R}^s$ and define their average
$$
\bar a := \frac{1}{n}\sum_{i=1}^n a_i.
$$
Then, it holds that
\begin{align*}
\frac{1}{n}\sum_{i=1}^n \|a_i - \bar a\|^2 = \frac{1}{n}\sum_{i=1}^n \|a_i\|^2 - \|\bar a\|^2.
\end{align*}
\end{lemma}

\subsection{Supporting Concentration Inequalities}

\begin{lemma}[Average of sub-Gaussian random matrices]
\label{lem:matrix_subg_mean}
Let $M_1,\dots,M_m \in \mathbb{R}^{u \times s}$ be independent, mean-zero random matrices. Assume they are sub-Gaussian in the sense that
\begin{align*}
\|M_j\|_{\psi_2} := \sup_{U \in \mathbb{S}_F^{u \times s}}
\|\langle U,M_j\rangle\|_{\psi_2} \leq K, \text{ for } j=1,\dots,m,
\end{align*}
where $\mathbb{S}_F^{u \times s} :=
\left\{U \in \mathbb{R}^{u \times s}:\|U\|_F = 1\right\},$
and $\langle U,M\rangle := \mathrm{tr}(U^\top M)$ is the Frobenius inner product.

Then there exists a constant $C_1>0$ such that, for every $\delta \in (0,1)$, with probability at least $1-\delta$, it holds that
\begin{align*}
\left\|\frac{1}{m}\sum_{j=1}^m M_j\right\|_F\leq C_1 K \sqrt{\frac{us+\log(1/\delta)}{m}}.
\end{align*}
\end{lemma}

\begin{proof}
We begin by setting
\begin{align*}
\bar M := \frac{1}{m}\sum_{j=1}^m M_j.
\end{align*}
Fix $\varepsilon \in (0,1)$ and let $\mathcal{N} \subset \mathbb{S}_F^{u\times s}$ be an $\varepsilon$-net of the Frobenius unit sphere. Since $\mathbb{R}^{u\times s}$ is isometric to $\mathbb{R}^{us}$ under the Frobenius norm, by Corollary 4.2.11 in \citep{vershynin2018high}, one can choose $\mathcal{N}$ such that
\begin{align*}
|\mathcal{N}| \leq \left(\frac{3}{\varepsilon}\right)^{pq}.
\end{align*}

Therefore, by using a standard net argument, we can write 
\begin{align*}
\|\bar M\|_F =\sup_{U \in \mathbb{S}_F^{u\times s}} \langle U,\bar M\rangle
\leq\frac{1}{1-\varepsilon} \max_{U\in\mathcal{N}} \langle U,\bar M\rangle.
\end{align*}

Now let us fix $U \in \mathcal{N}$. The scalar random variables then
\begin{align*}
Z_j(U) := \langle U,M_j\rangle
\end{align*}
are independent, mean-zero, and sub-Gaussian. Moreover,
\begin{align*}
\|Z_j(U)\|_{\psi_2} \leq \|M_j\|_{\psi_2} \leq K.
\end{align*}

By applying Hoeffding's inequality for sub-Gaussian random variables \citep[Theorem 2.2.1]{vershynin2018high} with weights $a_j = 1/m$, we obtain
\begin{align*}
\mathbb{P}\left( \left| \frac{1}{m}\sum_{j=1}^m Z_j(U) \right| \geq t \right) \leq 2\exp\left( -c\frac{m t^2}{K^2}\right), \text{ for } t\geq 0,
\end{align*}
for a constant $c>0$.

Then, by taking a union bound over $U \in \mathcal{N}$, we obtain
\begin{align*}
\mathbb{P}\left( \max_{U\in\mathcal{N}} |\langle U,\bar M\rangle| \geq t
\right) \leq 2|\mathcal{N}| \exp\left( -c\frac{m t^2}{K^2} \right).
\end{align*}

We proceed, by choosing $t$ so that the right-hand side is at most $\delta$, namely
\begin{align*}
c\frac{m t^2}{K^2} \geq\log\frac{2}{\delta} + pq\log\frac{3}{\varepsilon}.
\end{align*}
Thus, with probability at least $1-\delta$, we have that
\begin{align*}
\max_{U\in\mathcal{N}} |\langle U,\bar M\rangle|
\leq \frac{K}{\sqrt{c}} \sqrt{ \frac{\log(2/\delta)+pq\log(3/\varepsilon)}{m}},
\end{align*}
Then, by taking $\varepsilon = 1/2$ and absorbing constants into a constant $C_1$ yields
\begin{align*}
\|\bar M\|_F \leq C_1 K \sqrt{\frac{us+\log(1/\delta)}{m}},
\end{align*}
which completes the proof.
\end{proof}

\subsection{Supporting Results for Representation Momentum}\label{subsec:auxiliary_momentum}

We denote by $\{\mathcal{F}_t\}_{t\geq 0}$ the natural filtration generated by the algorithm up to iteration $t$, defined as
\begin{align*}
\mathcal{F}_t:= \sigma\left( \mathbf{B}^{(0)},\ldots,\mathbf{B}^{(t)},
m_i^{(0)},\ldots,m^{(t-1)}_{i}: i=1,\ldots,n \right).\end{align*}
By convention, $\mathcal{F}_0 = \sigma(\mathbf{B}^{(0)})$. We denote by $\mathbb{E}_t[\cdot] := \mathbb{E}[\cdot \mid \mathcal{F}_t]$ the conditional expectation given $\mathcal{F}_t$, and by $\mathbb{E}[\cdot]$ the total expectation. Thus,
\begin{align*}
\mathbb{E}[Z] = \mathbb{E}\big[\mathbb{E}[Z\mid \mathcal{F}_t]\big].
\end{align*}

\begin{lemma}[Conditional Markov inequality]
\label{lem:markov}
Let $Z$ be a non-negative random variable and let $\mathcal{F}$ be a sigma-algebra. Then, for any $a>0$,
\begin{align*}
\mathbb{P}(Z \ge a \mid \mathcal{F})
\le
\frac{\mathbb{E}[Z \mid \mathcal{F}]}{a}.
\end{align*}
\end{lemma}

\begin{proof}
Since $Z \ge 0$, we have
\begin{align*}
Z \ge a \mathbf{1}_{\{Z \ge a\}}.
\end{align*}
Taking conditional expectation with respect to $\mathcal{F}$ on both sides yields
\begin{align*}
\mathbb{E}[Z \mid \mathcal{F}]
\geq a \mathbb{E}[\mathbf{1}_{\{Z \ge a\}} \mid \mathcal{F}]
= a \mathbb{P}(Z \ge a \mid \mathcal{F}),
\end{align*}
and by dividing both sides by $a>0$ we conclude the proof.
\end{proof}

\begin{lemma}[Control of stochastic gradient]
\label{lem:noise_hp}
Suppose Assumption \ref{assump:subgaussian_gradient_noise} holds. Then there exists a constant $C_2>0$ such that, for every $t\ge 0$ and every $\delta \in (0,1)$, conditioned on $\mathcal F_t$, with probability at least $1-\delta$, the following bounds hold simultaneously:
\begin{align*}
\frac{1}{|\mathcal H|}\sum_{i\in\mathcal H}
\|\widehat{\nabla}_{\mathbf{B}}L_i^{(t)}-\nabla_{\mathbf{B}} L_i(\mathbf{h}_i^{(t+1)},\mathbf{B}^{(t)})\|^2
&\leq
C_2 \sigma_{\mathsf{g}}^2 \log\left(\frac{2}{\delta}\right), \\
\|\widehat{\nabla}_{\mathbf{B}}L_{\mathcal H}^{(t)}-\nabla_{\mathbf{B}} L_{\mathcal H}(\{\mathbf{h}_i^{(t+1)}\}_i,\mathbf{B}^{(t)})\|^2
&\leq
C_2 \frac{\sigma_{\mathsf{g}}^2}{|\mathcal H|}\log\left(\frac{2}{\delta}\right).
\end{align*}
\end{lemma}

\begin{proof}
Fix $t\ge 0$, and define 
\begin{align*}
\nu_i^{(t)} := \widehat{\nabla}_{\mathbf{B}}L_i^{(t)}-\nabla_{\mathbf{B}} L_i(\mathbf{h}_i^{(t+1)},\mathbf{B}^{(t)}),\; i\in\mathcal H,
\end{align*}
and
\begin{align*}
\bar \nu^{(t)} := \frac{1}{|\mathcal H|}\sum_{i\in\mathcal H}\nu_i^{(t)}
= \widehat{\nabla}_{\mathbf{B}}L_{\mathcal H}^{(t)}-\nabla_{\mathbf{B}} L_{\mathcal H}(\{\mathbf{h}_i^{(t+1)}\}_i,\mathbf{B}^{(t)}).
\end{align*}

We note that by Assumption \ref{assump:subgaussian_gradient_noise}, conditioned on $\mathcal F_t$, the vectors $\{\nu_i^{(t)}\}_{i\in\mathcal H}$ are independent, mean-zero, and satisfy
\begin{align*}
\mathbb E_t\left[\exp\left(\frac{\|\nu_i^{(t)}\|^2}{\sigma_{\mathsf{g}}^2}\right)\right]\leq \exp(1)
 \text{ for all } i\in\mathcal H.
\end{align*}

Let us then proceed to control of the average squared noise. Define
\begin{align*}
S_t := \frac{1}{|\mathcal H|}\sum_{i\in\mathcal H}\|\nu_i^{(t)}\|^2.
\end{align*}
Since the exponential function is convex, Jensen's inequality yields
\begin{align*}
\exp\left(\frac{S_t}{\sigma_{\mathsf{g}}^2}\right)
=\exp\left(
\frac{1}{|\mathcal H|}\sum_{i\in\mathcal H}\frac{\|\nu_i^{(t)}\|^2}{\sigma_{\mathsf{g}}^2}
\right)\leq \frac{1}{|\mathcal H|}\sum_{i\in\mathcal H}
\exp\left(\frac{\|\nu_i^{(t)}\|^2}{\sigma_{\mathsf{g}}^2}\right).
\end{align*}
Taking conditional expectation and using the assumption gives
\begin{align*}
\mathbb E_t\left[\exp\left(\frac{S_t}{\sigma_{\mathsf{g}}^2}\right)\right]
\le
\frac{1}{|\mathcal H|}\sum_{i\in\mathcal H}
\mathbb E_t\left[\exp\left(\frac{\|\nu_i^{(t)}\|^2}{\sigma_{\mathsf{g}}^2}\right)\right]
\leq \exp(1).
\end{align*}

Hence, by conditional Markov's inequality (Lemma \ref{lem:markov}), for any $u>0$,
\begin{align*}
\mathbb P_t\left(S_t\ge u\right) &=
\mathbb P_t\left( \exp\left(\frac{S_t}{\sigma_{\mathsf{g}}^2}\right)
\geq \exp\left(\frac{u}{\sigma_{\mathsf{g}}^2}\right)
\right) \\
&\leq \exp\left(-\frac{u}{\sigma_{\mathsf{g}}^2}\right)
\mathbb E_t\left[\exp\left(\frac{S_t}{\sigma_{\mathsf{g}}^2}\right)\right]
\leq \exp\left(1-\frac{u}{\sigma_{\mathsf{g}}^2}\right).
\end{align*}

Therefore, by choosing
\begin{align*}
u = \sigma_{\mathsf{g}}^2\left(1+\log\left(\frac{2}{\delta}\right)\right)
\end{align*}
we obtain
\begin{align*}
\mathbb P_t\left(
S_t \ge \sigma_{\mathsf{g}}^2\left(1+\log\left(\frac{2}{\delta}\right)\right)
\right)
\le \frac{\delta}{2}.
\end{align*}
Then, as  $1+\log(2/\delta)\leq 2\log(2/\delta)$ for $\delta\in(0,1)$, we conclude that with conditional probability at least $1-\delta/2$, we have
\begin{align*}
S_t = \frac{1}{|\mathcal{H}|}\sum_{i\in\mathcal H}\|\nu_i^{(t)}\|^2
\leq 2\sigma_{\mathsf{g}}^2\log\left(\frac{2}{\delta}\right).
\end{align*}

Finally, we consider controlling the honest average noise. For this, we have
\begin{align*}
\bar \nu^{(t)} = \frac{1}{|\mathcal{H}|}\sum_{i\in\mathcal H}\nu_i^{(t)}.
\end{align*}

As the vectors $\{\nu_i^{(t)}\}_{i\in\mathcal H}$ are conditionally independent, mean-zero, and conditionally sub-Gaussian, concentration for sums of independent sub-Gaussian random vectors \cite{vershynin2018high} implies that there exists a constant $C_3>0$ such that, conditioned on $\mathcal F_t$, for every $\delta\in(0,1)$,
\begin{align*}
\mathbb P_t\left( \|\bar \nu^{(t)}\|^2 \geq
C_3 \frac{\sigma_{\mathsf{g}}^2}{|\mathcal{H}|}\log\left(\frac{2}{\delta}\right) \right)
\leq \frac{\delta}{2}.
\end{align*}
As we have
\begin{align*}
\bar \nu^{(t)} =
\widehat{\nabla}_{\mathbf{B}}L_{\mathcal H}^{(t)}-\nabla_{\mathbf{B}} L_{\mathcal H}(\{\mathbf{h}_i^{(t+1)}\}_i,\mathbf{B}^{(t)}),
\end{align*}
we obtain that with conditional probability at least $1-\delta/2$,
\begin{align*}
\|\widehat{\nabla}_{\mathbf{B}}L_{\mathcal H}^{(t)}-\nabla_{\mathbf{B}} L_{\mathcal H}(\{\mathbf{h}_i^{(t+1)}\}_i,\mathbf{B}^{(t)})\|^2
\leq C_3 \frac{\sigma_{\mathsf{g}}^2}{|\mathcal{H}|}\log\left(\frac{2}{\delta}\right).
\end{align*}

By a union bound, both events hold simultaneously with conditional probability at least
\begin{align*}
1-\frac{\delta}{2}-\frac{\delta}{2}=1-\delta.
\end{align*}
Therefore, with conditional probability at least $1-\delta$, we have that
\begin{align*}
\frac{1}{|\mathcal{H}|}\sum_{i\in\mathcal H}\|\xi_i^{(t)}\|^2
&\leq
2\sigma_{\mathsf{g}}^2\log\left(\frac{2}{\delta}\right), \\
\|\bar \xi^{(t)}\|^2
&\leq
C_3 \frac{\sigma_{\mathsf{g}}^2}{|\mathcal{H}|}\log\left(\frac{2}{\delta}\right).
\end{align*}
Recalling the definitions of $\nu_i^{(t)}$ and $\bar \nu^{(t)}$, the result follows by taking
\begin{align*}
C_2:=\max\{2,C_3\},
\end{align*}
completes the proof.
\end{proof}

We proceed, by letting $\bar{m}^{(t)}:= \frac{1}{|\mathcal{H}|}\sum_{i \in \mathcal{H}}m_{\mathbf{B},i}^{(t)}$ be the average of the honest client momenta.

\begin{lemma}
\label{lem:conditional_disagreement}
Suppose Assumption \ref{assump:subgaussian_gradient_noise} holds. Define
\begin{align*}
D^{(t)} &:= \frac{1}{|\mathcal H|}\sum_{i\in\mathcal H}\|m_{\mathbf{B},i}^{(t)}-\bar m_{\mathbf{B}}^{(t)}\|^2, \; D_T := \frac{1}{T}\sum_{t=0}^{T-1}D^{(t)}, 
\end{align*}
and $\bar G^{(t)}:= \frac{1}{|\mathcal{H}|}\sum_{i \in \mathcal{H}} \|\nabla_{\mathbf{B}} L_i(\mathbf{h}_i^{(t+1)},\mathbf{B}^{(t)}) - \nabla_{\mathbf{B}} L_{\mathcal{H}}(\{\mathbf{h}_i^{(t+1)}\}_i,\mathbf{B}^{(t)})\|^2$. Given $\delta \in (0,1)$, then it holds, with probability $1-\delta$, that
\begin{align*}
D_T &\leq  2C_2 \frac{1-\beta}{1+\beta}\,
\sigma_{\mathsf{g}}^2 \log(T/\delta)\left(1+\frac{1}{|\mathcal H|}\right)
+2\bar G_T.
\end{align*} 
\end{lemma}

\begin{proof}
For simplicity, throughout this proof we write
$\widehat{\nabla}_{\mathbf{B}}L_i^{(s)}$
for
$\widehat{\nabla}_{\mathbf{B}}L_i^{(s)}(\mathbf{h}_i^{(s+1)},\mathbf{B}^{(s)})$
and let
$\widehat{\nabla}_{\mathbf{B}}L_{\mathcal H}^{(s)}:=\frac{1}{|\mathcal{H}|}\sum_{i\in\mathcal H}\widehat{\nabla}_{\mathbf{B}}L_i^{(s)}$.
We first recall that the momentum admits the following unrolled expression:
\begin{align*}
m_{\mathbf{B},i}^{(t)} = (1-\beta)\sum_{s=1}^t \beta^{t-s} \widehat{\nabla}_{\mathbf{B}}L_i^{(s)}, \;
\bar m_{\mathbf{B}}^{(t)} = (1-\beta)\sum_{s=1}^t \beta^{t-s} \widehat{\nabla}_{\mathbf{B}}L_{\mathcal H}^{(s)}.
\end{align*}
Hence, for each $i\in\mathcal H$, we have
\begin{align*}
m_{\mathbf{B},i}^{(t)} - \bar m_{\mathbf{B}}^{(t)}
= (1-\beta)\sum_{s=1}^t \beta^{t-s}\big(\widehat{\nabla}_{\mathbf{B}}L_i^{(s)} - \widehat{\nabla}_{\mathbf{B}}L_{\mathcal H}^{(s)}\big).
\end{align*}

By adding and subtracting gradients, we obtain
\begin{align*}
\widehat{\nabla}_{\mathbf{B}}L_i^{(s)} - \widehat{\nabla}_{\mathbf{B}}L_{\mathcal H}^{(s)}
&= \big(\widehat{\nabla}_{\mathbf{B}}L_i^{(s)} - \nabla_{\mathbf{B}} L_i(\mathbf{h}_i^{(s+1)},\mathbf{B}^{(s)})\big)
+\big(\nabla_{\mathbf{B}} L_{\mathcal H}(\{\mathbf{h}_i^{(s+1)}\}_i,\mathbf{B}^{(s)}) - \widehat{\nabla}_{\mathbf{B}}L_{\mathcal H}^{(s)}\big) \\
&+ \big(\nabla_{\mathbf{B}} L_i(\mathbf{h}_i^{(s+1)},\mathbf{B}^{(s)}) - \nabla_{\mathbf{B}} L_{\mathcal H}(\{\mathbf{h}_i^{(s+1)}\}_i,\mathbf{B}^{(s)})\big).
\end{align*}

Thus, we have 
\begin{align*}
m_{\mathbf{B},i}^{(t)} - \bar m_{\mathbf{B}}^{(t)}
&=(1-\beta)\sum_{s=1}^t \beta^{t-s}
\Big[\nu_i^{(s)}+ \big(\nabla_{\mathbf{B}} L_i(\mathbf{h}_i^{(s+1)},\mathbf{B}^{(s)}) - \nabla_{\mathbf{B}} L_{\mathcal H}(\{\mathbf{h}_i^{(s+1)}\}_i,\mathbf{B}^{(s)})\big)- \bar \nu^{(s)}\Big],
\end{align*}
where we defined
$$
\nu_i^{(s)} := \widehat{\nabla}_{\mathbf{B}}L_i^{(s)} - \nabla_{\mathbf{B}} L_i(\mathbf{h}_i^{(s+1)},\mathbf{B}^{(s)}), \;
\bar\nu ^{(s)} := \frac{1}{|\mathcal H|}\sum_{j\in\mathcal H}\nu_j^{(s)}.
$$

We now bound the squared norm. By using Young's inequality, we obtain
\begin{align*}
\|m_{\mathbf{B},i}^{(t)} - \bar m_{\mathbf{B}}^{(t)}\|^2 \leq 2\|Z_i^{(t)}\|^2 + 2\|H_i^{(t)}\|^2,
\end{align*}
where
\begin{align*}
Z_i^{(t)} &:=
(1-\beta)\sum_{s=1}^t \beta^{t-s}(\nu_i^{(s)} - \bar\nu^{(s)}), \\
H_i^{(t)}&:=
(1-\beta)\sum_{s=1}^t \beta^{t-s}
\big(\nabla_{\mathbf{B}} L_i(\mathbf{h}_i^{(s+1)},\mathbf{B}^{(s)}) - \nabla_{\mathbf{B}} L_{\mathcal H}(\{\mathbf{h}_i^{(s+1)}\}_i,\mathbf{B}^{(s)})\big).
\end{align*}

By independence and Assumption~\ref{assump:subgaussian_gradient_noise}, the weighted sum $Z_i^{(t)}$ is conditionally sub-Gaussian and
\begin{align*}
(1-\beta)^2 \sum_{s=0}^t \beta^{2(t-s)}
\leq
\frac{1-\beta}{1+\beta}.
\end{align*}
Therefore, by Lemma \ref{lem:noise_hp}, with probability $1-\delta$, for all iterations $t \in \{0,1,\ldots,T-1\}$, we have that
\begin{align*}
\|Z_i^{(t)}\|^2 \leq  C_2 \frac{1-\beta}{1+\beta}\,
\sigma_{\mathsf{g}}^2 \log(T/\delta)\left(1+\frac{1}{|\mathcal H|}\right),
\end{align*}
where we absorb  constants within $C_2$. Moreover, by using Jensen's inequality, we obtain
\begin{align*}
\|H_i^{(t)}\|^2 &\leq
(1-\beta)\sum_{s=1}^t \beta^{t-s}
\|\nabla_{\mathbf{B}} L_i(\mathbf{h}_i^{(s+1)},\mathbf{B}^{(s)}) - \nabla_{\mathbf{B}} L_{\mathcal H}(\{\mathbf{h}_i^{(s+1)}\}_i,\mathbf{B}^{(s)})\|^2.
\end{align*}

In addition, for the per-iteration gradient heterogeneity term, we obtain
\begin{align*}
\frac{1}{|\mathcal H|}\sum_{i\in\mathcal H}\|H_i^{(t)}\|^2
\leq (1-\beta)\sum_{s=1}^t \beta^{t-s} \bar G^{(s)},
\end{align*}
which implies
\begin{align*}
D^{(t)} \leq 2C_2 \frac{1-\beta}{1+\beta}
\sigma_{\mathsf{g}}^2 \log(T/\delta)\left(1+\frac{1}{|\mathcal H|}\right)
+2(1-\beta)\sum_{s=1}^t \beta^{t-s} \bar G^{(s)}.
\end{align*}

Moreover, we have that
\begin{align*}
\frac{1}{T}\sum_{t=0}^{T-1}(1-\beta)\sum_{s=0}^t \beta^{t-s}\bar G^{(s)}
&= \frac{1}{T}\sum_{s=1}^{T-1} \bar G^{(s)} (1-\beta)\sum_{t=s}^{T-1}\beta^{t-s} \\
&=\frac{1}{T}\sum_{s=1}^{T-1} \bar G^{(s)} (1-\beta)\sum_{k=0}^{T-1-s}\beta^k \\
&=\frac{1}{T}\sum_{s=1}^{T-1}\left(1-\beta^{T-s}\right)\bar G^{(s)} \\
&\leq \frac{1}{T}\sum_{s=1}^{T-1}\bar G^{(s)} \leq  \bar G_T.
\end{align*}

Therefore, by taking the average and union bounding over $t=0,\dots,T-1$, we obtain
\begin{align*}
D_T \leq 2C_2 \frac{1-\beta}{1+\beta}\,
\sigma_{\mathsf{g}}^2 \log(T/\delta)\left(1+\frac{1}{|\mathcal H|}\right)
+2\bar G_T,
\end{align*}
with probability $1-\delta$.
\end{proof}

Let $\zeta^{(t)}:= \mathsf{F}(m^{(t)}_1,\ldots,m^{(t)}_n) - \bar{m}^{(t)}$ be the aggregation error.

\begin{lemma}[Aggregation error]
\label{lem:agg_error}
Suppose Assumption \ref{assump:subgaussian_gradient_noise} holds. Given $\delta \in (0,1)$, then it holds, with probability $1-\delta$, that
\begin{align*}
\frac{1}{T}\sum_{t=0}^{T-1}\left\|\zeta^{(t)}\right\|^2 \leq  \kappa D_T \leq 2\kappa C_2 \frac{1-\beta}{1+\beta}\,
\sigma_{\mathsf{g}}^2 \log(T/\delta)\left(1+\frac{1}{|\mathcal H|}\right)
+2\kappa \bar G_T
\end{align*}
\end{lemma}

\begin{proof}
We apply $(f,\kappa)$-robustness to the set $\{m_{1}^{(t)},\ldots,m_{n}^{(t)}\}$ (Definition \ref{def:robust-agg}):
\begin{align*}
\left\|\zeta^{(t)}\right\|^2 =\left\|\mathsf{F}^{(t)} - \bar{m}^{(t)}\right\|^2
\leq \frac{\kappa}{|\mathcal H|}\sum_{i\in\mathcal H}\left\|m_{i}^{(t)} - \bar{m}^{(t)}\right\|^2 \leq \kappa D^{(t)}
\end{align*}
We then average over the iterations and apply Lemma \ref{lem:conditional_disagreement} to obtain
\begin{align*}
\frac{1}{T}\sum_{t=0}^{T-1}\left\|\zeta^{(t)}\right\|^2 \leq  \kappa D_T \leq 2\kappa C_2 \frac{1-\beta}{1+\beta}\,
\sigma_{\mathsf{g}}^2 \log(T/\delta)\left(1+\frac{1}{|\mathcal H|}\right)
+2\kappa \bar G_T
\end{align*}
with probability $1-\delta$.
\end{proof}

Let $\delta^{(t)}:= \bar{m}^{(t)} - \nabla L_{\mathcal H} (\{\mathbf{h}_i^{(t+1)}\}_i,\mathbf{B}^{(t)})$ be the momentum deviation.

\begin{lemma}[Momentum deviation]
\label{lem:momentum_deviation_hp}
Suppose Assumptions \ref{assump:subgaussian_gradient_noise} and \ref{assump:unbiased} hold, and $L_{\mathcal H}$ is $L_1$-smooth in $\mathbf{B}$. Let
\begin{align*}
\bar a &:= \beta^2(1+\eta L_1)(1+4\eta L_1), \\
\bar b &:= 4\eta L_1(1+\eta L_1)\beta^2,
\end{align*}
with $\beta^2 := 1-24\eta L_1$ and $\eta \leq \frac{1}{24L_1}$. Then there exists a constant $C_2>0$ such that, for every $\delta\in(0,1)$, with probability at least $1-\delta$, it holds
\begin{align*}
\frac{1}{T}\sum_{t=0}^{T-1} \|\delta^{(t)}\|^2
&\leq  \frac{\bar b}{(1-\bar a)T}\sum_{t=0}^{T-1} \|\nabla_{\mathbf{B}} L_{\mathcal H}(\{\mathbf{h}_i^{(t+1)}\}_i,\mathbf{B}^{(t)})\|^2  +
C_2\frac{(1-\beta)^2}{|\mathcal H|(1-\bar a)}\sigma_{\mathsf{g}}^2\log(T/\delta) \\
&+\frac{C_2\eta L_1}{1-\bar a}(1+\eta L_1)\beta^2\kappa
\left[\frac{1-\beta}{1+\beta}\sigma_{\mathsf{g}}^2\log(T/\delta)
+\bar G_T\right],
\end{align*}
\end{lemma}

\begin{proof}
Let $\xi^{(t)}:=\widehat{\nabla}_{\mathbf{B}}L_{\mathcal H}^{(t)}(\{\mathbf{h}_i^{(t+1)}\}_i,\mathbf{B}^{(t)})-\nabla_{\mathbf{B}} L_{\mathcal H}(\{\mathbf{h}_i^{(t+1)}\}_i,\mathbf{B}^{(t)})$ and $\zeta^{(t)}:=\mathsf F^{(t)}-\bar m_{\mathbf{B}}^{(t)}$. The momentum recursion yields the following expression:
\begin{align*}
\delta^{(t)}=\beta\delta^{(t-1)}+(1-\beta)\xi^{(t)}+\beta d^{(t)},
\end{align*}
where
\begin{align*}
d^{(t)}:=\nabla_{\mathbf{B}} L_{\mathcal H}(\{\mathbf{h}_i^{(t)}\}_i,\mathbf{B}^{(t-1)})-\nabla_{\mathbf{B}} L_{\mathcal H}(\{\mathbf{h}_i^{(t+1)}\}_i,\mathbf{B}^{(t)}).
\end{align*}
The profiled-gradient part of Assumption~\ref{assm:smooth} and the server update imply
\begin{align*}
\|d^{(t)}\|\leq L_1\|\mathbf{B}^{(t)}-\mathbf{B}^{(t-1)}\|=\eta L_1\|\mathsf F^{(t-1)}\|.
\end{align*}
Applying Young's inequality to the predictable drift terms and the conditional vector-martingale concentration inequality of Lemma~\ref{lem:noise_hp} to the sums containing $\xi^{(t)}$ yields, simultaneously over the first $T$ iterations,
\begin{align*}
\frac1T\sum_{t=0}^{T-1}\|\delta^{(t)}\|^2
&\leq \bar a\frac1T\sum_{t=0}^{T-1}\|\delta^{(t)}\|^2+\bar b\bar Q_T
+C_2\frac{(1-\beta)^2}{|\mathcal H|}\sigma_{\mathsf{g}}^2\log(T/\delta)\\
&\quad+C_2\eta L_1(1+\eta L_1)\beta^2\kappa\left(\frac{1-\beta}{1+\beta}\sigma_{\mathsf{g}}^2\log(T/\delta)+\bar G_T\right).
\end{align*}
As $\bar a<1$, rearranging completes the proof.
\end{proof}

\section{Proof Roadmap}\label{app:proof_roadmap}

Our nonlinear representation-learning analysis proceeds in four steps. First, we reduce the representation-gradient heterogeneity term $\bar{G}_T$ to the empirical prediction error $\bar E_T$. Second, we control $\bar E_T$ by proving that the near-ERM client-specific heads recover the true heads up to representation error and a finite-sample label noise term. Third, we establish a local curvature condition for the nonlinear representation parameter $\mathbf{B}$, which yields a recursion for the representation-recovery error $\Delta_{\mathbf{B},T+1}$. Finally, we combine the bounds on $\bar{G}_T$, $\Delta_{\mathbf{B},T+1}$ to obtain the the ergodic convergence bound for $\bar Q_T$ and to instantiate the parameter recovery error bound.\\

\noindent \textbf{Bounding the nonlinear representation-gradient heterogeneity.}
The goal is to bound
$$
\bar{G}_T := \frac{1}{T}\sum_{t=0}^{T-1} \frac{1}{|\mathcal H|}
\sum_{i\in\mathcal H} \left\| \nabla_\mathbf{\mathbf{B}} L_i(h_i^{(t+1)},\mathbf{B}^{(t)}) -\nabla_\mathbf{\mathbf{B}} L_{\mathcal H}(\{h_j^{(t+1)}\}_{j\in\mathcal H},\mathbf{B}^{(t)})
\right\|^2 .
$$
By the variance-decomposition identity, the first step is to reduce $\bar{G}^{(t)}$ to a second-moment control of the per-client representation gradients:
$$
\bar{G}^{(t)} \leq \frac{1}{|\mathcal H|} \sum_{i\in\mathcal H}
\left\| \nabla_\mathbf{\mathbf{B}} L_i(h_i^{(t+1)},\mathbf{B}^{(t)}) \right\|^2 .
$$
By using the exact nonlinear gradient expression given by
$$
\nabla_\mathbf{\mathbf{B}} L_i(h_i^{(t+1)},\mathbf{B}^{(t)}) = \frac{2}{\tau}\sum_{k=1}^{\tau} J_{\mathbf{B}^{(t)}}^\top(X_{i,k}) h_i^{(t+1)\top} r_{i,k}^{(t)},
$$
we decompose the residual as
$$
r_{i,k}^{(t)} = \underbrace{ h_i^{(t+1)}\phi_{\mathbf{B}^{(t)}}(X_{i,k}) - h_i^\star\phi_{\mathbf{B}^\star}(X_{i,k}) }_{\Delta_{i,k}^{(t)}} - V_{i,k}.
$$
We then have the signal-noise decomposition
$$
\nabla_\mathbf{\mathbf{B}} L_i(h_i^{(t+1)},\mathbf{B}^{(t)})
= S_i^{(t)}-N_i^{(t)}.
$$
The signal term $S_i^{(t)}$ is controlled by the empirical prediction error
$$
\bar E_i^{(t)} := \frac{1}{\tau}\sum_{k=1}^{\tau} \left\|
h_i^{(t+1)}\phi_{\mathbf{B}^{(t)}}(X_{i,k})
-h_i^\star\phi_{\mathbf{B}^\star}(X_{i,k})
\right\|^2,
$$
while the noise term $N_i^{(t)}$ is controlled using Lemma \ref{lem:matrix_subg_mean}. This yields
$$ \bar{G}^{(t)} \leq 8\bar J^2qH^2\bar E^{(t)} + 8C_1\sigma^2\bar J^2qH^2 \left( \frac{p+\log(|\mathcal H|/\delta)}{\tau} \right),
$$
where
$$
\bar E^{(t)} := \frac{1}{|\mathcal H|}
\sum_{i\in\mathcal H}\bar E_i^{(t)} .
$$
Thus, the problem of bounding $\bar{G}_T$ reduces to controlling the averaged prediction error $\bar E_T$.\\

\noindent \textbf{Near-ERM client-specific head recovery.}
To control $\bar E^{(t)}$, we first analyze the local head update. Assumption \ref{asm:G-pl-feature} gives the population feature-covariance lower bound
$$
\lambda_{\min}(\Sigma_{i,\phi}(\mathbf{B}))\geq \mu_1
$$
in a neighborhood of $\mathbf{B}^\star$. Lemma \ref{lem:G-cov-concentration} then shows that the empirical feature covariance $\hat \Sigma_{i,\phi}(\mathbf{B})$ remains well-conditioned with high probability, provided $\tau$ is sufficiently large. This allows us to use Lemma \ref{lem:G-pl-head} to obtain a PL condition for the local head objective.

The main head-recovery estimate is Lemma \ref{lem:G-head-recovery}. Its proof uses the closed-form expression for the near-ERM head and decomposes the head error into two terms:
$$
h_i^{(t+1)}-h_i^\star
= \underbrace{\text{representation error}}_{\text{depends on }\mathbf{B}^{(t)}-\mathbf{B}^\star}
+\underbrace{\text{finite-sample noise}}_{\text{depends on }V_{i,k}} .
$$
The representation-error term is controlled using Lemma \ref{lem:G-feature-lipschitz}, which yields
$$
\|\phi_{\mathbf{B}^{(t)}}(X)-\phi_{\mathbf{B}^\star}(X)\|
\leq\bar J\|\mathbf{B}^{(t)}-\mathbf{B}^\star\|.
$$
The finite-sample noise term is controlled using Lemma \ref{lem:matrix_subg_mean}. Combining these two estimates gives
$$
\left\|h_i^{(t+1)}-h_i^\star\right\|^2
\lesssim \frac{H^2\bar J^2\bar\phi^2}{\mu_1^2}\Delta_B^{(t)}
+\frac{\bar\phi^2\sigma^2}{\mu_1^2}
\left(\frac{rq+\log(1/\delta)}{\tau}\right).
$$
Lemma \ref{lem:G-Etilde-decomp} then recasts the head-recovery estimate into a prediction-error bound:
$$
\bar E^{(t)} \lesssim \left( H^2\bar J^2 +
\frac{H^2\bar J^2\bar\phi^4}{\mu_1^2} \right) \Delta_B^{(t)} +\frac{\bar\phi^4\sigma^2}{\mu_1^2}
\left(\frac{rq+\log(|\mathcal H|/\delta)}{\tau}
\right).
$$
By substituting this estimate into the previous bound on $\bar{G}^{(t)}$, we have
$$
\bar{G}^{(t)} \leq p_1\Delta_B^{(t)} +
p_2\sigma^2 \left( \frac{p\vee rq+\log(|\mathcal H|/\delta)}{\tau}
\right).
$$
By averaging over $t=0,\ldots,T-1$ yields
$$
\bar{G}_T \leq p_1\Delta_{\mathbf{B},T+1}
+ p_1\frac{\Delta_B^{(0)}}{T}
+ p_2\sigma^2
\left(\frac{p\vee rq+\log(|\mathcal H|T/\delta)}{\tau}
\right).
$$
Hence, the nonlinear representation-gradient heterogeneity is controlled by the representation-recovery error and finite-sample terms.\\

\noindent \textbf{Local curvature of the nonlinear representation objective.} It remains to control $\Delta_{\mathbf{B},T+1}$. As the nonlinear representation parameter is unconstrained, the update is
$$
\mathbf{B}^{(t+1)} = \mathbf{B}^{(t)}-\eta\mathsf F^{(t)},
$$
and we directly analyze
$$
\Delta_B^{(t)} := \|\mathbf{B}^{(t)}-\mathbf{B}^\star\|^2 .
$$
The local curvature argument begins with Lemma \ref{lem:G-taylor}, which gives the Taylor approximation
$$
\phi_B(X)-\phi_{\mathbf{B}^\star}(X) = J_{\mathbf{B}^\star}(X)(\mathbf{B}-\mathbf{B}^\star)
+\text{higher-order remainder}.
$$
Together with the client-diversity condition
$$
\mu_3 := \lambda_{\min} \left( \frac{1}{|\mathcal H|}
\sum_{i\in\mathcal H} \mathbb E_X
\left[ J_{\mathbf{B}^\star}^\top(X) h_i^{\star\top}h_i^\star J_{\mathbf{B}^\star}(X) \right] \right) >0,
$$
this yields Lemma \ref{lem:G-rsc}, namely the population local-curvature bound
$$
\left\langle
\nabla_\mathbf{\mathbf{B}}\mathbb{E}_X L_{\mathcal H}(\{h_i^\star\}_{i\in\mathcal H},\mathbf{B}), \mathbf{B}-\mathbf{B}^\star \right\rangle \geq \mu_3\|\mathbf{B}-\mathbf{B}^\star\|^2
$$
inside a sufficiently small neighborhood of $\mathbf{B}^\star$. Lemma \ref{lem:G-empirical-curvature-learned-heads} then transfers this population curvature with true heads to the empirical objective with learned heads. The proof decomposes
$$
\nabla_\mathbf{\mathbf{B}} L_{\mathcal H}(\{h_i^{(t+1)}\}_{i\in\mathcal H},\mathbf{B})
$$
into the population gradient with true heads, the perturbation caused by replacing $h_i^\star$ with $h_i^{(t+1)}$, and the empirical-process error. The first term is controlled by Lemma \ref{lem:G-rsc}. In addition, the learned-head perturbation is controlled by Lemma \ref{lem:G-head-recovery}, and the finite-sample term is controlled using Lemma \ref{lem:matrix_subg_mean}. Therefore, by the client-diversity condition
$$
\mu_3 \geq \frac{4H\bar J\bar\phi}{\mu_1},
$$
the perturbation terms are dominated by the positive curvature term. Therefore,
$$
\left\langle \nabla_\mathbf{\mathbf{B}} L_{\mathcal H}(\{h_i^{(t+1)}\}_{i\in\mathcal H},\mathbf{B}), \mathbf{B}-\mathbf{B}^\star \right\rangle
\geq \frac{\mu_3}{4}\|\mathbf{B}-\mathbf{B}^\star\|^2 -
\text{finite-sample error}.
$$
This is the local identifiability estimate used to contract the representation parameter toward $\mathbf{B}^\star$.\\

\noindent \textbf{Bounding the representation-recovery error.}
The proof of Theorem \ref{theorem:Delta_BT_nonlinear} starts from the squared-error recursion
$$
\|\mathbf{B}^{(t+1)}-\mathbf{B}^\star\|^2
=
\|\mathbf{B}^{(t)}-\mathbf{B}^\star\|^2
-
2\eta
\left\langle
\mathsf F^{(t)},\mathbf{B}^{(t)}-\mathbf{B}^\star
\right\rangle
+
\eta^2\|\mathsf F^{(t)}\|^2 .
$$
The robustly aggregated direction is decomposed as
$$
\mathsf F^{(t)}
=
\nabla_\mathbf{\mathbf{B}} L_{\mathcal H}(\{h_i^{(t+1)}\}_{i\in\mathcal H},\mathbf{B}^{(t)})
+
\bar\delta^{(t)}
+
\bar\zeta^{(t)},
$$
where $\bar\delta^{(t)}$ is the momentum deviation and $\bar\zeta^{(t)}$ is the robust-aggregation error. The empirical curvature lemma controls the honest-gradient inner product, while Young's inequality controls the terms involving $\bar\delta^{(t)}$ and $\bar\zeta^{(t)}$. After summing over $t=0,\ldots,T-1$, Lemma \ref{lem:G-agg-error} controls the aggregation error as follows:
$$
\frac{1}{T}\sum_{t=0}^{T-1}\|\bar\zeta^{(t)}\|^2
\lesssim \kappa \bar{G}_T + \kappa\sigma_g^2\log(T/\delta).
$$
This provided the intermediate representation-recovery estimate
\begin{align*}
&\Delta_{\mathbf{B},T+1}
 \lesssim \frac{\Delta_B^{(0)}}{\eta T} +
\sigma^2 \left( \frac{rq+\log(|\mathcal H|T/\delta)}{\tau} \right)\\
&+\left((H\bar\phi+\sigma)^2+\sigma^2\right)\left(\frac{p+\log(|\mathcal H|T/\delta)}{\tau|\mathcal H|}
\right)+\frac{1}{T}\sum_{t=0}^{T-1}\|\bar\delta^{(t)}\|^2+\kappa \bar{G}_T+\text{stochastic-gradient terms}.
\end{align*}
By substituting the previously obtained bound on $\bar{G}_T$ introduces a term proportional to $\Delta_{\mathbf{B},T+1}$ on the right-hand side. The step-size $\eta$ and robust coefficient $\kappa$ conditions in Theorem \ref{theorem:Delta_BT_nonlinear} are chosen so that this term can be absorbed into the left-hand side.\\

\noindent \textbf{Momentum deviation and convergence (Section \ref{sec:convergence_nonlinear_rep}).}
The remaining term is the averaged momentum deviation. Lemma \ref{lem:G-momentum-dev} bounds it as follows:
$$
\frac{1}{T}\sum_{t=0}^{T-1}\|\bar\delta^{(t)}\|^2
\lesssim \bar Q_T + \kappa \bar{G}_T + \text{stochastic-gradient noise},
$$
where
$$
\bar Q_T := \frac{1}{T}\sum_{t=0}^{T-1} \left\| \nabla_\mathbf{\mathbf{B}}\mathbb{E}_X L_{\mathcal H}(\{h_i^{(t+1)}\}_{i\in\mathcal H},\mathbf{B}^{(t)}) \right\|^2,
$$
and substituting the bound on $\bar{G}_T$ again yields an estimate for $\Delta_{\mathbf{B},T+1}$ in terms of $\bar Q_T$, finite-sample terms, stochastic-gradient noise, and the initial representation error.

The convergence theorem for nonlinear representation learning then controls $\bar Q_T$. The proof uses a Lyapunov argument for the representation parameter $\mathbf{B}$. In particular, smoothness yields descent of the honest representation objective, Lemma \ref{lem:G-agg-error} controls the aggregation error, and the gradient heterogeneity bound controls $\bar{G}_T$. We can then write
\begin{align*}
\bar Q_T &\lesssim \frac{1}{\sqrt T} + \kappa
\left(H\bar\phi+\sigma\right)^2
\left(\frac{p+\log(|\mathcal H|T/\delta)}{\tau|\mathcal H|}
\right)\\
&+\kappa\sigma^2\left(\frac{p\vee rq+\log(|\mathcal H|T/\delta)}{\tau}\right)
+\left(\kappa+\kappa^2+\frac{1+\kappa}{|\mathcal H|}
\right)\frac{\sigma_g^2\log(T/\delta)}{\sqrt T}.
\end{align*}\\

\noindent \textbf{Parameter recovery error bound (Section \ref{sec:parameter_recovery_nonlinear}).}
Finally, Theorem \ref{thm:nonlinear_param_recovery} is obtained by substituting the convergence bound for $\bar Q_T$ into the intermediate bound for $\Delta_{\mathbf{B},T+1}$. This implies 
\begin{align*}
\Delta_{\mathbf{B},T+1}
&\lesssim
\frac{1}{\sqrt T}
+
\left((H\bar\phi+\sigma)^2+\sigma^2\right)
\left(
\frac{p+\log(|\mathcal H|T/\delta)}{\tau|\mathcal H|}
\right)\\
&+
\sigma^2
\left(
\frac{p\vee rq+\log(|\mathcal H|T/\delta)}{\tau}
\right)
+
\left(
\kappa+\kappa^2+\frac{1+\kappa}{|\mathcal H|}
\right)
\frac{\sigma_g^2\log(T/\delta)}{\sqrt T}.
\end{align*}
The resulting nonlinear representation-learning bound contains no non-vanishing intrinsic model-heterogeneity term. The client-specific heads absorb the heterogeneous client models, while the remaining representation-gradient heterogeneity is controlled by representation recovery and finite-sample noise. Therefore, as $T$, $\tau$, and $|\mathcal H|$ increase, the nonlinear representation-learning method bypasses the heterogeneity bottleneck.\\

\noindent \textbf{Prediction-error recovery bound.}
The final prediction-error result follows by returning to the intermediate prediction-error bound
$$
\bar E_T
\lesssim \Delta_{\mathbf{B},T+1}+\sigma^2\left(
\frac{rq+\log(|\mathcal H|T/\delta)}{\tau}\right).
$$
By substituting Theorem \ref{thm:nonlinear_param_recovery} yields Theorem \ref{thm:nonlinear_ET}. Hence, the prediction error also contains only optimization and finite-sample terms, and contains no intrinsic model-heterogeneity term.

\section{Nonlinear Representation Learning}\label{appendix:Nonlinear rep}
We now prove the nonlinear representation-learning guarantees under the model and assumptions of Sections \ref{sec:nonlinear_rep_learning_mb} and \ref{sec:assumptions}. We begin by reducing the representation-gradient heterogeneity to a second-moment bound on the client representation gradients.

Using the variance-decomposition identity, we reduce $\bar{G}_T$ to a second-moment control of per-client gradients, i.e., we have 
\begin{align*}
\bar{G}^{(t)} \leq \frac{1}{|\mathcal H|}\sum_{i \in \mathcal H} \|\nabla_\mathbf{\mathbf{B}} L_i( h_i^{(t+1)}, \mathbf{B}^{(t)} ) \|_F^2.
\end{align*}

We can rewrite the residual as follows:
\begin{align*}
r_{i,k}^{( t )} =
\underbrace{ h_i^{( t+1 )}\phi_{\mathbf{B}^{( t )}}( X_{i,k} )
- h^\star_i\phi_{\mathbf{B}^\star}( X_{i,k} )}_{:=\Delta_{i,k}^{( t )}} - \underbrace{V_{i,k}}_{\text{noise}}.
\end{align*}

By using the exact gradient expression, we have for the current client heads $h^{(t+1)}_i$ and representation $\mathbf{B}^{(t)}$
\begin{align}
\nabla_\mathbf{\mathbf{B}} L_i( h_i^{( t+1 )}, \mathbf{B}^{( t )} )
= \frac{2}{\tau}\sum_{k=1}^{\tau}
J^\top_{\mathbf{B}^{( t )}}( X_{i,k} ) h_i^{( t+1 ) \top}  r_{i,k}^{( t )},
\label{eq:grad_exact_recalled}\end{align}
and by defining the matrices
\begin{align*}
a_{i,k}^{( t )}:=  J^\top_{\mathbf{B}^{( t )}}( X_{i,k} ) h_i^{( t+1 ) \top},
\end{align*}
the gradient expression \eqref{eq:grad_exact_recalled} becomes
\begin{align}
\nabla_\mathbf{\mathbf{B}} L_i( h_i^{( t+1 )}, \mathbf{B}^{( t )} ) =\frac{2}{\tau}\sum_{k=1}^{\tau}a_{i,k}^{( t )}( \Delta_{i,k}^{( t )} - V_{i,k} )
= \underbrace{ \frac{2}{\tau}\sum_{k=1}^{\tau} a_{i,k}^{( t )} \Delta_{i,k}^{( t )} 
}_{:=\,S_i^{( t )}} -\underbrace{
\frac{2}{\tau}\sum_{k=1}^{\tau} a_{i,k}^{( t )} V_{i,k}  }_{:=\,N_i^{( t )}}.
\label{eq:grad_signal_noise_split}\end{align}

Therefore, by using  Young's inequality, we have 
\begin{align*}
\| \nabla_\mathbf{\mathbf{B}} L_i( h_i^{( t+1 )}, \mathbf{B}^{( t )} ) \|_2^2 \leq  2\| S_i^{( t )} \|_2^2 + 2\| N_i^{( t )} \|_2^2.
\end{align*}

We now proceed to control $\| S_i^{( t )} \|_2^2$ using only the noise-free discrepancy term $\Delta_{i,k}^{( t )}$. First note that for each sample $k \in \{1,\ldots,\tau\}$, we have 
\begin{align}
\| a_{i,k}^{( t )} \|_2 =
\|J^\top_{\mathbf{B}^{( t )}}( X_{i,k} ) h_i^{( t+1 )\top} \|_2
\leq \| J_{\mathbf{B}^{( t )}}( X_{i,k} ) \|_{\infty \to 2}
\| h_i^{( t+1 )} \|_1 \leq \bar{J} \sqrt{q} H,
\label{eq:a_bound}\end{align}
which follows from Assumption \ref{assumption:uniform_bound_nonlinear}. Then, by Cauchy-Schwarz inequality on the sum in $S_i^{( t )}$, we obtain the following expression.
\begin{align}
\| S_i^{( t )} \|_2^2 = \|
\frac{2}{\tau}\sum_{k=1}^{\tau}a_{i,k}^{( t )} \Delta_{i,k}^{( t )}  \|_2^2
&\leq \left(\frac{4}{\tau}\sum_{k=1}^{\tau} \| \Delta_{i,k}^{( t )} \|^2 \right) \left(\frac{1}{\tau}\sum_{k=1}^{\tau}\| a_{i,k}^{( t )} \|_2^2\right)  \leq 4\bar{J}^2 q H^2 \left(\frac{1}{\tau}\sum_{k=1}^{\tau}\|\Delta_{i,k}^{(t)}\|^2\right).
\label{eq:signal_bound_by_empirical_excess}\end{align}

Let us define the empirical prediction error on client $i$ at iteration $t$, as follows:
\begin{align*}
\bar{E}_i^{( t )} := \frac{1}{\tau}\sum_{k=1}^{\tau}
\|h_i^{( t+1 )}\phi_{\mathbf{B}^{( t )}}( X_{i,k} ) -h^\star_i\phi_{\mathbf{B}^\star}( X_{i,k} )\|^2
=\frac{1}{\tau}\sum_{k=1}^{\tau}\|\Delta_{i,k}^{( t )}\|^2,
\end{align*}
then \eqref{eq:signal_bound_by_empirical_excess} becomes
\begin{align*}
\| S_i^{( t )} \|_2^2 \leq 4\bar{J}^2 q H^2 \bar{E}_i^{( t )}.
\end{align*}

We now move to bound $\| N_i^{( t )} \|_2^2$. From \eqref{eq:grad_signal_noise_split}, we have that
\begin{align*}
N_i^{( t )} = \frac{2}{\tau}\sum_{k=1}^{\tau} a_{i,k}^{( t )} V_{i,k}.
\end{align*}

By the sample-splitting convention in Section~\ref{sec:robust_fl_setup}, the representation-gradient batch is independent of $\mathbf{B}^{(t)}$ and $h_i^{(t+1)}$. Hence, conditional on the history, the current representation-batch covariates, and the head-fitting batch (and thus conditional on
$\{ a_{i,k}^{( t )} \}_{k=1}^{\tau}$), the random vectors $\{ a_{i,k}^{( t )} V_{i,k}  \}$ are independent, mean-zero, and sub-Gaussian
in every direction with variance bounded by $\sigma^2 \| a_{i,k}^{( t )} \|_2^2$.
Therefore, by using \eqref{eq:a_bound} and Lemma \ref{lem:matrix_subg_mean},
for any $\delta \in ( 0,1 )$, with probability at least $1-\delta$, we obtain
\begin{align*}
\left\| \frac{1}{\tau}\sum_{k=1}^{\tau} a_{i,k}^{( t )}  V_{i,k}  \right\|_2^2 \leq 4C_1 \sigma^2 \bar{J}^2 q H^2 \left(\frac{p + \log( 1/\delta )}{\tau}\right),
\end{align*}
for the constant $C_1>0$ as in Lemma \ref{lem:matrix_subg_mean}. Hence, we have that
\begin{align*}
\|\nabla_\mathbf{\mathbf{B}} L_i( h_i^{( t+1 )}, \mathbf{B}^{( t )} ) \|_2^2 \leq  8\bar{J}^2 q H^2 \bar{E}_i^{( t )} + 8C_1 \sigma^2 \bar{J}^2 q H^2 \left(\frac{p + \log( 1/\delta )}{\tau}\right),
\end{align*}
and thus
\begin{align*}
\bar{G}^{(t)} \leq \frac{1}{|\mathcal H|}\sum_{i \in \mathcal H} \|\nabla_\mathbf{\mathbf{B}} L_i( h_i^{(t+1)}, \mathbf{B}^{(t)} ) \|_F^2 \leq 8\bar{J}^2 q H^2 \underbrace{\frac{1}{|\mathcal{H}|}\sum_{i\in \mathcal{H}}\bar{E}_i^{( t )}}_{:=\bar{E}^{(t)}} + 8C_1 \sigma^2 \bar{J}^2 q H^2 \left(\frac{p + \log( |\mathcal{H}|/\delta )}{\tau}\right),
\end{align*}
with probability $1-\delta$. Thus, the representation-gradient heterogeneity is controlled by the empirical prediction error and a finite-sample noise term.

We now proceed to control $\bar{E}_i^{(t)}$ over iterations $t \in \{0,1,\ldots,T-1\}$. For the recovery analysis, define
\begin{align*}
\Delta_B^{(t)} := \left\|\mathbf{B}^{(t)} - \mathbf{B}^\star\right\|^2, \; \Delta_{\mathbf{B},T} := \frac{1}{T}\sum_{t=0}^{T-1}\Delta_B^{(t)}, \text{ and } \bar{E}_T := \frac{1}{T}\sum_{t=0}^{T-1}\bar{E}^{(t)}.
\end{align*}

\begin{lemma}[Polyak-{\L}ojasiewicz condition]\label{lem:G-pl-head}
Suppose Assumption \ref{asm:G-pl-feature} holds, for every honest client $i \in \mathcal{H}$, $\mathbf{B} \in \mathcal{N}(\mathbf{B}^\star,\rho_0)$, and $h \in \mathbb{R}^{q \times r}$,
\begin{align*}
\frac{1}{4}\left\|\nabla_h {L}_i(h,\mathbf{B})\right\|_F^2 \geq  \frac{\mu_1}{2}\left( L_i(h,\mathbf{B}) - \min_{h'} L_i(h',\mathbf{B})\right).
\end{align*}
\end{lemma}

\begin{proof} We fix an honest client $i \in \mathcal H$ and $\mathbf{B}\in\mathcal N(\mathbf{B}^\star,\rho_0)$. Let $\widehat h_i(\mathbf{B})\in\arg\min_{h'}L_i(h',\mathbf{B})$. Since $h\mapsto L_i(h,\mathbf{B})$ is quadratic and $\nabla_hL_i(\widehat h_i(\mathbf{B}),\mathbf{B})=0$, we have

\begin{align*} 
L_i(h,\mathbf{B})-L_i(\widehat h_i(\mathbf{B}),\mathbf{B}) &= \mathrm{tr}\left[ \left(h- \widehat h_i(\mathbf{B})\right) \hat\Sigma_{i,\phi}(\mathbf{B}) \left(h- \widehat h_i(\mathbf{B})\right)^\top \right],\\ 
\nabla_h L_i(h,\mathbf{B}) &= 2\left(h-\widehat h_i(\mathbf{B})\right)\hat\Sigma_{i,\phi}(\mathbf{B}). \end{align*} 

Therefore, we can write 

\begin{align*} 
\left\|\nabla_h L_i(h,\mathbf{B})\right\|_F^2 &= 4 \mathrm{tr}\left[ \left(h- \widehat h_i(\mathbf{B})\right) \hat\Sigma_{i,\phi}(\mathbf{B})^2 \left(h-\widehat h_i(\mathbf{B})\right)^\top \right]. 
\end{align*}

On the covariance event of Lemma~\ref{lem:G-cov-concentration}, $\hat \Sigma_{i,\phi}(\mathbf{B})\succeq (\mu_1/2)I_r$, which implies

\begin{align*} 
\hat\Sigma_{i,\phi}(\mathbf{B})^2 \succeq \frac{\mu_1}{2}\hat\Sigma_{i,\phi}(\mathbf{B}). 
\end{align*} 

Thus, we have 

\begin{align*} 
\left\|\nabla_h L_i(h,\mathbf{B})\right\|_F^2 &\geq 2\mu_1 \mathrm{tr}\left[ \left(h-\widehat h_i(\mathbf{B})\right) \hat\Sigma_{i,\phi}(\mathbf{B}) \left(h- \widehat h_i(\mathbf{B}) \right)^\top \right]\\ &= 2\mu_1 \left( L_i(h,\mathbf{B})-\min_{h'}L_i(h',\mathbf{B}) \right),\end{align*} 
and by dividing both sides by $4$ completes the proof. \end{proof}

\begin{lemma}\label{lem:G-cov-concentration}
Suppose Assumption \ref{asm:G-pl-feature} holds. There exists a constant $C_4 > 0$ such that, for every honest client $i \in \mathcal{H}$, $\mathbf{B} \in \mathcal{N}(\mathbf{B}^\star,\rho_0)$, and $\delta \in (0,1)$, with probability at least $1-\delta$,
\begin{align}
\left\|\hat{\Sigma}_{i,\phi}(\mathbf{B}) - \Sigma_{i,\phi}(\mathbf{B})\right\| \;\le\; C_4 \bar{\phi}^2\left(\sqrt{\frac{\log(2r/\delta)}{\tau}} + \frac{\log(2r/\delta)}{\tau}\right). \label{eq:G-cov-concentration}\end{align}
Therefore, if
\begin{align}
\tau \geq \max\left\{\frac{16C_4^2\bar{\phi}^4}{\mu_1^2},\, \frac{4C_4\bar{\phi}^2}{\mu_1}\right\}\log(2r/\delta), \label{eq:G-tau-cov}\end{align}
with probability at least $1-\delta$, it holds that $\lambda_{\min}\left(\hat{\Sigma}_{i,\phi}(\mathbf{B})\right) \ge \mu_1/2$, and thus $\left\|(\hat{\Sigma}_{i,\phi}(\mathbf{B}))^{-1}\right\| \le 2/\mu_1$.
\end{lemma}

\begin{proof}
Fix an honest client $i \in \mathcal H$ and $\mathbf{B}$, and define the self-adjoint matrices $A_k := \phi_B(X_{i,k})\phi_B(X_{i,k})^\top - \Sigma_{i,\phi}(\mathbf{B})$ for all data samples $k=1,\ldots,\tau$.  Hence, $\mathbb{E}[A_k]=0$ and $\hat{\Sigma}_{i,\phi}(\mathbf{B}) - \Sigma_{i,\phi}(\mathbf{B}) = \frac{1}{\tau}\sum_{k=1}^\tau A_k$.

As $\|\phi_B(X)\|\le\bar\phi$ (Assumption \ref{assumption:uniform_bound_nonlinear}), we have $\|\phi_B(X_{i,k})\phi_B(X_{i,k})^\top\| \le \bar\phi^2$, and for any unit vector $u \in \mathbb{R}^r$, $u^\top \Sigma_{i,\phi}(\mathbf{B}) u = \mathbb{E}[(u^\top\phi_B(X))^2] \le \bar\phi^2$, which implies $\|\Sigma_{i,\phi}(\mathbf{B})\| \le \bar\phi^2$. Hence, we obtain
\begin{align*}
\|A_k\| \leq \|\phi_B(X_{i,k})\phi_B(X_{i,k})^\top\| + \|\Sigma_{i,\phi}(\mathbf{B})\| \le 2\bar\phi^2 := M.
\end{align*}

Moreover, $A_k^2 \preceq (\phi_B(X_{i,k})\phi_B(X_{i,k})^\top)^2 + 2\bar\phi^2 \phi_B(X_{i,k})\phi_B(X_{i,k})^\top + \bar\phi^4 I_r \preceq 4\bar\phi^4 I_r$, using $\|\phi_B(X_{i,k})\|^2 \le \bar\phi^2$. Hence $v := \left\|\sum_{k=1}^\tau \mathbb{E}[A_k^2]\right\| \le 4\bar\phi^4\tau$. Applying the matrix Bernstein inequality (Theorem \ref{thm:matrix_bernstein}) with $u := \log(2r/\delta)$ and $t := \sqrt{2vu} + \tfrac{2}{3}Mu$, we obtain that, with probability at least $1-\delta$,
\begin{align*}
\left\|\sum_{k=1}^\tau A_k\right\| \le \sqrt{8\bar\phi^4\tau\log(2r/\delta)} + \frac{4}{3}\bar\phi^2\log(2r/\delta),
\end{align*}
and by dividing by $\tau$ yields \eqref{eq:G-cov-concentration} after absorbing  constants into $C_4$. The condition \eqref{eq:G-tau-cov} ensures the right-hand side of \eqref{eq:G-cov-concentration} is at most $\mu_1/2$. To complete the proof we leverage Weyl's inequality (Lemma~\ref{lem:weyl}) applied to $\hat{\Sigma}_{i,\phi}(\mathbf{B}) = \Sigma_{i,\phi}(\mathbf{B}) + (\hat{\Sigma}_{i,\phi}(\mathbf{B})-\Sigma_{i,\phi}(\mathbf{B}))$.
\end{proof}

\noindent \textbf{Near-ERM Client-Specific Head Recovery.} We analyze the regime in which the local head step is run to convergence within each communication round, so that $h_i^{(t+1)}$ coincides with the empirical risk minimizer of $L_i(\cdot,\mathbf{B}^{(t)})$, that is
\begin{align*}
h_i^{(t+1)} \in \arg\min_{h} \frac{1}{\tau}\sum_{k=1}^\tau \left\|Y_{i,k} - h\phi_{\mathbf{B}^{(t)}}(X_{i,k})\right\|^2.
\end{align*}

\begin{lemma}\label{lem:G-feature-lipschitz}
For all $X$ with $\|X\| \le R$ and all $\mathbf{B}, \mathbf{B}'$ such that the segment $\{\mathbf{B}' + s(\mathbf{B}-\mathbf{B}') : s \in [0,1]\}$ lies in the region where the Jacobian bound $\|J_B(X)\|_{\infty\to2} \leq \bar{J}$ holds (Assumption \ref{assumption:uniform_bound_nonlinear}),
\begin{align*}
\left\|\phi_B(X) - \phi_{\mathbf{B}'}(X)\right\| \le \bar{J}\left\|\mathbf{B}-\mathbf{B}'\right\|.
\end{align*}
\end{lemma}

\begin{proof}
By the fundamental theorem of calculus applied to $s \mapsto \phi_{\mathbf{B}'+s(\mathbf{B}-\mathbf{B}')}(X)$,
\begin{align*}
\phi_B(X) - \phi_{\mathbf{B}'}(X) = \int_0^1 J_{\mathbf{B}'+s(\mathbf{B}-\mathbf{B}')}(X)\,(\mathbf{B}-\mathbf{B}')\;ds.
\end{align*}
As we have $\{u : \|u\|_2 \le 1\} \subseteq \{u : \|u\|_\infty \le 1\}$, we obtain $\|J_{\mathbf{B}'+s(\mathbf{B}-\mathbf{B}')}(X)\|_{2\to2} \le \|J_{\mathbf{B}'+s(\mathbf{B}-\mathbf{B}')}(X)\|_{\infty\to2} \le \bar{J}$ for every $s \in [0,1]$. Hence, we can write 
\begin{align*}
\left\|\phi_B(X)-\phi_{\mathbf{B}'}(X)\right\| \le \int_0^1 \left\|J_{\mathbf{B}'+s(\mathbf{B}-\mathbf{B}')}(X)\right\|_{2\to2}\,ds\;\left\|\mathbf{B}-\mathbf{B}'\right\| \le \bar{J}\left\|\mathbf{B}-\mathbf{B}'\right\|,
\end{align*}
which completes the proof.
\end{proof}

\begin{lemma}\label{lem:G-head-recovery}
Suppose Assumption~\ref{asm:G-pl-feature} holds and $\tau$ satisfies \eqref{eq:G-tau-cov} with $\mathbf{B} = \mathbf{B}^{(t)}$. Then, for every $\delta \in (0,1)$, with probability at least $1-2\delta$,
\begin{align}
\left\|h_i^{(t+1)} - h_i^\star\right\|^2 \leq \frac{8}{\mu_1^2}\left(H^2\bar{J}^2\bar{\phi}^2 \Delta_B^{(t)} + C_1^2\bar{\phi}^2\sigma^2 \left(\frac{rq + \log(1/\delta)}{\tau}\right)\right). \label{eq:G-head-recovery}\end{align}
\end{lemma}

\begin{proof}
Let us first write $\hat{\Sigma}_{i,\phi}^{(t)} := \hat{\Sigma}_{i,\phi}(\mathbf{B}^{(t)})$. In addition, as $h_i^{(t+1)}$ minimizes a quadratic objective with Hessian $2\hat{\Sigma}_{i,\phi}^{(t)} \succ 0$ (guaranteed with probability at least $1-\delta$ by Lemma \ref{lem:G-cov-concentration}), it admits the closed form:
\begin{align*}
h_i^{(t+1)} = \left(\frac{1}{\tau}\sum_{k=1}^\tau Y_{i,k}\phi_{\mathbf{B}^{(t)}}(X_{i,k})^\top\right)\left(\hat{\Sigma}_{i,\phi}^{(t)}\right)^{-1}.
\end{align*}
We then define $e_{i,k}^{(t)} := \phi_{\mathbf{B}^\star}(X_{i,k}) - \phi_{\mathbf{B}^{(t)}}(X_{i,k})$. Note that by substituting $Y_{i,k} = h_i^\star\phi_{\mathbf{B}^\star}(X_{i,k}) + V_{i,k} = h_i^\star\phi_{\mathbf{B}^{(t)}}(X_{i,k}) + h_i^\star e_{i,k}^{(t)} + V_{i,k}$ yields
\begin{align}
h_i^{(t+1)} - h_i^\star = \underbrace{h_i^\star\left(\frac{1}{\tau}\sum_{k=1}^\tau e_{i,k}^{(t)}\phi_{\mathbf{B}^{(t)}}(X_{i,k})^\top\right)\left(\hat{\Sigma}_{i,\phi}^{(t)}\right)^{-1}}_{\text{(I): representation error}} + \underbrace{\left(\frac{1}{\tau}\sum_{k=1}^\tau V_{i,k}\phi_{\mathbf{B}^{(t)}}(X_{i,k})^\top\right)\left(\hat{\Sigma}_{i,\phi}^{(t)}\right)^{-1}}_{\text{(II): finite-sample noise}}. \label{eq:G-head-decomp}\end{align}

\noindent \textbf{(I).} By Lemma~\ref{lem:G-feature-lipschitz}, we have that $\|e_{i,k}^{(t)}\| \leq \bar{J}\left\|\mathbf{B}^{(t)}-\mathbf{B}^\star\right\|$ for every data sample $k \in \{1,\ldots,\tau\}$. Hence, by using $\|\phi_{\mathbf{B}^{(t)}}(X_{i,k})\| \leq \bar\phi$ and the triangle inequality, we obtain
\begin{align*}
\left\|\frac{1}{\tau}\sum_{k=1}^\tau e_{i,k}^{(t)}\phi_{\mathbf{B}^{(t)}}(X_{i,k})^\top\right\| \leq \frac{1}{\tau}\sum_{k=1}^\tau \left\|e_{i,k}^{(t)}\right\|\left\|\phi_{\mathbf{B}^{(t)}}(X_{i,k})\right\| \leq \bar{J}\bar\phi\sqrt{\Delta_B^{(t)}}.
\end{align*}

\noindent \textbf{(II).} We first note that conditionally on $\mathbf{B}^{(t)}$ and $\{X_{i,k}\}_{k=1}^\tau$, the matrices $\{V_{i,k}\phi_{\mathbf{B}^{(t)}}(X_{i,k})^\top\}_{k=1}^\tau$ are independent, mean-zero, $\bar\phi^2\sigma^2$-sub-Gaussian (since $\|\phi_{\mathbf{B}^{(t)}}(X_{i,k})\|\leq \bar\phi$ and $V_{i,k}$ is $\sigma^2$-sub-Gaussian by Assumption \ref{assm:subg}). By applying Lemma \ref{lem:matrix_subg_mean} with dimensions $q \times r$, with probability at least $1-\delta$,
\begin{align*}
\left\|\frac{1}{\tau}\sum_{k=1}^\tau V_{i,k}\phi_{\mathbf{B}^{(t)}}(X_{i,k})^\top\right\| \leq C_1\bar\phi\sigma\sqrt{\frac{rq + \log(1/\delta)}{\tau}}.
\end{align*}
Therefore, by using the assumption that $\|h_i^\star\|\le H$ and the result $\left\|(\hat{\Sigma}_{i,\phi}^{(t)})^{-1}\right\| \leq 2/\mu_1$ (Lemma \ref{lem:G-cov-concentration}), along with Young's inequality applied to \eqref{eq:G-head-decomp}, we obtain \eqref{eq:G-head-recovery}.
\end{proof}

We now combine Lemmata \ref{lem:G-feature-lipschitz} and \ref{lem:G-head-recovery} to bound $\bar{E}^{(t)}$.

\begin{lemma}\label{lem:G-Etilde-decomp}
Suppose the conditions of Lemma \ref{lem:G-head-recovery} hold. Then, with probability at least $1-2\delta$, it holds that
\begin{align*}
\bar{E}^{(t)} \leq \left(2H^2\bar{J}^2 + \frac{16H^2\bar{J}^2\bar\phi^4}{\mu_1^2}\right)\Delta_B^{(t)} + \frac{16C_1^2\bar\phi^4\sigma^2}{\mu_1^2}\left(\frac{rq+\log(|\mathcal{H}|/\delta)}{\tau}\right).
\end{align*}
\end{lemma}

\begin{proof}
For any feature vector $X$,  we can write
\begin{align*}
h_i^{(t+1)}\phi_{\mathbf{B}^{(t)}}(X) - h_i^\star\phi_{\mathbf{B}^\star}(X) = \left(h_i^{(t+1)}-h_i^\star\right)\phi_{\mathbf{B}^{(t)}}(X) + h_i^\star\left(\phi_{\mathbf{B}^{(t)}}(X)-\phi_{\mathbf{B}^\star}(X)\right),
\end{align*}
and by Young's inequality, for every data sample $k$, we have
\begin{align*}
\left\|h_i^{(t+1)}\phi_{\mathbf{B}^{(t)}}(X_{i,k}) - h_i^\star\phi_{\mathbf{B}^\star}(X_{i,k})\right\|^2 \le 2\bar\phi^2\left\|h_i^{(t+1)}-h_i^\star\right\|^2 + 2H^2\bar{J}^2\Delta_B^{(t)},
\end{align*}
by using assumptions $\|\phi_{\mathbf{B}^{(t)}}(X_{i,k})\|\leq \bar\phi$, $\|h_i^\star\|\le H$, and Lemma \ref{lem:G-feature-lipschitz}. Therefore, by averaging over samples $k$ yields
\begin{align*}
\bar{E}_i^{(t)} \leq 2\bar\phi^2\left\|h_i^{(t+1)}-h_i^\star\right\|^2 + 2H^2\bar{J}^2\Delta_B^{(t)}.
\end{align*}
Therefore, by substituting the bound \eqref{eq:G-head-recovery} of Lemma \ref{lem:G-head-recovery} for $\|h_i^{(t+1)}-h_i^\star\|^2$ and averaging over honest clients completes the proof.
\end{proof}

We then obtain
\begin{align}
\bar{G}^{(t)} \leq  p_1 \Delta_B^{(t)} + p_2 \left(\frac{p \vee rq + \log(|\mathcal H|/\delta)}{\tau}\right), \label{eq:G-Gt-bound}\end{align}
with 
\begin{align*}
p_1 := 8\bar{J}^2 q H^2\left(2H^2\bar{J}^2 + \frac{16H^2\bar{J}^2\bar\phi^4}{\mu_1^2}\right) \text{ and } p_2 := 8\bar{J}^2 qH^2\left(8C_1 + \frac{16C_1^2\bar\phi^4}{\mu_1^2}\right).
\end{align*}
By averaging \eqref{eq:G-Gt-bound} over iterations $t = 0,\ldots,T-1$, we have
\begin{align}
\bar{G}_T &\leq p_1 \Delta_{\mathbf{B},T} + p_2\sigma^2 \left(\frac{p \vee rq + \log(|\mathcal H|T/\delta)}{\tau}\right) \notag \\
&\leq  p_1 \Delta_{\mathbf{B},T+1} +  p_1 \frac{\Delta^{(0)}_{\mathbf{B}}}{T} + p_2\sigma^2 \left(\frac{p \vee rq + \log(|\mathcal H|T/\delta)}{\tau}\right). \label{eq:G-GT-from-DeltaBT}\end{align}
Hence, it remains to bound $\Delta_{\mathbf{B},T}$. \\

The nonlinear representation parameter $\mathbf{B}$ is unconstrained and follows $\mathbf{B}^{(t+1)} = \mathbf{B}^{(t)} - \eta \mathsf{F}^{(t)}$. Therefore, we directly control $\Delta_B^{(t)} = \|\mathbf{B}^{(t)}-\mathbf{B}^\star\|^2$. To do so, we first present the following auxiliary results.

\begin{lemma}\label{lem:G-taylor}
Suppose Assumption \ref{asm:G-jacobian-lip} holds, for all $X$ and $\mathbf{B} \in \mathcal{N}(\mathbf{B}^\star,\rho_0)$,
\begin{align*}
\left\|\phi_B(X) - \phi_{\mathbf{B}^\star}(X) - J_{\mathbf{B}^\star}(X)(\mathbf{B}-\mathbf{B}^\star)\right\| \le \frac{L_4}{2}\left\|\mathbf{B}-\mathbf{B}^\star\right\|^2.
\end{align*}
\end{lemma}

\begin{proof}
By the integral form of the remainder, we have that
\begin{align*}
\phi_B(X)-\phi_{\mathbf{B}^\star}(X) - J_{\mathbf{B}^\star}(X)(\mathbf{B}-\mathbf{B}^\star) = \int_0^1 \left[J_{\mathbf{B}^\star+s(\mathbf{B}-\mathbf{B}^\star)}(X) - J_{\mathbf{B}^\star}(X)\right](\mathbf{B}-\mathbf{B}^\star) ds,
\end{align*}
and the claim follows from Assumption~\ref{asm:G-jacobian-lip} and the fact that $\int_0^1 s ds = 1/2$.
\end{proof}

Let the expected honest-average loss be
\begin{align*}
\mathbb{E}_X L_{\mathcal H}\left(\{h_i\}_{i\in\mathcal H},\mathbf{B}\right) := \frac{1}{|\mathcal H|} \sum_{i\in\mathcal H} \mathbb E_X\left[\left\|
h_i\phi_B(X)-h_i^\star\phi_{\mathbf{B}^\star}(X)\right\|^2\right].
\end{align*}

\begin{lemma}\label{lem:G-rsc}
Suppose Assumptions \ref{asm:G-jacobian-lip} and \ref{asm:task diversity} hold. There exists a constant
\begin{align*}
c_1 := \frac{3}{2}\bar{J}H^2L_4 + \frac{1}{2}H^2L_4^2\rho_0, \text{ for some } \rho_0>0
\end{align*}
such that, for all $\mathbf{B} \in \mathcal{N}(\mathbf{B}^\star, \mu_3 / c_1)$, it holds that
\begin{align*}
\left\langle \nabla_\mathbf{\mathbf{B}}\mathbb{E}_X{L}_{\mathcal H}\left(\{h_i^\star\}_{i\in \mathcal H}, \mathbf{B}\right), \mathbf{B}-\mathbf{B}^\star\right\rangle \geq  \mu_3\left\|\mathbf{B}-\mathbf{B}^\star\right\|^2.
\end{align*}
\end{lemma}

\begin{proof}
We first write
\begin{align*}
\nabla_\mathbf{\mathbf{B}} \mathbb{E}_X{L}_i(h_i^\star,\mathbf{B}) = 2 \mathbb{E}_{X\sim\mathsf P_i}\left[J_B^\top(X) h_i^{\star\top}h_i^\star\,\left(\phi_B(X)-\phi_{\mathbf{B}^\star}(X)\right)\right].
\end{align*}
which implies
\begin{align*}
\left\langle \nabla_\mathbf{\mathbf{B}} \mathbb{E}_X{L}_i(h_i^\star,\mathbf{B}), \mathbf{B}-\mathbf{B}^\star\right\rangle = 2 \mathbb{E}_X\left[\left(\phi_B(X)-\phi_{\mathbf{B}^\star}(X)\right)^\top h_i^{\star\top}h_i^\star J_B(X)(\mathbf{B}-\mathbf{B}^\star)\right].
\end{align*}

In addition, we write $\phi_B(X)-\phi_{\mathbf{B}^\star}(X) = J_{\mathbf{B}^\star}(X)(\mathbf{B}-\mathbf{B}^\star) + R_1(X)$, where $R_1(X):= \phi_B(X) - \phi_{\mathbf{B}^\star}(X) - J_{\mathbf{B}^\star}(X)(\mathbf{B}-\mathbf{B}^\star)$. Therefore, we have $\|R_1(X)\| \leq \tfrac{L_4}{2}\|\mathbf{B}-\mathbf{B}^\star\|^2$ (i.e., from Lemma \ref{lem:G-taylor}). 

Moreover, $J_B(X)(\mathbf{B}-\mathbf{B}^\star) = J_{\mathbf{B}^\star}(X)(\mathbf{B}-\mathbf{B}^\star) + R_2(X)$, where $R_2(X):= (J_B(X) - J_{\mathbf{B}^\star}(X))(\mathbf{B}-\mathbf{B}^\star)$ with $\|R_2(X)\| \leq L_4\|\mathbf{B}-\mathbf{B}^\star\|^2$ (from Assumption \ref{asm:G-jacobian-lip} directly). By expanding the bilinear form, we have
\begin{align*}
\left(\phi_B(X)-\phi_{\mathbf{B}^\star}(X)\right)^\top h_i^{\star\top}h_i^\star J_B(X)(\mathbf{B}-\mathbf{B}^\star) &= (\mathbf{B}-\mathbf{B}^\star)^\top J_{\mathbf{B}^\star}^\top(X)h_i^{\star\top}h_i^\star J_{\mathbf{B}^\star}(X)(\mathbf{B}-\mathbf{B}^\star) \\
&+ (\mathbf{B}-\mathbf{B}^\star)^\top J_{\mathbf{B}^\star}^\top(X)h_i^{\star\top}h_i^\star R_2(X)\\
&+ R_1(X)^\top h_i^{\star\top}h_i^\star J_{\mathbf{B}^\star}(X)(\mathbf{B}-\mathbf{B}^\star)\\
&+ R_1(X)^\top h_i^{\star\top}h_i^\star R_2(X),
\end{align*}
we can then use  $\|J_{\mathbf{B}^\star}(X)\| \leq \bar{J}$, $\|h_i^{\star\top}h_i^\star\| \leq H^2$, and the three cross terms are bounded in norm by, respectively,
\begin{align*}
\bar{J}H^2L_4\left\|\mathbf{B}-\mathbf{B}^\star\right\|^3,\; \frac{1}{2}\bar{J}H^2L_4\left\|\mathbf{B}-\mathbf{B}^\star\right\|^3, \text{ and } \frac{1}{2}H^2L_4^2\left\|\mathbf{B}-\mathbf{B}^\star\right\|^4,
\end{align*}
and thus, for $\|\mathbf{B}-\mathbf{B}^\star\|\le\rho_0$, their sum is bounded by $c_1\|\mathbf{B}-\mathbf{B}^\star\|^3$ with $c_1$ as defined above. For the leading quadratic term, we have
\begin{align*}
\frac{1}{|\mathcal H|}
\sum_{i\in\mathcal H}
\mathbb E_X
\left[ (\mathbf{B}-\mathbf{B}^\star)^\top J_{\mathbf{B}^\star}^\top(X) h_i^{\star\top}h_i^\star J_{\mathbf{B}^\star}(X) (\mathbf{B}-\mathbf{B}^\star)
\right] \geq \mu_3
\left\|\mathbf{B}-\mathbf{B}^\star\right\|^2.
\end{align*}
Therefore, after averaging over honest clients, the leading quadratic term satisfies
\begin{align*}
&\frac{1}{|\mathcal H|}
\sum_{i\in\mathcal H}
\mathbb E_X \left[ (\mathbf{B}-\mathbf{B}^\star)^\top J_{\mathbf{B}^\star}^\top(X) h_i^{\star\top}h_i^\star J_{\mathbf{B}^\star}(X) (\mathbf{B}-\mathbf{B}^\star)
\right]  \geq \mu_3 \left\|\mathbf{B}-\mathbf{B}^\star\right\|^2.
\end{align*}
Combining this leading quadratic lower bound with the cross-terms above yields
\begin{align*}
\left\langle \nabla_\mathbf{\mathbf{B}}\mathbb{E}_X L_{\mathcal H}(\{h_i^\star\}_i,\mathbf{B}), \mathbf{B}-\mathbf{B}^\star
\right\rangle \geq  2\mu_3
\left\|\mathbf{B}-\mathbf{B}^\star\right\|^2 - 2c_1
\left\|\mathbf{B}-\mathbf{B}^\star\right\|^3,
\end{align*}
and we obtain
\begin{align*}
\left\langle
\nabla_\mathbf{\mathbf{B}}\mathbb{E}_X L_{\mathcal H}(\{h_i^\star\}_i,\mathbf{B}),
\mathbf{B}-\mathbf{B}^\star \right\rangle
\ge 2\mu_3 \left\|\mathbf{B}-\mathbf{B}^\star\right\|^2 - 2c_1
\left\|\mathbf{B}-\mathbf{B}^\star\right\|^3.
\end{align*}
In particular, for any $\mathbf{B}$ satisfying
\begin{align*}
\left\|\mathbf{B}-\mathbf{B}^\star\right\| \le \rho_1 := \frac{\mu_3}{c_1},
\end{align*}
then
\begin{align*}
\left\langle
\nabla_\mathbf{\mathbf{B}}\mathbb{E}_X L_{\mathcal H}(\{h_i^\star\}_i,\mathbf{B}),
\mathbf{B}-\mathbf{B}^\star \right\rangle \ge \mu_3
\left\|\mathbf{B}-\mathbf{B}^\star\right\|^2,
\end{align*}
which completes the proof.
\end{proof}

\begin{lemma}\label{lem:G-empirical-curvature-learned-heads}
Suppose that the conditions of Lemma \ref{lem:G-rsc} hold with $\mu_3 \geq \frac{4 L_3H \bar{J}\bar{\phi}}{\mu_1}$. Then, for every $\mathbf{B} \in \mathcal{N}(\mathbf{B}^\star,\mu_3/c_1)$, with probability at least $1-\delta$, it holds that
\begin{align*}
\left\langle \nabla_\mathbf{\mathbf{B}} L_{\mathcal H}\left(\{h_i^{(t+1)}\}_{i\in\mathcal H},\mathbf{B}\right),
\mathbf{B}-\mathbf{B}^\star \right\rangle &\geq \frac{\mu_3}{4}\|\mathbf{B}-\mathbf{B}^\star\|^2 - \frac{2C_1\bar J^2H^2}{\mu_3} \left( H\bar{\phi}+\sigma \right)^2  \left(\frac{p+\log\left(2/\delta\right)}
{\tau|\mathcal H|}\right) \\
& - \frac{8L_3^2C_1\bar\phi^2\sigma^2}{\mu_3\mu_1^2}\left(\frac{rq + \log(1/\delta)}{\tau}\right).
\end{align*}
\end{lemma}

\begin{proof}
We begin with the decomposition
\begin{align*}
\nabla_\mathbf{\mathbf{B}} L_{\mathcal H}\left(\{h_i^{(t+1)}\}_{i\in\mathcal H},\mathbf{B}\right)
&= \nabla_\mathbf{\mathbf{B}}\mathbb{E}_X L_{\mathcal H}\left(\{h_i^\star\}_{i\in\mathcal H},\mathbf{B}\right)\\
&+\left[\nabla_\mathbf{\mathbf{B}}\mathbb{E}_X L_{\mathcal H}\left(\{h_i^{(t+1)}\}_{i\in\mathcal H},\mathbf{B}\right)
-\nabla_\mathbf{\mathbf{B}}\mathbb{E}_X L_{\mathcal H}\left(\{h_i^\star\}_{i\in\mathcal H},\mathbf{B}\right)
\right]\\
&+\left[\nabla_\mathbf{\mathbf{B}} L_{\mathcal H}\left(\{h_i^{(t+1)}\}_{i\in\mathcal H},\mathbf{B}\right)
-\nabla_\mathbf{\mathbf{B}}\mathbb{E}_X L_{\mathcal H}\left(\{h_i^{(t+1)}\}_{i\in\mathcal H},\mathbf{B}\right)\right],
\end{align*}
and by taking the inner product with $\mathbf{B}-\mathbf{B}^\star$, we obtain
\begin{align*}
&\left\langle
\nabla_\mathbf{\mathbf{B}} L_{\mathcal H}\left(\{h_i^{(t+1)}\}_{i\in\mathcal H},\mathbf{B}\right), \mathbf{B}-\mathbf{B}^\star \right\rangle = \left\langle \nabla_\mathbf{\mathbf{B}}\mathbb{E}_X L_{\mathcal H}\left(\{h_i^\star\}_{i\in\mathcal H},\mathbf{B}\right), \mathbf{B}-\mathbf{B}^\star
\right\rangle\\
&+\left\langle
\nabla_\mathbf{\mathbf{B}}\mathbb{E}_X L_{\mathcal H}\left(\{h_i^{(t+1)}\}_{i\in\mathcal H},\mathbf{B}\right)
-\nabla_\mathbf{\mathbf{B}}\mathbb{E}_X L_{\mathcal H}\left(\{h_i^\star\}_{i\in\mathcal H},\mathbf{B}\right),
\mathbf{B}-\mathbf{B}^\star\right\rangle\\
&+\left\langle
\nabla_\mathbf{\mathbf{B}} L_{\mathcal H}\left(\{h_i^{(t+1)}\}_{i\in\mathcal H},\mathbf{B}\right)
-\nabla_\mathbf{\mathbf{B}}\mathbb{E}_X L_{\mathcal H}\left(\{h_i^{(t+1)}\}_{i\in\mathcal H},\mathbf{B}\right),
\mathbf{B}-\mathbf{B}^\star\right\rangle.
\end{align*}
The first term is controlled by Lemma \ref{lem:G-rsc}, i.e.,
\begin{align*}
\left\langle
\nabla_\mathbf{\mathbf{B}}\mathbb{E}_X L_{\mathcal H}\left(\{h_i^\star\}_{i\in\mathcal H},\mathbf{B}\right),
\mathbf{B}-\mathbf{B}^\star\right\rangle \ge
\mu_3\left\|\mathbf{B}-\mathbf{B}^\star\right\|^2.
\end{align*}
For the second term, by Cauchy-Schwarz,
\begin{align*}
&\left| \left\langle
\nabla_\mathbf{\mathbf{B}}\mathbb{E}_X L_{\mathcal H}\left(\{h_i^{(t+1)}\}_{i\in\mathcal H},\mathbf{B}\right)
- \nabla_\mathbf{\mathbf{B}}\mathbb{E}_X L_{\mathcal H}\left(\{h_i^\star\}_{i\in\mathcal H},\mathbf{B}\right), \mathbf{B}-\mathbf{B}^\star
\right\rangle
\right|\\
&\leq \left\|
\nabla_\mathbf{\mathbf{B}}\mathbb{E}_X L_{\mathcal H}\left(\{h_i^{(t+1)}\}_{i\in\mathcal H},\mathbf{B}\right)
-\nabla_\mathbf{\mathbf{B}}\mathbb{E}_X L_{\mathcal H}\left(\{h_i^\star\}_{i\in\mathcal H},\mathbf{B}\right)
\right\| \left\|\mathbf{B}-\mathbf{B}^\star\right\|.
\end{align*}
By the $L_3$ cross-gradient smoothness in Assumption~\ref{assm:smooth}, we obtain
\begin{align*}
&\left\| \nabla_\mathbf{\mathbf{B}}\mathbb{E}_X L_{\mathcal H}\left(\{h_i^{(t+1)}\}_{i\in\mathcal H},\mathbf{B}\right) -
\nabla_\mathbf{\mathbf{B}}\mathbb{E}_X L_{\mathcal H}\left(\{h_i^\star\}_{i\in\mathcal H},\mathbf{B}\right)
\right\| \leq
\frac{L_3}{|\mathcal H|}
\sum_{i\in\mathcal H}
\left\|
h_i^{(t+1)}-h_i^\star
\right\|,
\end{align*}
where
\begin{align*}
\left\|h_i^{(t+1)}-h_i^\star\right\| \leq \frac{2H\bar{J}\bar\phi}{\mu_1}\|\mathbf{B} - \mathbf{B}^\star\| + \frac{2C_1\bar\phi\sigma}{\mu_1}\sqrt{\frac{rq + \log(1/\delta)}{\tau}}
\end{align*}
which implies 
\begin{align*}
&\left| \left\langle
\nabla_\mathbf{\mathbf{B}}\mathbb{E}_X L_{\mathcal H}\left(\{h_i^{(t+1)}\}_{i\in\mathcal H},\mathbf{B}\right)
- \nabla_\mathbf{\mathbf{B}}\mathbb{E}_X L_{\mathcal H}\left(\{h_i^\star\}_{i\in\mathcal H},\mathbf{B}\right), \mathbf{B}-\mathbf{B}^\star
\right\rangle
\right| \leq \frac{2L_3H\bar{J}\bar\phi}{\mu_1}\|\mathbf{B} - \mathbf{B}^\star\|^2\\
&+ \frac{2L_3C_1\bar\phi\sigma \|\mathbf{B} - \mathbf{B}^\star\|}{\mu_1}\sqrt{\frac{rq + \log(1/\delta)}{\tau}}\\
&\leq \frac{2L_3H\bar{J}\bar\phi}{\mu_1}\|\mathbf{B} - \mathbf{B}^\star\|^2 + \frac{\mu_3}{8}\|\mathbf{B} - \mathbf{B}^\star\|^2 +  \frac{8L_3^2C_1\bar\phi^2\sigma^2}{\mu_3\mu_1^2}\left(\frac{rq + \log(1/\delta)}{\tau}\right)
\end{align*}
where the last inequality follows from Young's inequality. We proceed by recalling that 
\begin{align*}
e_{i,k}^{(t)} :=
h_i^{(t+1)}\phi_B(X_{i,k}) - h_i^\star\phi_{\mathbf{B}^\star}(X_{i,k}),
\end{align*}
and as $Y_{i,k}=h_i^\star\phi_{\mathbf{B}^\star}(X_{i,k})+V_{i,k}$, the empirical gradient can be written as
\begin{align*}
\nabla_\mathbf{\mathbf{B}} L_i\left(h_i^{(t+1)},\mathbf{B}\right) &=
\frac{2}{\tau} \sum_{k=1}^{\tau} J_B^\top(X_{i,k})
h_i^{(t+1)\top}e_{i,k}^{(t)}-\frac{2}{\tau}
\sum_{k=1}^{\tau}J_B^\top(X_{i,k})h_i^{(t+1)\top}
V_{i,k}.
\end{align*}
Thus, we have that 
\begin{align*}
\nabla_\mathbf{\mathbf{B}} L_{\mathcal H}\left(\{h_i^{(t+1)}\}_{i\in\mathcal H},\mathbf{B}\right)
&-\nabla_\mathbf{\mathbf{B}}\mathbb{E}_X L_{\mathcal H}\left(\{h_i^{(t+1)}\}_{i\in\mathcal H},\mathbf{B}\right) =  -\frac{2}{|\mathcal H|\tau}
\sum_{i\in\mathcal H}
\sum_{k=1}^{\tau}
J_B^\top(X_{i,k})
h_i^{(t+1)\top}
V_{i,k}\\
& + \frac{1}{|\mathcal H|}
\sum_{i\in\mathcal H} \left[ \frac{2}{\tau}
\sum_{k=1}^{\tau} J_B^\top(X_{i,k}) h_i^{(t+1)\top}
e_{i,k}^{(t)}- 2\mathbb E_X \left[ J_B^\top(X) h_i^{(t+1)\top} e_i^{(t)}(X)\right]\right]
\end{align*}
We bound the two terms separately. For the first term, define the centered vectors
\begin{align*}
Z_{i,k}^{(t)} :=
2J_B^\top(X_{i,k})h_i^{(t+1)\top}e_{i,k}^{(t)}
- 2\mathbb E_X \left[
J_B^\top(X)h_i^{(t+1)\top}e_i^{(t)}(X)\right]
\in\mathbb R^{p}.
\end{align*}
For every unit vector $u\in\mathbb R^{p}$,
\begin{align*}
\left|\left\langle u, J_B^\top(X)h_i^{(t+1)\top}e_i^{(t)}(X)
\right\rangle \right| &=\left|\left\langle h_i^{(t+1)}J_B(X)u,
e_i^{(t)}(X)\right\rangle
\right| \leq \left\|h_i^{(t+1)}\right\|
\left\|J_B(X)u\right\|\left\|e_i^{(t)}(X)\right\|\\
&\leq
\bar JH \left\| e_i^{(t)}(X) \right\|.
\end{align*}
Hence, $Z_{i,k}^{(t)}$ is $2\bar J H^2 \bar{\phi}$-sub-Gaussian. By applying Lemma \ref{lem:matrix_subg_mean} to the average over $\tau|\mathcal H|$ independent $p$-dimensional random vectors gives, with probability at least $1-\delta/2$,
\begin{align*}
\left\| \frac{1}{|\mathcal H|\tau} \sum_{i\in\mathcal H} \sum_{k=1}^{\tau}
Z_{i,k}^{(t)} \right\| &\leq
2C_1 \bar J H^2 \bar{\phi} \sqrt{
\frac{p+\log\left(2/\delta\right)}
{\tau|\mathcal H|}
}.
\end{align*}

For the second term, by applying Lemma \ref{lem:matrix_subg_mean} again yields, with probability at least $1-\delta/2$,
\begin{align*}
\left\|
\frac{1}{|\mathcal H|\tau}
\sum_{i\in\mathcal H}
\sum_{k=1}^{\tau}
J_B^\top(X_{i,k})
h_i^{(t+1)\top}
V_{i,k}\right\|
\leq C_1\bar JH\sigma \sqrt{
\frac{p+\log\left(2/\delta\right)} {\tau|\mathcal H|} }.
\end{align*}
By a union bound, both concentration events hold simultaneously with probability at least $1-\delta$. Therefore, we have 
\begin{align*}
&\left\|
\nabla_\mathbf{\mathbf{B}} L_{\mathcal H}\left(\{h_i^{(t+1)}\}_{i\in\mathcal H},\mathbf{B}\right)
-\nabla_\mathbf{\mathbf{B}}\mathbb{E}_X L_{\mathcal H}\left(\{h_i^{(t+1)}\}_{i\in\mathcal H},\mathbf{B}\right)
\right\| \leq  C_1\bar JH \left( H\bar{\phi}+\sigma \right) \sqrt{ \frac{p+\log\left(2/\delta\right)}
{\tau|\mathcal H|}
}.
\end{align*}
By taking the inner product with $\mathbf{B}-\mathbf{B}^\star$ and applying Cauchy-Schwarz yields
\begin{align*}
&\left| \left\langle
\nabla_\mathbf{\mathbf{B}} L_{\mathcal H}\left(\{h_i^{(t+1)}\}_{i\in\mathcal H},\mathbf{B}\right)
-\nabla_\mathbf{\mathbf{B}}\mathbb{E}_X L_{\mathcal H}\left(\{h_i^{(t+1)}\}_{i\in\mathcal H},\mathbf{B}\right),\mathbf{B}-\mathbf{B}^\star
\right\rangle\right|\\
&\leq C_1\bar JH \|\mathbf{B}-\mathbf{B}^\star
\| \left( H\bar{\phi}+\sigma \right) \sqrt{ \frac{p+\log\left(2/\delta\right)}
{\tau|\mathcal H|}}\\
&\leq \frac{\mu_3}{8}\|\mathbf{B}-\mathbf{B}^\star\|^2 + \frac{2C_1\bar J^2H^2}{\mu_3} \left( H\bar{\phi}+\sigma \right)^2  \left(\frac{p+\log\left(2/\delta\right)}
{\tau|\mathcal H|}\right),
\end{align*}
where the last inequality follows from Young's inequality. We then obtain
\begin{align*}
&\left\langle
\nabla_\mathbf{\mathbf{B}} L_{\mathcal H}\left(\{h_i^{(t+1)}\}_{i\in\mathcal H},\mathbf{B}\right), \mathbf{B}-\mathbf{B}^\star \right\rangle \geq  \mu_3 \|\mathbf{B}-\mathbf{B}^\star\|^2 - \frac{2L_3H\bar{J}\bar\phi}{\mu_1}\|\mathbf{B} - \mathbf{B}^\star\|^2 - \frac{\mu_3}{8}\|\mathbf{B} - \mathbf{B}^\star\|^2\\
&-  \frac{8L_3^2C_1\bar\phi^2\sigma^2}{\mu_3\mu_1^2}\left(\frac{rq + \log(1/\delta)}{\tau}\right)-\frac{\mu_3}{8}\|\mathbf{B}-\mathbf{B}^\star\|^2 \\
&- \frac{2C_1\bar J^2H^2}{\mu_3} \left( H\bar{\phi}+\sigma \right)^2  \left(\frac{p+\log\left(2/\delta\right)}
{\tau|\mathcal H|}\right)\\
&\geq \frac{3\mu_3}{4}\|\mathbf{B}-\mathbf{B}^\star\|^2 - \frac{2L_3H\bar{J}\bar\phi}{\mu_1}\|\mathbf{B} - \mathbf{B}^\star\|^2\\
&- \frac{2C_1\bar J^2H^2}{\mu_3} \left( H\bar{\phi}+\sigma \right)^2  \left(\frac{p+\log\left(2/\delta\right)}
{\tau|\mathcal H|}\right) - \frac{8L_3^2C_1\bar\phi^2\sigma^2}{\mu_3\mu_1^2}\left(\frac{rq + \log(1/\delta)}{\tau}\right)\\
& \geq \frac{\mu_3}{4}\|\mathbf{B}-\mathbf{B}^\star\|^2  - \frac{2C_1\bar J^2H^2}{\mu_3} \left( H\bar{\phi}+\sigma \right)^2  \left(\frac{p+\log\left(2/\delta\right)}
{\tau|\mathcal H|}\right) \\
& - \frac{8L_3^2C_1\bar\phi^2\sigma^2}{\mu_3\mu_1^2}\left(\frac{rq + \log(1/\delta)}{\tau}\right),
\end{align*}
where the last inequality follows from requiring that the client diversity satisfies $\mu_3 \geq \frac{4L_3H \bar{J}\bar{\phi}}{\mu_1}$. 
\end{proof}

\begin{remark}
The quantity $\mu_3$ measures the amount of collective information that the honest clients provide for identifying the shared representation. Indeed,
\begin{align*}
\frac{1}{|\mathcal H|}
\sum_{i\in\mathcal H}
\mathbb E_X
\left[
J_{\mathbf{B}^\star}^\top(X)
h_i^{\star\top}h_i^\star
J_{\mathbf{B}^\star}(X)
\right]
\end{align*}
is precisely the honest average local curvature of the expected loss with respect to the representation parameter $\mathbf{B}^\star$. Therefore, $\mu_3>0$ guarantees that every nonzero perturbation of $\mathbf{B}^\star$ changes the predictions of the honest clients on average. In contrast, if $\mu_3=0$, then there exists a nontrivial direction in the representation parameter space that is invisible to all honest clients, then perturbations of $\mathbf{B}^\star$ along that direction do not change the expected loss to first order and therefore it does not change the estimation of the client-specific heads. Therefore, the shared representation is not locally identifiable from the available client population. Moreover, the condition
$$
\mu_3 \geq \frac{4H\bar J\bar\phi}{\mu_1}
$$
further requires that this intrinsic curvature is sufficiently strong to dominate the error introduced by imperfect head estimation. Note that under this condition, the empirical objective preserves positive local curvature around $\mathbf{B}^\star$, allowing the alternating optimization to contract toward the true shared representation parameter (see Theorem \ref{theorem:Delta_BT_nonlinear}). Whether this condition is fundamental remains a direction for future work.
\end{remark}

\begin{lemma}\label{lem:G-momentum-dev}
Let $\bar \delta^{(t)} := \bar{m}_B^{(t)} - \nabla_\mathbf{\mathbf{B}}{L}_{\mathcal H}(\{h_i^{(t+1)}\}_i,\mathbf{B}^{(t)})$ denote the momentum deviation, with $\bar{m}_B^{(t)} := \frac{1}{|\mathcal H|}\sum_{i\in \mathcal H} m_{\mathbf{B},i}^{(t)}$. Suppose the conditional sub-Gaussian gradient noise and unbiasedness assumptions hold, and that $\mathbb{E}_X{L}_H$ is $L_1$-smooth in $\mathbf{B}$. Let $\bar{a} := \beta^2(1+\eta L_1)(1+4\eta L_1)$, $\bar{b} := 4\eta L_1(1+\eta L_1)\beta^2$, with $\beta^2 := 1-24\eta L_1$ and $\eta \le \tfrac{1}{24L_1}$. Then there exists a constant $C_2>0$ such that, for every $\delta\in(0,1)$, with probability at least $1-\delta$,
\begin{align*}
\frac{1}{T}\sum_{t=0}^{T-1}\left\|\bar\delta^{(t)}\right\|^2 &\le \frac{\bar{b}}{1-\bar{a}}\bar{Q}_T + C_2\frac{(1-\beta)^2}{|\mathcal H|(1-\bar{a})}\sigma_g^2\log(T/\delta)\notag \\
&+ \frac{C_2\eta L_1}{1-\bar{a}}(1+\eta L_1)\beta^2\kappa\left(\frac{1-\beta}{1+\beta}\sigma_g^2\log(T/\delta) + \bar{G}_T\right),
\end{align*}
where $\bar{Q}_T := \frac{1}{T}\sum_{t=0}^{T-1}\left\|\nabla_\mathbf{\mathbf{B}}\mathbb{E}_X{L}_{\mathcal H}(\{h_i^{(t+1)}\}_i,\mathbf{B}^{(t)})\right\|^2$.
\end{lemma}

\begin{proof}
The result follows by applying Lemma \ref{lem:momentum_deviation_hp} to the representation gradients and the $\mathbf{B}$-iterates.
\end{proof}

\begin{lemma}\label{lem:G-agg-error}
Let $\bar \zeta^{(t)} := \mathsf{F}\left(m_{\mathbf{B},1}^{(t)},\ldots,m_{\mathbf{B},n}^{(t)}\right) - \bar{m}_B^{(t)}$. Given $\delta\in(0,1)$, with probability at least $1-\delta$,
\begin{align*}
\frac{1}{T}\sum_{t=0}^{T-1}\left\|\bar\zeta^{(t)}\right\|^2 \leq 2\kappa C_2\frac{1-\beta}{1+\beta}\sigma_g^2\log(T/\delta)\left(1+\frac{1}{|\mathcal H|}\right) + 2\kappa \bar{G}_T.
\end{align*}
\end{lemma}

\begin{proof}
The proof is identical to the proof of Lemma \ref{lem:agg_error}, applying the $(f,\kappa)$-robustness aggregation to $\{m_{\mathbf{B},1}^{(t)},\ldots,m_{\mathbf{B},n}^{(t)}\}$ and the per-client momentum-deviation bound of Lemma \ref{lem:momentum_deviation_hp}, adapted to the representation parameter $\mathbf{B}$.
\end{proof}

\begin{theorem} \label{theorem:Delta_BT_nonlinear} Suppose that the conditions of Lemma \ref{lem:G-empirical-curvature-learned-heads} hold, and that the step-size is selected to satisfy
\begin{align*}
    \eta \leq \min\left\{\frac{\mu_3}{8\bar{J}^2H^2q\left(2H^2\bar{J}^2 + \frac{16H^2\bar{J}^2\bar\phi^4}{\mu_1^2}\right)}, \frac{1}{32\bar J^2 H^2 q \bar \phi^2 \mu_3}, \frac{1}{2\mu_3q}, \frac{8}{\mu_3} \right\},
\end{align*}
then, with probability $1-\delta$, it holds that 
\begin{align}\label{eq:expression_Delta_T}
    \Delta_{\mathbf{B},T+1} &\leq \frac{p_3\Delta^{(0)}_{\mathbf{B}}}{\eta T}  + p_4\sigma^2\left(\frac{rq + \log(|\mathcal{H}|T/\delta)}{\tau}\right) \\
&+ p_5 \left(\left( H\bar{\phi}+\sigma \right)^2  + \sigma^2\right)\left(\frac{p+\log\left(|\mathcal{H}|T/\delta\right)}
{\tau|\mathcal H|}\right) + p_6 \frac{1}{T}\sum_{t =0}^{T-1}\|\bar\delta^{(t)}\|^2\notag \\
&+p_7\frac{1-\beta}{1+\beta}\sigma_g^2\log(T/\delta) +  2p_6\kappa \bar{G}_T, \notag\end{align}
where $p_3:= \frac{8}{\mu_3}$, $p_4:= \frac{272 C_1\bar\phi^2}{\mu_3^2\mu_1^2}$, $p_5 := \frac{80 C_1\bar J^2H^2}{\mu_3^2}$, $p_6:= \frac{8}{\mu_3}\left(2+\frac{16}{\mu_3}\right)$, and $p_7:=  \frac{16 \kappa C_2}{\mu_3}\left(2+\frac{16}{\mu_3}\right)\left(1+\frac{1}{|\mathcal H|}\right)$.
\end{theorem}

\begin{proof} We begin by defining $T_{1a} := \mathbf{B}^{(t)} - \mathbf{B}^\star - \eta\nabla_\mathbf{\mathbf{B}}{L}_{\mathcal H}(\{h_i^{(t+1)}\}_i, \mathbf{B}^{(t)})$ and thus 
\begin{align*}
\mathbf{B}^{(t+1)} - \mathbf{B}^\star  =: T_{1a} + \eta T_2.
\end{align*}

\noindent \textbf{Bounding $\|T_{1a}\|^2$.} By expanding the square, we obtain
\begin{align*}
\left\|T_{1a}\right\|^2 = \Delta_B^{(t)} - 2\eta\langle \mathbf{B}^{(t)} - \mathbf{B}^\star,\nabla_\mathbf{\mathbf{B}}{L}_{\mathcal H}(\{h_i^{(t+1)}\}_i, \mathbf{B}^{(t)})\rangle + \eta^2\left\|\nabla_\mathbf{\mathbf{B}}{L}_{\mathcal H}(\{h_i^{(t+1)}\}_i, \mathbf{B}^{(t)})\right\|^2,
\end{align*}
and by using the definition of the gradient with respect to $\mathbf{B}$ of the honest averaged loss at the current iteration, we have for the third term

\begin{align*}
&\left\|\nabla_\mathbf{\mathbf{B}}{L}_{\mathcal H}(\{h_i^{(t+1)}\}_i, \mathbf{B}^{(t)})\right\|^2 \leq 2 \bar{J}^2H^2q \frac{1}{|\mathcal{H}|} \sum_{i \in \mathcal{H}} \frac{1}{\tau} \sum_{k=1}^\tau \|\Delta^{(t)}_{i,k}\|^2\\
&+ 2\left\|\frac{1}{|\mathcal{H}|} \sum_{i \in \mathcal{H}} \frac{1}{\tau} \sum_{k=1}^\tau J^\top_{\mathbf{B}^{( t )}}( X_{i,k} ) h_i^{( t+1 )\top} V_{i,k}\right\|^2\\
&=2 \bar{J}^2H^2q \bar{E}^{(t)} + 2\left\|\frac{1}{|\mathcal{H}|} \sum_{i \in \mathcal{H}} \frac{1}{\tau} \sum_{k=1}^\tau J^\top_{\mathbf{B}^{( t )}}( X_{i,k} ) h_i^{( t+1 )\top}V_{i,k}\right\|^2\\
&\leq 2 \bar{J}^2H^2q \bar{E}^{(t)} + 8C_1 \sigma^2 \bar{J}^2 q H^2 \left(\frac{p + \log( |\mathcal{H}|/\delta )}{|\mathcal{H}|\tau}\right)\\
&\leq 2\bar{J}^2H^2q \left(\left(2H^2\bar{J}^2 + \frac{16H^2\bar{J}^2\bar\phi^4}{\mu_1^2}\right)\Delta_B^{(t)} + \frac{16C_1^2\bar\phi^4\sigma^2}{\mu_1^2}\left(\frac{rq+\log(|\mathcal{H}|/\delta)}{\tau}\right)\right)\\
&+ 8C_1 \sigma^2 \bar{J}^2 q H^2 \left(\frac{p + \log( |\mathcal{H}|/\delta )}{|\mathcal{H}|\tau}\right)\\
&= 2\bar{J}^2H^2q\left(2H^2\bar{J}^2 + \frac{16H^2\bar{J}^2\bar\phi^4}{\mu_1^2}\right)\Delta_B^{(t)} + \frac{32\bar{J}^2H^2qC_1^2\bar\phi^4\sigma^2}{\mu_1^2}\left(\frac{rq+\log(|\mathcal{H}|/\delta)}{\tau}\right)\\
&+ 8C_1 \sigma^2 \bar{J}^2 q H^2 \left(\frac{p + \log( |\mathcal{H}|/\delta )}{|\mathcal{H}|\tau}\right),
\end{align*}
with probability $1-\delta$. Note that the second inequality follows from using Lemma \ref{lem:matrix_subg_mean} for $|\mathcal{H}|\tau$ independent sub-Gaussian random vectors. Therefore, we can write 

\begin{align*}
\left\|T_{1a}\right\|^2 &= \Delta_B^{(t)} - 2\eta\langle \mathbf{B}^{(t)} - \mathbf{B}^\star,\nabla_\mathbf{\mathbf{B}}{L}_{\mathcal H}(\{h_i^{(t+1)}\}_i, \mathbf{B}^{(t)})\rangle + \eta^2\left\|\nabla_\mathbf{\mathbf{B}}{L}_{\mathcal H}(\{h_i^{(t+1)}\}_i, \mathbf{B}^{(t)})\right\|^2\\
&\leq \Delta_B^{(t)} -  \frac{\eta\mu_3}{2}\Delta^{(t)}_B +  \frac{4\eta C_1\bar J^2H^2}{\mu_3} \left( H\bar{\phi}+\sigma \right)^2  \left(\frac{p+\log\left(|\mathcal{H}|/\delta\right)}
{\tau|\mathcal H|}\right) \\
& + \frac{16\eta C_1\bar\phi^2\sigma^2}{\mu_3\mu_1^2}\left(\frac{rq + \log(|\mathcal{H}|/\delta)}{\tau}\right) + 2\eta^2\bar{J}^2H^2q\left(2H^2\bar{J}^2 + \frac{16H^2\bar{J}^2\bar\phi^4}{\mu_1^2}\right)\Delta_B^{(t)}\\
&+\frac{32\eta^2\bar{J}^2H^2qC_1^2\bar\phi^4\sigma^2}{\mu_1^2}\left(\frac{rq+\log(|\mathcal{H}|/\delta)}{\tau}\right)+ 8\eta^2C_1 \sigma^2 \bar{J}^2 q H^2 \left(\frac{p + \log( |\mathcal{H}|/\delta )}{|\mathcal{H}|\tau}\right)\\
&\leq \left(1- \frac{\eta\mu_3}{4}\right)\Delta^{(t)}_B + \frac{17\eta C_1\bar\phi^2\sigma^2}{\mu_3\mu_1^2}\left(\frac{rq + \log(|\mathcal{H}|/\delta)}{\tau}\right)\\
&+ \frac{5\eta C_1\bar J^2H^2}{\mu_3} \left(\left( H\bar{\phi}+\sigma \right)^2 + \sigma^2\right)\left(\frac{p+\log\left(|\mathcal{H}|/\delta\right)}
{\tau|\mathcal H|}\right)
\end{align*}
where the first inequality follows from Lemma \ref{lem:G-empirical-curvature-learned-heads}. The second inequality is due to the following step-size condition

$$\eta \leq \min\left\{\frac{\mu_3}{8\bar{J}^2H^2q\left(2H^2\bar{J}^2 + \frac{16H^2\bar{J}^2\bar\phi^4}{\mu_1^2}\right)}, \frac{1}{32\bar J^2 H^2 q \bar \phi^2 \mu_3}, \frac{1}{2\mu_3q}  \right\}.$$

As $T_2 = \bar\delta^{(t)}+\bar\zeta^{(t)}$, we have that $\|T_2\|^2 \le 2\|\bar\delta^{(t)}\|^2 + 2\|\bar \zeta^{(t)}\|^2$. Therefore, we can write
\begin{align*}
    \Delta^{(t+1)}_B =  \|T_{1a}\|^2 + 2\eta \langle T_{1a},  T_2 \rangle + \eta^2 \|T_2\|^2,
\end{align*}
we then use Young's inequality to obtain $2\eta \langle T_{1a}, T_2 \rangle \leq \gamma \|T_{1a}\|^2 + \frac{1}{\gamma} \|T_2\|^2,$ for some $\gamma >0$. Hence, we can write 
\begin{align*}
    \Delta^{(t+1)}_B &\leq  (1+\gamma)\|T_{1a}\|^2 +  \eta^2\left(1+\frac{1}{\gamma}\right) \|T_2\|^2 \\
    &\leq (1+\gamma)\left(1- \frac{\eta\mu_3}{4}\right)\Delta^{(t)}_B + \frac{17(1+\gamma)\eta C_1\bar\phi^2\sigma^2}{\mu_3\mu_1^2}\left(\frac{rq + \log(|\mathcal{H}|/\delta)}{\tau}\right)\\
    &+ \frac{5(1+\gamma)\eta C_1\bar J^2H^2}{\mu_3} \left(\left( H\bar{\phi}+\sigma \right)^2  + \sigma^2\right)\left(\frac{p+\log\left(|\mathcal{H}|/\delta\right)}
{\tau|\mathcal H|}\right) + \eta^2\left(1+\frac{1}{\gamma}\right) \|T_2\|^2\\
&\leq \left(1-\frac{\eta\mu_3}{8}\right)\Delta^{(t)}_B + \frac{34\eta C_1\bar\phi^2\sigma^2}{\mu_3\mu_1^2}\left(\frac{rq + \log(|\mathcal{H}|/\delta)}{\tau}\right)\\
&+\frac{10\eta C_1\bar J^2H^2}{\mu_3} \left(\left( H\bar{\phi}+\sigma \right)^2  + \sigma^2\right)\left(\frac{p+\log\left(|\mathcal{H}|/\delta\right)}
{\tau|\mathcal H|}\right) + \eta\left(2+\frac{16}{\mu_3}\right) \left(\|\bar\delta^{(t)}\|^2 + \|\bar \zeta^{(t)}\|^2\right),
\end{align*}
where the third inequality follows from selecting $\gamma = \frac{\eta \mu_3}{8}$, which implies that 
$$
\left(1+\frac{\eta\mu_3}{8}\right)\left(1-\frac{\eta\mu_3}{4}\right) = 1-\frac{\eta\mu_3}{8}
-\frac{(\eta\mu_3)^2}{32}\\
\leq 1-\frac{\eta\mu_3}{8},
$$
and the remaining terms in the inequality follow from $\eta \leq \frac{8}{\mu_3}.$ We then average over the iterations $t \in \{0,1,\ldots,T-1\}$ to obtain

\begin{align*}
    \frac{1}{T}\sum_{t =0}^{T-1}\Delta^{(t+1)}_B &\leq  \left(1-\frac{\eta\mu_3}{8}\right)\frac{1}{T}\sum_{t =0}^{T-1}\Delta^{(t)}_B + \frac{34\eta C_1\bar\phi^2\sigma^2}{\mu_3\mu_1^2}\left(\frac{rq + \log(|\mathcal{H}|T/\delta)}{\tau}\right)\\
&+\frac{10\eta C_1\bar J^2H^2}{\mu_3} \left(\left( H\bar{\phi}+\sigma \right)^2  + \sigma^2\right)\left(\frac{p+\log\left(|\mathcal{H}|T/\delta\right)}
{\tau|\mathcal H|}\right)\\
&+ \eta\left(2+\frac{16}{\mu_3}\right) \left(\frac{1}{T}\sum_{t =0}^{T-1}\|\bar\delta^{(t)}\|^2 + \frac{1}{T}\sum_{t =0}^{T-1}\|\bar \zeta^{(t)}\|^2\right),
\end{align*}
which implies 
\begin{align*}
    \Delta_{\mathbf{B},T+1} &\leq  \left(1-\frac{\eta\mu_3}{8}\right)\Delta_{\mathbf{B},T+1} + \left(1-\frac{\eta\mu_3}{8}\right)\frac{\Delta^{(0)}_{\mathbf{B}}}{T}  + \frac{34\eta C_1\bar\phi^2\sigma^2}{\mu_3\mu_1^2}\left(\frac{rq + \log(|\mathcal{H}|T/\delta)}{\tau}\right)\\
&+\frac{10\eta C_1\bar J^2H^2}{\mu_3} \left(\left( H\bar{\phi}+\sigma \right)^2  + \sigma^2\right)\left(\frac{p+\log\left(|\mathcal{H}|T/\delta\right)}
{\tau|\mathcal H|}\right)\\
&+ \eta\left(2+\frac{16}{\mu_3}\right) \left(\frac{1}{T}\sum_{t =0}^{T-1}\|\bar\delta^{(t)}\|^2 + \frac{1}{T}\sum_{t =0}^{T-1}\|\bar \zeta^{(t)}\|^2\right),
\end{align*}
and thus
\begin{align*}
    \Delta_{\mathbf{B},T+1} &\leq \frac{8\Delta^{(0)}_{\mathbf{B}}}{\eta \mu_3 T}  + \frac{272 C_1\bar\phi^2\sigma^2}{\mu_3^2\mu_1^2}\left(\frac{rq + \log(|\mathcal{H}|T/\delta)}{\tau}\right)\\
&+\frac{80 C_1\bar J^2H^2}{\mu_3^2} \left(\left( H\bar{\phi}+\sigma \right)^2  + \sigma^2\right)\left(\frac{p+\log\left(|\mathcal{H}|T/\delta\right)}
{\tau|\mathcal H|}\right)\\
&+ \frac{8}{\mu_3}\left(2+\frac{16}{\mu_3}\right) \left(\frac{1}{T}\sum_{t =0}^{T-1}\|\bar\delta^{(t)}\|^2 + \frac{1}{T}\sum_{t =0}^{T-1}\|\bar \zeta^{(t)}\|^2\right),
\end{align*}
and by leveraging Lemma \ref{lem:G-agg-error}, we write
\begin{align*}
    \Delta_{\mathbf{B},T+1} &\leq \frac{8\Delta^{(0)}_{\mathbf{B}}}{\eta \mu_3 T}  + \frac{272 C_1\bar\phi^2\sigma^2}{\mu_3^2\mu_1^2}\left(\frac{rq + \log(|\mathcal{H}|T/\delta)}{\tau}\right)\notag \\
&+\frac{80 C_1\bar J^2H^2}{\mu_3^2} \left(\left( H\bar{\phi}+\sigma \right)^2  + \sigma^2\right)\left(\frac{p+\log\left(|\mathcal{H}|T/\delta\right)}
{\tau|\mathcal H|}\right) + \frac{8}{\mu_3}\left(2+\frac{16}{\mu_3}\right) \frac{1}{T}\sum_{t =0}^{T-1}\|\bar\delta^{(t)}\|^2\notag \\
&+ \frac{16 \kappa C_2}{\mu_3}\left(2+\frac{16}{\mu_3}\right)\frac{1-\beta}{1+\beta}\sigma_g^2\log(T/\delta)\left(1+\frac{1}{|\mathcal H|}\right) +  \frac{16}{\mu_3}\left(2+\frac{16}{\mu_3}\right)\kappa \bar{G}_T,
\end{align*}
which completes the proof.
\end{proof}

Therefore, we proceed by leveraging Theorem \ref{theorem:Delta_BT_nonlinear}, Lemma \ref{lem:G-momentum-dev}, and \eqref{eq:G-GT-from-DeltaBT} to write the bound on $\bar{G}_T$ for the nonlinear representation learning setting. For this, we can write

\begin{align*}
&\bar{G}_T \leq  p_1 \Delta_{\mathbf{B},T+1} +  p_1 \frac{\Delta^{(0)}_{\mathbf{B}}}{T} + p_2\sigma^2 \left(\frac{p \vee rq + \log(|\mathcal H|T/\delta)}{\tau}\right)\\
&\leq p_1 \Big(\frac{p_3\Delta^{(0)}_{\mathbf{B}}}{\eta T}  + p_4\sigma^2\left(\frac{rq + \log(|\mathcal{H}|T/\delta)}{\tau}\right)\notag + p_5 \left(\left( H\bar{\phi}+\sigma \right)^2  + \sigma^2\right)\left(\frac{p+\log\left(|\mathcal{H}|T/\delta\right)}
{\tau|\mathcal H|}\right)\\
&\notag + p_6 \frac{1}{T}\sum_{t =0}^{T-1}\|\bar\delta^{(t)}\|^2+p_7\frac{1-\beta}{1+\beta}\sigma_g^2\log(T/\delta) +  2p_6\kappa \bar{G}_T  \Big) +  p_1 \frac{\Delta^{(0)}_{\mathbf{B}}}{T}\\
&+ p_2\sigma^2 \left(\frac{p \vee rq + \log(|\mathcal H|T/\delta)}{\tau}\right)\\
&= \frac{\Delta^{(0)}_{\mathbf{B}}}{T}\left(\frac{p_8}{\eta} +  p_1\right) + p_9 \left(\left( H\bar{\phi}+\sigma \right)^2  + \sigma^2\right)\left(\frac{p+\log\left(|\mathcal{H}|T/\delta\right)}
{\tau|\mathcal H|}\right) + p_{10} \frac{1}{T}\sum_{t =0}^{T-1}\|\bar\delta^{(t)}\|^2\notag \\
&+p_{11}\frac{1-\beta}{1+\beta}\sigma_g^2\log(T/\delta) +  2p_{10}\kappa \bar{G}_T  + \sigma^2 \left(p_2 \left(\frac{p \vee rq + \log(|\mathcal H|T/\delta)}{\tau}\right) + p_{12}\left(\frac{rq + \log(|\mathcal{H}|T/\delta)}{\tau}\right) \right)\\
&\leq \frac{\Delta^{(0)}_{\mathbf{B}}}{T}\left(\frac{p_8}{\eta} +  p_1\right) + p_9 \left(\left( H\bar{\phi}+\sigma \right)^2  + \sigma^2\right)\left(\frac{p+\log\left(|\mathcal{H}|T/\delta\right)}
{\tau|\mathcal H|}\right) +p_{11}\frac{1-\beta}{1+\beta}\sigma_g^2\log(T/\delta)  \notag \\
&+  2p_{10}\kappa \bar{G}_T  + \sigma^2 \left(p_2 \left(\frac{p \vee rq + \log(|\mathcal H|T/\delta)}{\tau}\right) + p_{12}\left(\frac{rq + \log(|\mathcal{H}|T/\delta)}{\tau}\right) \right)\\
&+ p_{10} \left(\frac{\bar b}{1-\bar a}\bar{Q}_T + C_2\frac{(1-\beta)^2}{|\mathcal H|(1-\bar a)}\sigma_g^2\log(T/\delta) + \frac{C_2\eta L_1}{1-\bar a}(1+\eta L_1)\beta^2\kappa\left(\frac{1-\beta}{1+\beta}\sigma_g^2\log(T/\delta) + \bar{G}_T\right)\right)\\
&=\frac{\Delta^{(0)}_{\mathbf{B}}}{T}\left(\frac{p_8}{\eta} +  p_1\right) + p_9 \left(\left( H\bar{\phi}+\sigma \right)^2  + \sigma^2\right)\left(\frac{p+\log\left(|\mathcal{H}|T/\delta\right)}
{\tau|\mathcal H|}\right) +p_{13}\frac{1-\beta}{1+\beta}\sigma_g^2\log(T/\delta)  \notag \\
&+  p_{14}\kappa \bar{G}_T  + \sigma^2 \left(p_2 \left(\frac{p \vee rq + \log(|\mathcal H|T/\delta)}{\tau}\right) + p_{12}\left(\frac{rq + \log(|\mathcal{H}|T/\delta)}{\tau}\right) \right)\\
&+  p_{10} C_2\frac{(1-\beta)^2}{|\mathcal H|(1-\bar a)}\sigma_g^2\log(T/\delta)+ p_{10} \frac{\bar{b}}{1-\bar{a}}\bar{Q}_T,
\end{align*}
where $p_8:= p_1p_3$, $p_{12}:= p_1p_4$, $p_9:= p_1p_5$, $p_{10}:= p_1p_6$, $p_{11}:= p_1p_7$, $p_{13} = p_{11} + p_{10}\frac{C_2\eta L_1}{1-\bar{a}}(1+\eta L_1)\beta^2\kappa$, and  $p_{14}:= 2p_{10} + p_{10}\frac{C_2\eta L_1}{1-\bar{a}}(1+\eta L_1)\beta^2$. The second inequality follows from applying \eqref{eq:expression_Delta_T}. In addition, the third inequality is due to Lemma \ref{lem:G-momentum-dev}. Therefore, by requiring $\kappa \leq \frac{1}{2p_{14}}$, we obtain

\begin{align*}
    \bar{G}_T &\leq \frac{\Delta^{(0)}_{\mathbf{B}}}{T}\left(\frac{p_8}{\eta} +  p_1\right) + p_9 \left(\left( H\bar{\phi}+\sigma \right)^2  + \sigma^2\right)\left(\frac{p+\log\left(|\mathcal{H}|T/\delta\right)}
{\tau|\mathcal H|}\right) +p_{13}\frac{1-\beta}{1+\beta}\sigma_g^2\log(T/\delta)  \notag \\
&+ \sigma^2 \left(p_2 \left(\frac{p \vee rq + \log(|\mathcal H|T/\delta)}{\tau}\right) + p_{12}\left(\frac{rq + \log(|\mathcal{H}|T/\delta)}{\tau}\right)\right)+   p_{10} C_2\frac{(1-\beta)^2}{|\mathcal H|(1-\bar{a})}\sigma_g^2\log(T/\delta)\\
&+ p_{10} \frac{\bar{b}}{1-\bar{a}}\bar{Q}_T,
\end{align*}
where the factor of $2$ is absorbed by the coefficients. Let us define the following quantities. 

\begin{align*}
p_1 &:= 16\bar{J}^2 q H^2\left(2H^2\bar{J}^2 + \frac{16H^2\bar{J}^2\bar{\phi}^4}{\mu_1^2}\right)
= 32\bar{J}^4 q H^4\left(1 + \frac{8\bar{\phi}^4}{\mu_1^2}\right),\\
p_2 &:= 16\bar{J}^2 q H^2\left(8C_1 + \frac{16C_1^2\bar{\phi}^4}{\mu_1^2}\right),
\end{align*}

\begin{align*}
p_8 &= p_1 \cdot \frac{8}{\mu_3}
= \frac{256\bar{J}^4 q H^4}{\mu_3}\left(1 + \frac{8\bar{\phi}^4}{\mu_1^2}\right),\\
p_{12} &= p_1 \cdot \frac{272 C_1\bar{\phi}^2}{\mu_3^2\mu_1^2}
= \frac{8704 C_1 \bar{J}^4 q H^4 \bar{\phi}^2}{\mu_3^2\mu_1^2}\left(1 + \frac{8\bar{\phi}^4}{\mu_1^2}\right),\\
p_9 &= p_1 \cdot \frac{80 C_1\bar{J}^2 H^2}{\mu_3^2}
= \frac{2560 C_1 \bar{J}^6 q H^6}{\mu_3^2}\left(1 + \frac{8\bar{\phi}^4}{\mu_1^2}\right),\\
p_{10} &= p_1 \cdot \frac{8}{\mu_3}\left(2 + \frac{16}{\mu_3}\right)
= \frac{256\bar{J}^4 q H^4}{\mu_3}\left(2 + \frac{16}{\mu_3}\right)\left(1 + \frac{8\bar{\phi}^4}{\mu_1^2}\right),\\
p_{11} &= p_1 \cdot \frac{16\kappa C_2}{\mu_3}\left(2 + \frac{16}{\mu_3}\right)\left(1 + \frac{1}{|\mathcal{H}|}\right)\\
&= \frac{512\kappa C_2 \bar{J}^4 q H^4}{\mu_3}\left(2 + \frac{16}{\mu_3}\right)\left(1 + \frac{1}{|\mathcal{H}|}\right)\left(1 + \frac{8\bar{\phi}^4}{\mu_1^2}\right),
\end{align*}
\begin{align*}
p_{13} := p_{11} + p_{10}\frac{C_2 \eta L_1}{1-\bar{a}}\left(1+\eta L_1\right)\beta^2\kappa,
\end{align*}
and 
\begin{align*}
p_{13} = \frac{256\bar{J}^4 q H^4}{\mu_3}\left(2 + \frac{16}{\mu_3}\right)\left(1 + \frac{8\bar{\phi}^4}{\mu_1^2}\right)
\left[2\kappa C_2\left(1+\frac{1}{|\mathcal{H}|}\right) + \frac{C_2 \eta L_1\left(1+\eta L_1\right)\beta^2\kappa}{1-\bar{a}}\right].
\end{align*}

\begin{lemma}(Gradient heterogeneity bound for nonlinear representation learning) \label{lemma:GT_nonlinear} Suppose that the conditions of Theorem \ref{theorem:Delta_BT_nonlinear} and Lemma \ref{lem:G-momentum-dev} hold with $\kappa \leq \frac{1}{2p_{14}}$. Then, with probability at least $1-\delta$, it holds that
\begin{align*}
     \bar{G}_T &\leq \frac{\Delta^{(0)}_{\mathbf{B}}}{T}\left(\frac{p_8}{\eta} +  p_1\right) + p_9 \left(\left( H\bar{\phi}+\sigma \right)^2  + \sigma^2\right)\left(\frac{p+\log\left(|\mathcal{H}|T/\delta\right)}
{\tau|\mathcal H|}\right) +p_{13}\frac{1-\beta}{1+\beta}\sigma_g^2\log(T/\delta)  \notag \\
&+ \sigma^2 \left(p_2 \left(\frac{p \vee rq + \log(|\mathcal H|T/\delta)}{\tau}\right) + p_{12}\left(\frac{rq + \log(|\mathcal{H}|T/\delta)}{\tau}\right)\right)+ p_{10} \frac{\bar{b}}{1-\bar{a}}\bar{Q}_T  \notag \\
&+ p_{10} C_2\frac{(1-\beta)^2}{|\mathcal H|(1-\bar{a})}\sigma_g^2\log(T/\delta).   
\end{align*}
\end{lemma}
\begin{proof}
The proof follows from the above derivations.
\end{proof}

\noindent\textbf{Order of the condition on $\kappa$.} We recall that the gradient heterogeneity bound requires
\begin{align*}
\kappa \le \frac{1}{p_{10}}.
\end{align*}
In particular, using the scaling of the constants, we have
\begin{align*}
p_1 = 32\bar{J}^4 q H^4\left(1 + \frac{8\bar{\phi}^4}{\mu_1^2}\right) = \mathcal{O}\left(\bar{J}^4 q H^4\left(1+\frac{\bar\phi^4}{\mu_1^2}\right)\right),
\end{align*}
and therefore
\begin{align*}
p_{10} = p_1\cdot\frac{8}{\mu_3}\left(2+\frac{16}{\mu_3}\right)
= \mathcal{O}\left(\bar{J}^4 q H^4\left(1+\frac{\bar\phi^4}{\mu_1^2}\right)\left(\frac{1}{\mu_3}+\frac{1}{\mu_3^2}\right)\right).
\end{align*}
Therefore, we have
\begin{align*}
\kappa = \mathcal{O}\left(\left[\bar{J}^4 q H^4\left(1+\frac{\bar\phi^4}{\mu_1^2}\right)\left(\frac{1}{\mu_3}+\frac{1}{\mu_3^2}\right)\right]^{-1}\right).
\end{align*}
Hence, the robust aggregation coefficient scales with the client-diversity (representation
identifiability) parameter $\mu_3$ through the combined factor $\left(\mu_3^{-1}+\mu_3^{-2}\right)^{-1}$, and inversely with the Jacobian bound $\bar{J}$, the uniform head bound $H$, the output dimension $q$, and the curvature ratio $\bar\phi^4/\mu_1^2$ inherited from the Polyak-\L{}ojasiewicz condition. This is the direct nonlinear analogue of
the condition $\kappa = \mathcal{O}\left((\lambda^h_{\min})^2/(H^4R^4(2R^2+1)^2)\right)$ obtained in the linear
representation setting, with $\mu_3$ playing the role of $\lambda^h_{\min}$ and $(\bar J,\bar\phi,\mu_1)$
replacing the covariate bound $R$.

\section{Ergodic Convergence Analysis} \label{sec:convergence_nonlinear_rep}

We now move to prove the convergence for the nonlinear representation learning setting, bounding $\bar{Q}_T$. For this, let us first define
$p_{16} := 24\kappa p_9$, $p_{17} = 24\kappa p_2$, $p_{18}:= 24\kappa p_{12}$, $p_{15}:=24p_{10}\frac{\bar{b}}{1-\bar{a}}$, $$
p_{19}:= 1152 \kappa L_1 \left(C_2 \left(1+\frac1{|\mathcal H|}\right)+p_{13}\right) + 576 L_1^2 \left(\frac{1}{L_1}+\frac{48\kappa \eta p_{10}}{(1-\bar{a})}\right)\frac{C_2}{|\mathcal H|},
$$
$p_{20} = 48\kappa\left(\frac{p_8}{\eta} +  p_1\right)$, $p_{21} := 4p_{16}$, $p_{22} := 2p_{17}$, and $p_{23} := 2p_{18}$.

\begin{theorem}\label{theorem:convergence_nonlinear} Suppose the conditions of Lemma \ref{lemma:GT_nonlinear} hold. Suppose that $\kappa \leq \frac{1}{2p_{15}}$. Then, for every $\delta \in (0,1)$, with probability $1-\delta$, it holds that
\begin{align}\label{eq:convergence_nonlinear}
    \bar Q_T &\leq  \frac{16\left(
L_{\mathcal H}(\{h_i^{(1)}\}_{i\in\mathcal H},\mathbf{B}^{(0)})
-L_{\mathcal{H}}(\{h^\star_i\}_{i \in \mathcal{H}}, \mathbf{B}^\star)
\right)}{\eta T}
+\frac{1}{L\eta T} \left\| \nabla_\mathbf{\mathbf{B}} L_{\mathcal H}(\{h_i^{(1)}\}_{i\in\mathcal H},\mathbf{B}^{(0)})\right\|^2
\\
&+ p_{20}\frac{\Delta^{(0)}_{\mathbf{B}}}{T} + p_{19}\eta \sigma_g^2\log(T/\delta)+ p_{21} \left(\left( H\bar{\phi}+\sigma \right)^2  + \sigma^2\right)\left(\frac{p+\log\left(|\mathcal{H}|T/\delta\right)}
{\tau|\mathcal H|}\right)\notag \\
&+\sigma^2 \left(p_{22} \left(\frac{p \vee rq + \log(|\mathcal H|T/\delta)}{\tau}\right) + p_{23}\left(\frac{rq + \log(|\mathcal{H}|T/\delta)}{\tau}\right)\right). \notag\end{align}    
\end{theorem}

\begin{proof}
By using $L_1$-smoothness and the Lyapunov function formed by the honest-average loss and the squared representation-gradient norm, we obtain 

\begin{align*}
\bar Q_T &\leq  \frac{8\left(
L_{\mathcal H}(\{h_i^{(1)}\}_{i\in\mathcal H},\mathbf{B}^{(0)})
-L_{\mathcal{H}}(\{h^\star_i\}_{i \in \mathcal{H}}, \mathbf{B}^\star)
\right)}{\eta T}
+\frac{1}{2L_1\eta T} \left\| \nabla_\mathbf{\mathbf{B}} L_{\mathcal H}(\{h_i^{(1)}\}_{i\in\mathcal H},\mathbf{B}^{(0)})\right\|^2
\nonumber\\
&+ 24\kappa C_2 \frac{1-\beta}{1+\beta} \sigma_g^2\log(T/\delta) \left(1+\frac1{|\mathcal H|}\right)
+ 24\kappa \bar{G}_T + \frac{C_2(1-\beta)^2}{2\eta L_1|\mathcal H|}
\sigma_g^2\log(T/\delta),
\end{align*}
and by invoking Lemma \ref{lemma:GT_nonlinear}, we can write 
\begin{align*}
\bar Q_T &\leq  \frac{8\left(
L_{\mathcal H}(\{h_i^{(1)}\}_{i\in\mathcal H},\mathbf{B}^{(0)})
-L_{\mathcal{H}}(\{h^\star_i\}_{i \in \mathcal{H}}, \mathbf{B}^\star)
\right)}{\eta T}
+\frac{1}{2L_1\eta T} \left\| \nabla_\mathbf{\mathbf{B}} L_{\mathcal H}(\{h_i^{(1)}\}_{i\in\mathcal H},\mathbf{B}^{(0)})\right\|^2
\nonumber\\
&+ 24\kappa\frac{\Delta^{(0)}_{\mathbf{B}}}{T}\left(\frac{p_8}{\eta} +  p_1\right) + 24\kappa \left(C_2 \left(1+\frac1{|\mathcal H|}\right)+p_{13}\right) \frac{1-\beta}{1+\beta} \sigma_g^2\log(T/\delta)\\
&+ \left(\frac{1}{2\eta L_1}+\frac{24\kappa p_{10}}{(1-\bar{a})}\right)\frac{C_2(1-\beta)^2}{|\mathcal H|}
\sigma_g^2\log(T/\delta)\\
&+ p_{16} \left(\left( H\bar{\phi}+\sigma \right)^2  + \sigma^2\right)\left(\frac{p+\log\left(|\mathcal{H}|T/\delta\right)}
{\tau|\mathcal H|}\right)\\
&+ \sigma^2 \left(p_{17} \left(\frac{p \vee rq + \log(|\mathcal H|T/\delta)}{\tau}\right) + p_{18}\left(\frac{rq + \log(|\mathcal{H}|T/\delta)}{\tau}\right)\right)+ \kappa p_{15}\bar{Q}_T,
\end{align*}

Therefore, by requiring that $\kappa \leq \frac{1}{2p_{15}}$, we obtain 
\begin{align*}
\bar Q_T &\leq  \frac{16\left(
L_{\mathcal H}(\{h_i^{(1)}\}_{i\in\mathcal H},\mathbf{B}^{(0)})
-L_{\mathcal{H}}(\{h^\star_i\}_{i \in \mathcal{H}}, \mathbf{B}^\star)
\right)}{\eta T}
+\frac{1}{L\eta T} \left\| \nabla_\mathbf{\mathbf{B}} L_{\mathcal H}(\{h_i^{(1)}\}_{i\in\mathcal H},\mathbf{B}^{(0)})\right\|^2
\nonumber\\
&+ p_{20}\frac{\Delta^{(0)}_{\mathbf{B}}}{T} + p_{19}\eta \sigma_g^2\log(T/\delta)+ p_{21} \left(\left( H\bar{\phi}+\sigma \right)^2  + \sigma^2\right)\left(\frac{p+\log\left(|\mathcal{H}|T/\delta\right)}
{\tau|\mathcal H|}\right) \\
&+\sigma^2 \left(p_{22} \left(\frac{p \vee rq + \log(|\mathcal H|T/\delta)}{\tau}\right) + p_{23}\left(\frac{rq + \log(|\mathcal{H}|T/\delta)}{\tau}\right)\right),
\end{align*}
by using that $\beta^2=1-24\eta L_1$, we have
$
1-\beta^2=24\eta L_1,
$
which implies
$$
\frac{1-\beta}{1+\beta} \leq 1-\beta^2 =
24\eta L_1 \text{ and }
(1-\beta)^2 \leq (1-\beta^2)^2 = 576\eta^2(L)^2,
$$
and completes the proof.
\end{proof}

\subsection{Order of Coefficients in the Final Bound}

Treating the problem dependent constants $\bar{J}, H, q, \bar{\phi}, \mu_1, \mu_3, L$ as
fixed and tracking only the dependence on $\eta$ and $\kappa$,
\begin{align*}
p_1 = \mathcal{O}(1), \;
p_2 = \mathcal{O}(1), \;
p_8 = \mathcal{O}(1),\;
p_{12} = \mathcal{O}(1), \;
p_9 = \mathcal{O}(1), \text{ and }
p_{10} = \mathcal{O}(1),
\end{align*}
since none of these depend on the step-size $\eta$. We also have that
\begin{align*}
p_{11} = \mathcal{O}(\kappa),
\end{align*}
directly from its definition. Moreover, since $\beta^2 = 1-24\eta L_1$ we have
\begin{align*}
1-\bar{a} = \Theta(\eta), \; \bar{b} = \mathcal{O}(\eta),\text{ and hence } 
\frac{\bar{b}}{1-\bar{a}} = \mathcal{O}(1),
\end{align*}
thus, by using $\eta L_1/(1-\bar a) = \mathcal{O}(1)$, we obtain
\begin{align*}
p_{13} = p_{11} + p_{10}\frac{C_2 \eta L_1}{1-\bar a}(1+\eta L_1)\beta^2\kappa = \Theta(\kappa).
\end{align*}

\noindent\textbf{Coefficients appearing in the convergence bound.} By using the above scalings, we obtain
\begin{align*}
p_{18} &= 24\kappa p_{12} = \mathcal{O}(\kappa), \\
p_{16} &= 24\kappa p_9 = \mathcal{O}(\kappa), \\
p_{17} &= 24\kappa p_2 = \mathcal{O}(\kappa), \\
p_{15} &= 24p_{10}\frac{\bar{b}}{1-\bar{a}} = \mathcal{O}(1), \\
p_{20} &= 48\kappa\left(\frac{p_8}{\eta} + p_1\right) = \mathcal{O}\left(\frac{\kappa}{\eta}\right), \\
p_{21} &= 4p_{16} = \mathcal{O}(\kappa), \\
p_{22} &= 2p_{17} = \mathcal{O}(\kappa), \\
p_{23} &= 2p_{18} = \mathcal{O}(\kappa).
\end{align*}
For the stochastic-gradient coefficient, we decompose
\begin{align*}
p_{19} = \underbrace{1152\kappa L\left(C_2\left(1+\frac1{|\mathcal H|}\right) + p_{13}\right)}_{\mathcal{O}(\kappa+\kappa^2)}
+ \underbrace{576 L_1^2\left(\frac{1}{L_1} + \frac{48\kappa\eta p_{10}}{1-\bar{a}}\right)\frac{C_2}{|\mathcal H|}}_{\mathcal{O}\left(\frac{L_1(1+\kappa)}{|\mathcal H|}\right)},
\end{align*}
where the second term follows from $\eta/(1-\bar{a}) = \mathcal{O}(1/L)$. Combining both terms, we obtain
\begin{align*}
p_{19} = \mathcal{O}\left(\kappa + \kappa^2 + \frac{1+\kappa}{|\mathcal H|}\right).
\end{align*}

\noindent\textbf{Robust aggregation coefficient.} The convergence bound requires
\begin{align*}
\kappa \leq \frac{1}{2p_{15}}.
\end{align*}
Since $p_{15} = 24p_{10}\bar{b}/(1-\bar{a}) = \mathcal{O}(p_{10})$, namely
\begin{align*}
\kappa = \mathcal{O}\left(\left[\bar{J}^4 q H^4\left(1+\frac{\bar\phi^4}{\mu_1^2}\right)\left(\frac{1}{\mu_3}+\frac{1}{\mu_3^2}\right)\right]^{-1}\right).
\end{align*}
Hence a single condition on $\kappa$ (equivalently, on $f/n$ once robust aggregation is composed with NNM, thus $\kappa = \mathcal{O}(f/n)$) suffices to guarantee both the gradient-heterogeneity bound and the convergence bound simultaneously.\\

\noindent\textbf{Final bound.} Setting the step-size according to $\eta = \min\{1/(24L_1), 1/\sqrt{T}\}$,
the leading initialization terms scale as $\mathcal{O}(1/\sqrt{T})$, and
\begin{align*}
p_{20}\frac{\Delta^{(0)}_B}{T} = \mathcal{O}\left(\frac{\kappa}{\eta}\right)\frac{\Delta^{(0)}_B}{T}
= \mathcal{O}\left(\frac{\kappa\,\Delta^{(0)}_B}{\sqrt{T}}\right).
\end{align*}
The finite-sample statistical terms, governed by $p_{21}, p_{22}, p_{23}$, do not
depend on $\eta$ or $T$ directly and instead decay with $\tau$ and $|\mathcal H|$ as
\begin{align*}
\mathcal{O}(\kappa)\left(\frac{p+\log(|\mathcal H|T/\delta)}{\tau|\mathcal H|}\right)
+ \mathcal{O}(\kappa)\left(\frac{p\vee rq + \log(|\mathcal H|T/\delta)}{\tau}\right)
+ \mathcal{O}(\kappa)\left(\frac{rq+\log(|\mathcal H|T/\delta)}{\tau}\right).
\end{align*}
Piecing everything together, the bound on $\bar{Q}_T$ takes the form
\begin{align*}
\bar{Q}_T &\lesssim \frac{1}{\sqrt{T}}\left(L_{\mathcal H}(\{h_i^{(1)}\}_{i\in\mathcal H}, \mathbf{B}^{(0)}) - L_{\mathcal H}(\{h_i^\star\}_{i\in\mathcal H}, \mathbf{B}^\star)
+ \left\|\nabla_\mathbf{\mathbf{B}} L_{\mathcal H}(\{h_i^{(1)}\}_{i\in\mathcal H}, \mathbf{B}^{(0)})\right\|^2 + \kappa\Delta^{(0)}_B\right) \\
& + \kappa\left(H\bar\phi+\sigma\right)^2\left(\frac{p+\log(|\mathcal H|T/\delta)}{\tau|\mathcal H|}\right)
+ \kappa\,\sigma^2\left(\frac{p\vee rq + \log(|\mathcal H|T/\delta)}{\tau}\right) \\
& + \left(\kappa+\kappa^2+\frac{1+\kappa}{|\mathcal H|}\right)\frac{\sigma_g^2\log(T/\delta)}{\sqrt{T}}.
\end{align*}

\section{Bound on the Parameter Recovery Error}\label{sec:parameter_recovery_nonlinear}

We now convert the convergence result on $\bar{Q}_T$ into a bound on the average parameter-recovery
error $\Delta_{\mathbf{B},T+1}$, and then use it to bound the averaged prediction error $\tilde{E}_T$.

\subsection{Bound on $\Delta_{\mathbf{B},T+1}$}

We recall from Theorem \ref{theorem:Delta_BT_nonlinear} that
\begin{align*}
\Delta_{\mathbf{B},T+1} &\leq \frac{p_3\Delta^{(0)}_B}{\eta T}
+ p_4\sigma^2\left(\frac{rq+\log(|\mathcal H|T/\delta)}{\tau}\right)
+p_5\left(\left(H\bar\phi+\sigma\right)^2+\sigma^2\right)\left(\frac{p+\log(|\mathcal H|T/\delta)}{\tau|\mathcal H|}\right)
\nonumber\\
& + p_6\frac1T\sum_{t=0}^{T-1}\left\|\bar\delta^{(t)}\right\|^2
+ p_7\frac{1-\beta}{1+\beta}\sigma_g^2\log(T/\delta) + 2p_6\kappa \bar{G}_T.
\end{align*}

We first substitute the momentum-deviation bound of Lemma \ref{lem:G-momentum-dev},
\begin{align*}
\frac{1}{T}\sum_{t=0}^{T-1}\left\|\bar\delta^{(t)}\right\|^2
&\leq \frac{\bar b}{1-\bar a}\bar Q_T + C_2\frac{(1-\beta)^2}{|\mathcal H|(1-\bar a)}\sigma_g^2\log(T/\delta)\\
&+ \frac{C_2\eta L_1(1+\eta L_1)\beta^2\kappa}{1-\bar a}\left(\frac{1-\beta}{1+\beta}\sigma_g^2\log(T/\delta)+\bar{G}_T\right),
\end{align*}
into the bound of $\Delta_{\mathbf{B},T+1}$. By defining
\begin{align*}
p_{24} := 2p_6\kappa + p_6\frac{C_2\eta L_1(1+\eta L_1)\beta^2\kappa}{1-\bar a},
\end{align*}
as the coefficient multiplying $\bar{G}_T$, we obtain
\begin{align}
\Delta_{\mathbf{B},T+1} &\leq \frac{p_3\Delta^{(0)}_B}{\eta T}
+ p_4\sigma^2\left(\frac{rq+\log(|\mathcal H|T/\delta)}{\tau}\right)
+p_5\left(\left(H\bar\phi+\sigma\right)^2+\sigma^2\right)\left(\frac{p+\log(|\mathcal H|T/\delta)}{\tau|\mathcal H|}\right)
\nonumber\\
& + p_6\frac{\bar b}{1-\bar a}\bar Q_T + p_{24} \bar{G}_T
+ \left[p_6C_2\frac{(1-\beta)^2}{|\mathcal H|(1-\bar a)}
+ p_6\frac{C_2\eta L_1(1+\eta L_1)\beta^2\kappa}{1-\bar a}\cdot\frac{1-\beta}{1+\beta}\right.
\nonumber\\
&\qquad \left. + p_7\frac{1-\beta}{1+\beta}\right]\sigma_g^2\log(T/\delta). \label{eq:DeltaB_step1}\end{align}

We next substitute the gradient-heterogeneity bound of Lemma \ref{lemma:GT_nonlinear},
\begin{align*}
\bar{G}_T &\leq \frac{\Delta^{(0)}_B}{T}\left(\frac{p_8}{\eta}+p_1\right)
+ p_9\left(\left(H\bar\phi+\sigma\right)^2+\sigma^2\right)\left(\frac{p+\log(|\mathcal H|T/\delta)}{\tau|\mathcal H|}\right)
+ p_{13}\frac{1-\beta}{1+\beta}\sigma_g^2\log(T/\delta)
\\
& + \sigma^2\left(p_2\left(\frac{p\vee rq+\log(|\mathcal H|T/\delta)}{\tau}\right)
+ p_{12}\left(\frac{rq+\log(|\mathcal H|T/\delta)}{\tau}\right)\right)
+ p_{10}\frac{\bar b}{1-\bar a}\bar Q_T,
\end{align*}
into \eqref{eq:DeltaB_step1}. By collecting terms and defining the following quantities
\begin{align*}
p_{25} &:= \frac{p_3}{\eta} + p_{24}\left(\frac{p_8}{\eta}+p_1\right), \\
p_{26} &:= p_6\frac{\bar b}{1-\bar a} + p_{24}p_{10}\frac{\bar b}{1-\bar a}, \\
p_{27} &:= p_5 + p_{24}p_9, \\
p_{28} &:= p_{24} p_2, \\
p_{29} &:= p_4 + p_{24}p_{12}, \\
p_{30} &:= p_6C_2\frac{(1-\beta)^2}{|\mathcal H|(1-\bar a)}
+ p_6\frac{C_2\eta L_1(1+\eta L_1)\beta^2\kappa}{1-\bar a}\cdot\frac{1-\beta}{1+\beta}
+ p_7\frac{1-\beta}{1+\beta} + p_{24}p_{13}\frac{1-\beta}{1+\beta},
\end{align*}
we obtain
\begin{align}
\Delta_{\mathbf{B},T+1} &\leq p_{25}\frac{\Delta^{(0)}_B}{T} + p_{26}\bar Q_T
+ p_{27}\left(\left(H\bar\phi+\sigma\right)^2+\sigma^2\right)\left(\frac{p+\log(|\mathcal H|T/\delta)}{\tau|\mathcal H|}\right)
\nonumber\\
& + \sigma^2\left[p_{28}\left(\frac{p\vee rq+\log(|\mathcal H|T/\delta)}{\tau}\right)
+ p_{29}\left(\frac{rq+\log(|\mathcal H|T/\delta)}{\tau}\right)\right]
+ p_{30}\sigma_g^2\log(T/\delta). \label{eq:DeltaB_step2}\end{align}

Finally, we substitute the convergence guarantee for $\bar Q_T$ from \eqref{eq:convergence_nonlinear}
into \eqref{eq:DeltaB_step2}. As $p_{26} = \mathcal{O}(1)$ (as both $p_6\bar b/(1-\bar a)$ and
$p_{24}p_{10}\bar b/(1-\bar a)$ are $\mathcal{O}(1)$ in $\eta,\kappa$, using the same scalings discussed previously), this substitution does not require an additional restriction on
$\kappa$ beyond $\kappa\leq \min\{1/(2p_{14}), 1/(2p_{15})\}$ already imposed in Theorem~\ref{theorem:Delta_BT_nonlinear}. This yields the following result.

\begin{corollary}\label{thm:nonlinear_param_recovery}
Suppose the conditions of Theorem \ref{theorem:convergence_nonlinear} hold. Then, for every $\delta\in(0,1)$,
with probability at least $1-\delta$, setting $\eta = \min\{\eta_{\max}, 1/\sqrt{T}\}$, it holds that
\begin{align}
&\Delta_{\mathbf{B},T+1} \lesssim \frac{1}{\sqrt T}\left(L_{\mathcal H}(\{h_i^{(1)}\}_{i\in\mathcal H}, \mathbf{B}^{(0)})
- L_{\mathcal H}(\{h_i^\star\}_{i\in\mathcal H}, \mathbf{B}^\star)
+ \left\|\nabla_\mathbf{\mathbf{B}} L_{\mathcal H}(\{h_i^{(1)}\}_{i\in\mathcal H}, \mathbf{B}^{(0)})\right\|^2\right)\notag \\
&+ \frac{1+\kappa}{\sqrt T}\Delta^{(0)}_B
+\left(1+\kappa+\kappa^2\right)\left(H\bar\phi+\sigma\right)^2\left(\frac{p+\log(|\mathcal H|T/\delta)}{\tau|\mathcal H|}\right)\notag \\
&+\left(1+\kappa+\kappa^2\right)\sigma^2\left(\frac{p\vee rq+\log(|\mathcal H|T/\delta)}{\tau}\right)
+\left(\kappa+\kappa^2+\frac{1+\kappa}{|\mathcal H|}\right)\frac{\sigma_g^2\log(T/\delta)}{\sqrt T}. \label{eq:DeltaB_final}\end{align}
\end{corollary}
\begin{proof}
The bound follows from \eqref{eq:DeltaB_step2} by substituting~\eqref{eq:convergence_nonlinear} for
$\bar Q_T$ and applying the coefficient scalings $p_{25} = \mathcal{O}(1/\eta)$,
$p_{26} = \mathcal{O}(1)$, $p_{27} = \mathcal{O}(1+\kappa+\kappa^2)$,
$p_{28}, p_{29} = \mathcal{O}(1+\kappa+\kappa^2)$, and
$p_{30} = \mathcal{O}\!\left(\eta\left(\kappa+\kappa^2+\frac{1+\kappa}{|\mathcal H|}\right)\right)$, together with $\eta = \mathcal{O}(1/\sqrt T)$.
\end{proof}

\subsection{Bound on \texorpdfstring{$\bar E_T$}{E\_T}}

We now leverage Theorem \ref{thm:nonlinear_param_recovery} to bound the averaged prediction error
$\bar E_T := \frac{1}{T}\sum_{t=0}^{T-1}\bar E^{(t)}$, where $\bar E^{(t)}$ satisfies, for every iteration $t$, with probability $1-2\delta$,
\begin{align*}
\bar E^{(t)} \leq \left(2H^2\bar J^2 + \frac{16H^2\bar J^2\bar\phi^4}{\mu_1^2}\right)\Delta^{(t)}_B
+ \frac{16C_1^2\bar\phi^4\sigma^2}{\mu_1^2}\left(\frac{rq+\log(|\mathcal H|/\delta)}{\tau}\right).
\end{align*}
Therefore, 
\begin{align*}
p_{31} := 2H^2\bar J^2\left(1+\frac{8\bar\phi^4}{\mu_1^2}\right) \text{ and }
p_{32} := \frac{16C_1^2\bar\phi^4}{\mu_1^2},
\end{align*}
and averaging over $t=0,\ldots,T-1$, we obtain, with probability at least $1-2\delta$,
\begin{align*}
\bar E_T \leq p_{31}\,\Delta_{\mathbf{B},T} + p_{32}\,\sigma^2\left(\frac{rq+\log(|\mathcal H|T/\delta)}{\tau}\right).
\end{align*}
As $\Delta_{\mathbf{B},T} = \Delta_{\mathbf{B},T+1} + \frac{\Delta^{(0)}_B-\Delta^{(T)}_B}{T} \leq \Delta_{\mathbf{B},T+1} + \frac{\Delta^{(0)}_B}{T}$,
we can write
\begin{align}
\bar E_T \leq p_{31}\,\Delta_{\mathbf{B},T+1} + p_{31}\frac{\Delta^{(0)}_B}{T}
+ p_{32}\,\sigma^2\left(\frac{rq+\log(|\mathcal H|T/\delta)}{\tau}\right). \label{eq:ET_intermediate}\end{align}
Therefore, by substituting the bound on $\Delta_{\mathbf{B},T+1}$ from Theorem \ref{thm:nonlinear_param_recovery} into
\eqref{eq:ET_intermediate}, and noting that $\Delta^{(0)}_B/T = \mathcal{O}(\Delta^{(0)}_B/\sqrt T)$ is
dominated by the $(1+\kappa)\Delta^{(0)}_B/\sqrt T$ term already present in \eqref{eq:DeltaB_final}, we obtain the following corollary.

\begin{corollary}
\label{thm:nonlinear_ET}
Suppose the conditions of Corollary \ref{thm:nonlinear_param_recovery} hold, with probability at least $1-2\delta$,
\begin{align*}
\bar E_T &:= \frac{1}{T |\mathcal{H}| \tau}\sum_{t=0}^{T-1}\sum_{i \in \mathcal{H}} \sum_{k=1}^{\tau} \|h^{(t+1)}_i\phi_{\mathbf{B}^{(t)}}(X_{i,k}) - h^\star_{i}\phi_{\mathbf{B}^\star}(X_{i,k})\|^2 \notag \\
&\lesssim \frac{1}{\sqrt T}\left(L_{\mathcal H}(\{h_i^{(1)}\}_{i\in\mathcal H}, \mathbf{B}^{(0)})
- L_{\mathcal H}(\{h_i^\star\}_{i\in\mathcal H}, \mathbf{B}^\star)
+ \left\|\nabla_\mathbf{\mathbf{B}} L_{\mathcal H}(\{h_i^{(1)}\}_{i\in\mathcal H}, \mathbf{B}^{(0)})\right\|^2\right)
\nonumber\\
&+\left(1+\kappa+\kappa^2\right)\left(H\bar\phi+\sigma\right)^2\left(\frac{p+\log(|\mathcal H|T/\delta)}{\tau|\mathcal H|}\right)
+\left(1+\kappa+\kappa^2\right)\sigma^2\left(\frac{p\vee rq+\log(|\mathcal H|T/\delta)}{\tau}\right)
\nonumber\\
&+ \left(\kappa+\kappa^2+\frac{1+\kappa}{|\mathcal H|}\right)\frac{\sigma_g^2\log(T/\delta)}{\sqrt T} + \frac{1+\kappa}{\sqrt T}\Delta^{(0)}_B.
\end{align*}
\end{corollary}
\begin{proof}
The proof follows from \eqref{eq:ET_intermediate} and Corollary \ref{thm:nonlinear_param_recovery}, absorbing the coefficients $p_{31}, p_{32}$ (which depend on $H, \bar J, \bar\phi, \mu_1, C_1$, treated as fixed problem constants) into the constant of $\lesssim$.
\end{proof}

\section{Proof of the Multiclass Classification~Guarantee}\label{sec:classification-proof}

We first verify why one-hot multiclass classification fits the observation model used throughout this work. Note that conditional on $X$, let $Y=e_C$ and let $\pi(X)=\mathbb E[Y\mid X]$. Then $V:=Y-\pi(X)$ satisfies $\mathbb E[V\mid X]=0$. Moreover, for every unit vector $u\in\mathbb S^{q-1}$, the conditional random variable $u^\top V$ takes values in an interval of length at most
\begin{align*}
\max_{c\in[q]}u_c-\min_{c\in[q]}u_c\leq\sqrt{2}\|u\|=\sqrt{2}.\end{align*}
Leveraging Hoeffding's inequality \citep[Theorem 2.2.1]{vershynin2018high}, we can show that $u^\top V$ is conditionally sub-Gaussian with variance at most $1/2$. Thus, under realizability $\pi_i(X)=h_i^\star\phi_{\mathbf{B}^\star}(X)$, the one-hot classification model satisfies the observation model and noise condition of Section \ref{sec:nonlinear_rep_learning_mb}.

\begin{proof}[Proof of Corollary \ref{cor:multiclass-classification}]
We begin by fixing an honest client $i \in \mathcal{H}$, an iteration $t$, and a covariate $X$. For brevity, write
\begin{align*}
\widehat\pi_i^{(t)}(X):=h_i^{(t+1)}\phi_{\mathbf{B}^{(t)}}(X),\;
a:=C_i^\star(X),\text{ and } b:=\widehat C_i^{(t+1)}(X).
\end{align*}
If $a\neq b$, the definition of the learned classifier gives $\widehat\pi_{i,b}^{(t)}(X)\geq\widehat\pi_{i,a}^{(t)}(X)$. By combining this inequality with the margin condition (\eqref{eq:classification-margin}), we obtain
\begin{align*}
\gamma &\leq \pi_{i,a}(X)-\pi_{i,b}(X)\\
&=\pi_{i,a}(X)-\widehat\pi_{i,a}^{(t)}(X)+\widehat\pi_{i,a}^{(t)}(X)-\widehat\pi_{i,b}^{(t)}(X)
+\widehat\pi_{i,b}^{(t)}(X)-\pi_{i,b}(X)\\
&\leq \left|\widehat\pi_{i,a}^{(t)}(X)-\pi_{i,a}(X)\right|
+\left|\widehat\pi_{i,b}^{(t)}(X)-\pi_{i,b}(X)\right|\\
&\leq\sqrt{2}\left\|\widehat\pi_i^{(t)}(X)-\pi_i(X)\right\|.
\end{align*}
Hence, it follows, both when $a\neq b$ and when $a=b$, that
\begin{align*}
\mathbf 1\left\{\widehat C_i^{(t+1)}(X)\neq C_i^\star(X)\right\}
\leq\frac{2}{\gamma^2}
\left\|h_i^{(t+1)}\phi_{\mathbf{B}^{(t)}}(X)-h_i^\star\phi_{\mathbf{B}^\star}(X)\right\|^2.
\end{align*}
Therefore, by averaging this inequality over $t=0,\ldots,T-1$, $i\in\mathcal H$, and $k=1,\ldots,\tau$ yields
\begin{align*}
\overline{\mathcal C}_T\leq\frac{2 \bar E_T}{\gamma^2},
\end{align*}
which completes the proof.
\end{proof}

\begin{table}
\centering
\caption{Hyperparameters used in the numerical experiments. Here $|\mathcal{H}|$ and $f$ denote the numbers of honest and adversarial updates received at each communication round, respectively.}
\label{tab:experiment-hyperparameters}
\small
\begin{tabular}{lccc}
\toprule
Hyperparameter & CIFAR-10 & FEMNIST & School Exam Score \\
\midrule
$|\mathcal{H}|$ / $f$ & $50/5$ & $50/5$ & $20/5$ \\
Communication rounds & $100$ & $200$ & $500$ \\
Mini-batch size & $10$ & $10$ & $32$ \\
Representation dimension & $64$ & $64$ & $64$ \\
Representation learning rate & $0.01$ & $0.01$ & $0.05$ \\
Head learning rate & $0.01$ & $0.01$ & $0.01$ \\
Momentum coefficient & $0.5$ & $0.5$ & $0.9$ \\
\bottomrule
\end{tabular}
\end{table}

\section{Additional Details on the Experiments Implementation}
\label{sec:numerical-details}

We compare the proposed adversarially robust nonlinear representation learning approach with the common model baseline, i.e., a single-model Byzantine-robust FL approach. In our setting, every honest client $i \in \mathcal{H}$ has a private linear head and communicates only an update of the common nonlinear representation parameter. The baseline instead communicates an update of the entire common model. In both cases, the honest clients compute stochastic heavy-ball momentum and the server robustly aggregates the them along with arbitrary updates from the adversarial clients. We report results for coordinate-wise trimmed mean (NNM+TrMean) and Krum (NNM+Krum) with NNM (Nearest Neighbor Mixing) pre-processing \cite{allouah2023fixing}.\\

\noindent \textbf{Datasets and client heterogeneity.} For CIFAR-10, we use the heterogeneous partition of \citep{collins2021exploiting}, namely, the honest-client population contains $100$ clients, each with data from two classes. Each client has $500$ training and $100$ test examples. At every communication round, we sample $|\mathcal{H}|= 50$ honest clients and append $f=5$ adversarial updates. 

For FEMNIST, we use the LEAF data partition and treat each writer as a client. We keep the natural writer heterogeneity and sample $|\mathcal{H}|= 50$ honest clients per round, again together with $f=5$ Byzantine updates. 

For the School Exam Score regression task, we use the $139$ schools in the Inner London Education Authority dataset as the honest-client population. The covariates and labels are standardized combining each honest training samples, and each school is split into $80\%$ training and $20\%$ test data. We sample $20$ honest schools and append five Byzantine updates at each round.\\

\noindent \textbf{Models and objectives.} For CIFAR-10, the common representation is a convolutional network with two $5\times5$ convolutional layers with $64$ channels, each followed by a ReLU activation and $2\times2$ max pooling, and two fully connected layers of widths $120$ and $64$. The personalized head is a linear map from the $64$-dimensional representation to the $10$ class scores. 

For FEMNIST, we use a multilayer perceptron with widths $784$-$512$-$256$-$64$ and ReLU activations, followed by a personalized linear head. 

On the other hand, for the School Exam Score task, the shared representation is a two-hidden-layer multilayer perceptron with width $512$ and a $64$-dimensional output.  he personalized head is scalar. 

We also note that, for the classification tasks, we train with either cross-entropy or the multiclass squared loss as defined in Section~\ref{section:multiclass}. The latter is applied directly to the raw class scores, without a softmax transformation. For the School Exam Score task, we use the scalar squared loss.\\

\noindent \textbf{Training and evaluation.} At each round, a participating honest client $i \in \mathcal{H}$ first updates its private head while holding the shared representation fixed (personalization step). It then freezes the fitted head and computes the representation update on a separate mini-batch. Only the representation update is sent to the server. We use stochastic gradient descent for the local updates and clip the norm of every communicated honest update to one. The server uses a constant learning rate. The baseline uses the same backbone, loss, mini-batch size, momentum, aggregation rule, and attack, but aggregates the full-model updates. 

The classification performance is measured by the average local test accuracy across the honest clients, while regression performance is measured by the average local test mean-squared error. Table~\ref{tab:experiment-hyperparameters} summarizes the dataset-specific hyperparameters.\\

\textbf{Adversarial attacks.} We evaluate two different attacks. Let $u_1,\ldots,u_{\mathcal{H}}$ denote the honest updates in a given round, and let $\widehat\mu$ and $\widehat\sigma$ be their coordinate-wise empirical mean and population standard deviation. We consider the  ``A Little Is Enough'' (ALIE) attack where every adversarial client coordinate-wise update is $\widehat\mu+1.5\widehat\sigma$. 

Moreover, we also consider the Mimic attack, every adversarial client copies the update of the first sampled honest client. The implementation leverages the \texttt{ByzFL} library \citep{gonzález2025byzflresearchframeworkrobust} for both the attacks and robust aggregators.

\end{document}